\documentclass{article} % For LaTeX2e
\usepackage{iclr2024_conference,times}
\usepackage{algorithm}
\usepackage{algorithmic}

\usepackage{amsmath,amsfonts,bm}
\usepackage{amsthm} %%
\usepackage{mathtools}
\usepackage{graphicx} %%
\usepackage{wrapfig} %%
\usepackage{subfigure} %%
\usepackage{caption}
\usepackage{booktabs}       % professional-quality tables
\def\eqref#1{equation~\ref{#1}}
\def\1{\bm{1}}

\def\eps{{\epsilon}}

\newcommand{\R}{\mathbb{R}}

\DeclareMathOperator*{\argmax}{arg\,max}
\DeclareMathOperator*{\argmin}{arg\,min}

\usepackage{thmtools,thm-restate}

\newcommand{\loss} {\mathcal{L}}

\newcommand{\eat}[1]{}
\newcommand{\out}{\eat}
\newtheorem{theorem}{Theorem}[section]
\newtheorem{proposition}[theorem]{Proposition}

\newtheorem{definition}[theorem]{Definition}
\newtheorem{lemma}[theorem]{Lemma}

\newcommand{\bz}[1]{{\color{blue} BZ: #1}}

\makeatletter

\newcommand{\Rmnum}[1]{\expandafter\@slowromancap\romannumeral #1@}
\makeatother

\def\vv{{\bm{v}}}
\def\vw{{\bm{w}}}
\def\vx{{\bm{x}}}

\def\vz{{\bm{z}}}
\def\eps{{\varepsilon}}
\newcommand{\ro}{\partial}

\renewcommand{\L}{\mathcal{L}}
\newcommand{\pa}[1]{\left(#1\right)}
\newcommand{\br}[1]{\left[#1\right]}
\newcommand{\bk}[1]{\left<#1\right>}
\newcommand{\grad}{\nabla}

\DeclareMathOperator{\Lie}{Lie}

\def\vv{{\bm{v}}}
\def\vw{{\bm{w}}}
\def\vx{{\bm{x}}}

\def\eps{{\varepsilon}}

\newcommand{\cL}{\mathcal{L}}
\newcommand{\eqdef}{\overset{\text{def}}{=}}

\newcommand{\EE}[2]{\mathbb{E}_{#1}\left[#2\right] }
\newcommand{\E}[1]{\mathbb{E}\left[#1\right] }

\newcommand{\norm}[1]{\lVert#1\rVert}
\newcommand{\dotprod}[1]{\left< #1\right>}

\theoremstyle{plain}
\newtheorem{manualtheoreminner}{Theorem}
\newenvironment{manualtheorem}[1]{%
  \renewcommand\themanualtheoreminner{#1}%
  \manualtheoreminner
}{\endmanualtheoreminner}

\newtheorem{manualpropositioninner}{Proposition}
\newenvironment{manualproposition}[1]{%
  \renewcommand\themanualpropositioninner{#1}%
  \manualpropositioninner
}{\endmanualpropositioninner}

\theoremstyle{definition}

\theoremstyle{remark}

\usepackage{hyperref}
\usepackage{url}

\title{Improving Convergence and Generalization Using Parameter Symmetries}

\author{%
  Bo Zhao \\
  University of California San Diego \\
  \texttt{bozhao@ucsd.edu} \\
  \And
  Robert M. Gower \\
  Flatiron Institute \\
  \texttt{rgower@flatironinstitute.org} \\
  \And
  Robin Walters \\
  Northeastern University \\
  \texttt{r.walters@northeastern.edu} \\
  \And
  \hspace{21.5pt}Rose Yu \\
  \hspace{21.5pt}University of California San Diego \\
  \hspace{21.5pt}\texttt{roseyu@ucsd.edu} \\
}

\iclrfinalcopy % Uncomment for camera-ready version, but NOT for submission.
\begin{document}

\maketitle

\begin{abstract}
In many neural networks, different values of the parameters may result in the same loss value.
Parameter space symmetries are loss-invariant transformations that change the model parameters. 
Teleportation applies such transformations to accelerate optimization. However, the exact mechanism behind this algorithm's success is not well understood.
In this paper, we show that teleportation not only speeds up optimization in the short-term, but gives overall faster time to convergence.
Additionally, teleporting to minima with different curvatures improves generalization, which suggests a connection between the curvature of the minimum and generalization ability.
Finally, we show that integrating teleportation into a wide range of optimization algorithms and
optimization-based meta-learning improves convergence. 
Our results showcase the versatility of teleportation and demonstrate the potential of incorporating symmetry in optimization.
\end{abstract}

\section{Introduction}

Given a deep neural network architecture and a dataset, there may be multiple points in the parameter space that correspond to the same loss value. Despite having the same loss, the gradients and learning dynamics originating from these points can be very different \citep{kunin2021neural, van2017l2, grigsby2022functional}.
Parameter space symmetries, which are transformations of the parameters that leave the loss function invariant, allow us to \emph{teleport} between points in the parameter space on the same level set of the loss function \citep{armenta2023neural}.
% Recent work has identified a large set of continuous symmetries which act on the space of latent representations between layers in nonlinear deep neural networks \cite{zhao2022symmetries}. 
In particular, teleporting to a steeper point in the loss landscape leads to faster optimization. 
% Symmetry transformations can also be applied after training to create an ensemble of models with similar loss \cite{zhao2022symmetries}.
% This hints at a new approach for searching for models with good generalization ability.
% By relating conserved quantities values to convergence rate and sharpness of minimum, we can choose certain values of the conserved quantity at initialization that result in faster convergence or better model generalization ability.

% Teleportation for convergence: moving in orthogonal directions is not always useless. 
Despite the empirical evidence, the exact mechanism of how teleportation improves convergence in optimizing non-convex objectives remains elusive. 
Previous work shows that gradient increases momentarily after a teleportation, but could not show that this results in overall faster convergence \citep{zhao2022symmetry}. 
In this paper, we provide theoretical guarantees on the convergence rate. In particular, we show that stochastic gradient descent (SGD) with teleportation converges to a basin of stationary points, where every point reachable by teleportation is also stationary. 
% We also show that teleportation results in a quadratically contracting term in the convergence rate. 
% We then show that, if we use teleportation and the loss is Polyak-\L ojasiewicz (a general form of strongly convexity~\citep{PLschmidt}), then all points that can be reached by teleportation converge to the global minima.
We also provide conditions under which one teleportation guarantees optimality of the entire gradient flow trajectory.

Previous applications of teleportation are limited to accelerating optimization. The second part of this paper explores a different objective -- improving generalization. 
We relate properties of minima to their generalization ability and optimize them using teleportation. 
% One property that is believed to be correlated to the model's generalization ability is the sharpness of the loss function \citep{keskar2017large}.
We empirically verify that certain sharpness metrics are correlated with generalization \citep{keskar2017large}, although teleporting towards flatter regions has negligible effects on the validation loss. 
Additionally, we hypothesize that generalization also depends on the curvature of minima.
For fully connected networks, we derive an explicit expression for estimating curvatures and show that teleporting towards larger curvatures improves the model's generalizability.
% It is believed that there is a correlation between the curvature of loss and the model's generalization ability \cite{keskar2017large}. 
% Parameter space symmetry defines curves on the loss level sets. The curvature of these curves on minima measures a different sharpness of the minima. 
% We explore the relation between the curvature of loss, the curvature of minima, and generalization.
% We then introduce new objectives for teleportation to improve generalization by changing sharpness. 
% Only data-dependent symmetry helps - they produce different functions \citep{dinh2017sharp}.

% Gradient descent only explores a small portion of the parameter space \cite{gur2018gradient}. 
% % Gradient descent happens in a tiny subspace \cite{gur2018gradient}, 
% % Path Length Bounds for Gradient Descent and Flow \cite{gupta2021path}. 
% % Connecting Optimization and Generalization via Gradient Flow Path Length \cite{liu2022connecting}

To demonstrate the wide applicability of parameter space symmetry, we expand teleportation to standard optimization algorithms
beyond SGD, including  momentum, AdaGrad, RMSProp, and Adam.
Experimentally, teleportation improves the convergence speed for these algorithms.
Inspired by conditional programming and optimization-based meta-learning \citep{andrychowicz2016learning}, we also propose a meta-optimizer to learn where to move parameters in a loss level set. This approach avoids the computation cost of optimization on group manifolds and improves upon existing meta-learning methods that are restricted to local updates. 

The convergence speedup, applications in improving generalization, and the ability to integrate with different optimizers demonstrate the potential of improving optimization using symmetry. In summary, our main contributions are:
\begin{itemize}
    \item  theoretical guarantees that teleportation accelerates the convergence rate of SGD;
    \item quantifying the curvature of a minimum and evidence of its correlation with generalization; 
    \item a teleportation-based algorithm to improve generalization; 
    \item various optimization algorithms with integrated teleportation including momentum, AdaGrad, and optimization-based meta-learning.
\end{itemize}
\section{Related Work}
\paragraph{Parameter space symmetry.}
% Parameter space symmetry defines transformations on parameters that keep the loss function invariant.
% Various symmetries have been identified in neural networks. 
Continuous symmetries have been identified in the parameter space of various architectures, including homogeneous activations \citep{badrinarayanan2015symmetry, du2018algorithmic}, radial rescaling activations \citep{ganev2021universal}, and softmax and batchnorm functions \citep{kunin2021neural}.
Permutation symmetry has been linked to the structure of minima \citep{simsek2021geometry, entezari2022role}. 
Quiver representation theory provides a more general framework for symmetries in neural networks with pointwise \citep{armenta2021representation} and rescaling activations \citep{ganev2022quiver}. 
A new class of nonlinear and data-dependent symmetries are identified in \citep{zhao2022symmetries}.
Since symmetry defines transformations of parameters within a level set of the loss function, these works are the basis of the teleportation method discussed in our paper.

Knowledge of parameter space symmetry motivates new optimization methods.
One line of work seeks algorithms that are invariant to symmetry transformations \citep{neyshabur2015path-sgd, meng2019mathcal}.
Others search in the orbit for parameters that can be optimized faster \citep{armenta2023neural, zhao2022symmetry}. We build on the latter by providing theoretical analysis on the improvement of the convergence rate and by augmenting the teleportation objective to improve generalization. 

\paragraph{Initializations and restarts.}
% Effect of initialization/reinitialization on accuracy/generalization/optimization. 
Teleportation before training changes the initialization of parameters, which is known to affect the training dynamics.
For example, imbalance between layers at initialization affects the convergence of gradient flows in two-layer models \citep{tarmoun2021understanding}.
Different initializations, among other sources of variance, also lead to different model performance after convergence \citep{dodge2020fine,bouthillier2021accounting, ramasinghe2022you}.
% The Fourier spectrum at initialization is related to generalization because different frequency functions are learned at different rates \citep{ramasinghe2022you}. 
% For shallow networks, certain initialization is required to learn symmetric functions with generalization guarantees \citep{nachum2021symmetry}. 
In addition to initialization, teleportation allows changes in landscape multiple times throughout the training. 

Teleportation during training re-initializes the parameters to a point with the same loss. Its effect can resemble warm restart \citep{loshchilov2016sgdr}, which encourages parameters to move to more stable regions by periodically increasing the learning rate.
Compared to restarts, teleportation leads to smaller temporary increase in loss and provides more control of where to move the parameters.

% \cite{blanc2020implicit} Implicit regularization for deep neural networks driven by an Ornstein-Uhlenbeck like process
% \cite{agarwal2021deep} Deep Reinforcement Learning at the Edge of the Statistical Precipice
% \cite{picard2021torch}
% ``when learning symmetric functions, one can choose initial conditions so that standard SGD training efficiently produces generalization guarantees''
% ZerO Initialization: Initializing Neural Networks with only Zeros and Ones \cite{zhao2021zero}

\paragraph{Sharpness of minima and generalization.}
The sharpness of minima has been linked to the generalization ability of models both empirically and theoretically \citep{hochreiter1997flat, keskar2017large, petzka2021relative, ding2022flat, zhou2020towards}, which motivates optimization methods that find flatter minima \citep{chaudhari2017entropy, foret2021sharpness, kwon2021asam, kim2022fisher}. 
We employ teleportation to search for flatter points along the loss level sets.
The sharpness of a minimum is often defined using properties of the Hessian of the loss function, such as the number of small eigenvalues \citep{keskar2017large,chaudhari2017entropy, sagun2017empirical} or the product of the top $k$ eigenvalues \citep{wu2017towards}. 
%, and the Frobenius norm of the Hessian
Alternatively, sharpness can be characterized by the maximum loss within a neighborhood of a minimum \citep{keskar2017large,foret2021sharpness,kim2022fisher} or approximated by the growth in the loss curve averaged over random directions \citep{izmailov2018averaging}. 
The sharpness of minima does not always capture generalization \citep{dinh2017sharp} \citep{andriushchenko2023modern}. Some reparametrizations do not affect generalization but can lead to minima with different sharpness.

% \citep{hochreiter1997flat, keskar2017large, petzka2021relative, zhou2020towards}
% \cite{izmailov2018averaging, sagun2017empirical, wu2017towards}
% \citep{chaudhari2017entropy, foret2021sharpness, kim2022fisher, chao2020directional}. 
% \citep{dinh2017sharp}

% \cite{chao2020directional}
% SGD generalizes better than adaptive methods \cite{wilson2017marginal}, \cite{zhou2020towards}
% \citep{he2019asymmetric}

% \paragraph{Second-order optimization and approximations}
% Newton's method, Quasi-Newton methods, GGN

% It has been shown that update direction of first- and second-order optimization matches at the critical points of the gradient norm in loss level sets \cite{zhao2022symmetry}. 

% higher order optimization

% Natural gradient \citep{amari1998natural}, adaptive gradient methods \citep{duchi2011adaptive, kingma2015adam}, \citep{gupta2018shampoo}, K-FAC \citep{martens2018kfac}

% https://icerm.brown.edu/materials/Slides/sp-s19-w1/The_K-FAC_method_for_neural_network_optimization_]_James_Martens,_DeepMind.pdf

\section{Theoretical Guarantees for Improving Optimization}
\label{sec:optimization}
In this section, we provide a theoretical analysis of teleportation. We show that with teleportation, SGD converges to a basin of stationary points. Building on its relation to Newton's method, teleportation leads to a mixture of linear and quadratic convergence. Lastly, in certain loss functions, one teleportation guarantees optimality of the entire gradient flow trajectory. 

\paragraph{Symmetry Teleportation.} We briefly review the symmetry teleportation algorithm  \citep{zhao2022symmetry}, which searches for steeper points in a loss level set to accelerate gradient descent.
Consider the optimization problem 
\begin{equation*}
\vw^* = \underset{\vw \in \R^d}{\argmin}~ \cL(\vw), \quad \cL(\vw) \eqdef \EE{\xi\sim \cal  D}{\cL(\vw,\xi)}
\end{equation*}
where $\mathcal{D}$ is the data distribution, $\xi$ is data sampled from $\mathcal{D},$ $\cL$ the loss, $\vw$ the parameters of the model, and $\R^d$ the parameter space.
Let $ G$ be a group acting on the
parameter space, such that
\begin{align*}
    %\label{eq:invariant}
\cL(\vw) \; =\; \cL(g \cdot \vw), \quad \forall g \in  G, \; \forall \vw \in \R^d.
\end{align*}
Symmetry teleportation uses gradient ascent to find the group element $g$ that maximizes the magnitude of the gradient, and applies $g$ to the parameters while leaving the loss value unchanged:
% , before continuing with gradient descent:
\begin{align*}
    \vw' =& g \cdot \vw, \quad g = \underset{g\in G}{\text{argmax}}\| \grad \loss (g \cdot \vw)\|^2.
\end{align*}
% When each loss level set consists of a single $G$-orbit, this is equivalent to finding a point in the level set $S(x)=\{w: \loss(w) = x\}$ with the largest gradient norm:
% \begin{align}
%     w' =  \underset{v \in S(\loss(w))}{\text{argmax}}\| \grad \loss (v) \|^2.
% \end{align}

% After teleportation, we continue with gradient descent.

\subsection{Teleportation and SGD}
At each iteration $t \in \mathbb{N}^+$ in SGD, we choose a group element  $g^t \in G$ and 
use teleportation  before each gradient step as follows
\begin{equation}\label{eq:stochgradg}
\vw^{t+1}=g^t\cdot \vw^t - \eta \nabla \cL(g^t\cdot \vw^t,\xi^t).
\end{equation}
Here $ \eta$ is a learning rate,  $ \nabla \cL(\vw^t,\xi^t)$ is the gradient of  $\cL(\vw^t,\xi^t) $ with respect to the parameters $\vw$,
 and $\xi^t \sim \mathcal{D}$ is a mini-batch of data sampled i.i.d from the data distribution at each iteration. 
   
 By choosing the group element that maximizes the gradient norm, we show in the following theorem that the iterates in~\eqref{eq:stochgradg} converge to a basin of stationary points, where all points that can be reached via teleportation are also stationary points (visualized in Figure \ref{fig:thm1}).

 \begin{minipage}{0.61\textwidth}
\begin{theorem} \label{theo:telesgdconv}
(Smooth non-convex)
Let $\cL(\vw, \xi)$ be $\beta$--smooth and let
\[\sigma^2 \eqdef  \cL(\vw^*) -\E{\inf_\vw \cL(\vw,\xi)}.\]
Consider the iterates $\vw^t$ given by~\eqref{eq:stochgradg} where
\begin{equation*}
g^t \in \arg\max_{g\in G} \norm{ \nabla\cL(g\cdot \vw^t)}^2,
\end{equation*}
  which we assume exists.
  ~\footnotemark[1]
If  $\eta = \frac{1}{\beta \sqrt{ T-1}}$ then 
\begin{align}\label{eq:telesgdconv}
\min_{t=0,\ldots, T-1} &\E{ \max_{g\in G}  \norm{ \nabla\cL(g \cdot \vw^t)}^2  } \cr
\leq& \frac{2\beta}{ \sqrt{ T-1}}    \E{\cL(\vw^{0})-\cL(\vw^*)} + \frac{ \beta \sigma^2}{ \sqrt{ T-1}},~~% , \nonumber
\end{align}
where the expectation is the total expectation with respect to the data $\xi^t$ for $t=0, \ldots, T-1.$
% Consequently the complexity of reaching an $\epsilon$--stationary point is $\mathcal{O}1/\epsilon^2.$
% Furthermore,  if interpolation holds,  that is if $\sigma^2 =0 $ then  
\end{theorem}
\end{minipage} \hspace{0.07\textwidth}
\begin{minipage}{0.31\textwidth}
% \vskip -20pt
\centering
\includegraphics[width=1.0\textwidth]{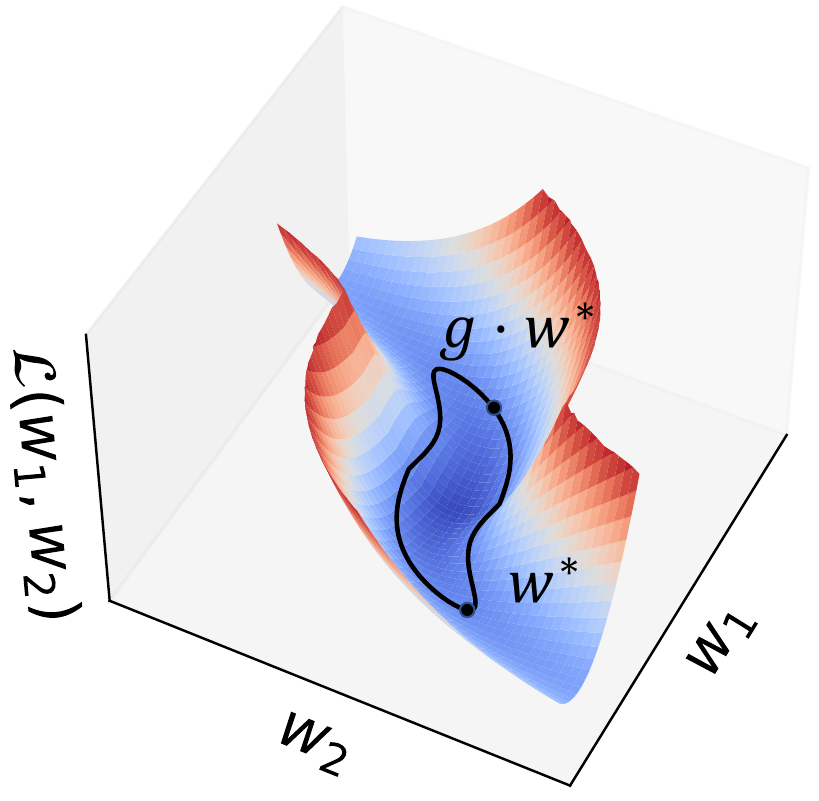}
\captionof{figure}{{{With teleportation, SGD converges to a basin where all points on the level set are stationary points.}} }
\label{fig:thm1}
\vskip -40pt
\end{minipage}
\footnotetext[1]{For instance when $G$ is compact and $\norm{ \nabla\cL(g\cdot \vw^t)}$ is continuous over $G$, or when the gradient is a coercive function and $G$ is bounded.
 % , since the gradient is continuous (consequence of smoothness, the maximum norm is attained
 }

This theorem is an improvement over vanilla SGD, for which we would have instead that
\begin{align*}
\min_{t=0,\ldots, T-1}  &\E{\norm{ \nabla\cL(\vw^t)}^2  } 
\leq  \frac{2\beta}{ \sqrt{ T-1}}    \E{\cL(\vw^{0})-\cL(\vw^*)} + \frac{ \beta \sigma^2}{ \sqrt{ T-1}} .
\end{align*}
The above only guarantees that there exists a single point $\vw^t$ for which the gradient norm will eventually be small.  In contrast,  our result in~\eqref{eq:telesgdconv} guarantees that for all points over the orbit $\{ g \cdot \vw^t \; : \; \forall g \in G \}$, the gradient norm will be small.
For strictly convex loss functions, $\max_{g\in G} \norm{ \nabla\cL(g\cdot \vw)}^2$ is non-decreasing with $\loss(\vw)$. In this case, the value of $\loss$ is smaller after $T$ steps of SGD with teleportation, compared to vanilla SGD (Proposition \ref{prop:convex-loss}).
% Furthermore in the following proposition we show that if $\loss$ is strictly convex, the value of $\loss$ is smaller with teleportation.
% %, as shown in Proposition \ref{prop:convex-loss}.
% The proof relies on showing $\max_{g\in G} \norm{ \nabla\cL(g\cdot w)}^2$ is non-decreasing with $\loss(w)$ and can be found in Appendix~\ref{sec:smooth-non-convex-proof}. 
% \begin{proposition}
% \label{prop:convex-loss}
% Assume that $\loss:\R^n \xrightarrow[]{} \R$ is strictly convex and twice continuously differentiable. 
% % Assume also that $G$ acts transitively on all loss level sets. 
% Assume also that for any two points $w_a, w_b \in \R^n$ such that $\loss(w_a) = \loss(w_b)$, there exists a $g \in G$ such that $w_a = g \cdot w_b$.
% At two points $w_1, w_2 \in \R^n$, if 
% $\max_{g\in G}  \norm{ \nabla\cL(g\cdot w_1)}^2 = \norm{ \nabla\cL(w_2)}^2$,
% then $\cL(w_1) \leq \cL(w_2)$.
% \end{proposition}

\out{
For non-convex loss functions that are smooth and are $\mu$--PL (Polyak-\L ojasiewicz), the following theorem shows that the expected loss converges to the optimal loss for every point on the orbit of $w^t$. 
\begin{theorem}{(Error Bound)}
\label{theo:error-bound}
    Consider the setting of Theorem~\ref{theo:telesgdconv}. If in addition to $\beta$--smoothness we also assume that the $\mu$-PL inequality holds, that is  %Ł
    \begin{eqnarray}\label{eq:PL}
      \loss(w)-\inf \loss \leq \frac{1}{2 \mu} \|\nabla \loss(w)\|^2, \quad \forall w \in \R^s,
    \end{eqnarray}
   then the  $w^t$ iterates given by~\eqref{eq:stochgradg} with  learning rate $\eta = \frac{1}{\beta \sqrt{ T-1}}$ 
converge according to
\begin{align}\label{eq:PLconv}
\min_{t=0,\ldots, T-1} &\E{  \max_{g\in G} \norm{g\circ w^t- w^*(g\circ w^t)}  } 
\leq \frac{2\beta}{ \mu}\frac{1}{\sqrt{ T-1}}    \E{\cL(w^{0})-\cL(w^*)} + \frac{\beta}{ \mu}\frac{  \sigma^2}{ \sqrt{ T-1}},%\nonumber
\end{align}
where $w^*(g\circ w^t)$ is the projection of $g\circ w^t$ onto to set of minimizers of $\cL(w).$
\end{theorem}

% Our result in~\eqref{eq:PLconv} now shows that, for loss functions that are smooth and are $\mu$--PL, we have that the expected loss converges to the optimal loss for every point on the orbit of $w^t$. 
Assuming the loss function is $\mu$--PL can be reasonable for neural networks that are overparametrized.
Indeed, in~\cite{LIU202285} the authors showed that, for overparametrized neural networks, when the task is regression, the loss function can often be $\mu$--PL. }

\subsection{Teleportation and Newton's method}
Intuitively,  teleportation can speed up optimization as it behaves similarly to Newton's method.
% The update direction in Newton’s method can be decomposed into two components: one in the gradient direction and one orthogonal to the gradient. Among all directions orthogonal to the gradient, this second component matches the direction along which the gradient norm increases the fastest (Proposition \ref{prop:positive-derivative-convex-max}).
% In other words, teleportation narrows the gap  between the second-order and first-order optimization methods.
% Since teleportation also increases the gradient norm, it has the effect of bringing a gradient descent update closer to a Newton's method update.
After a teleportation that takes parameters to a critical point on a level set, the gradient descent direction is the same as the Newton direction \citep{zhao2022symmetry}. As a result, we can leverage the convergence of Newton's method to derive the convergence rate of teleportation for the deterministic setting. 

\begin{proposition}[Quadratic term in convergence rate]\label{lem:newtondir}
Let $\loss$ be strictly convex and let $\vw_0\in \R^d$.  Let 
\begin{align*}%\label{eq:lemmanewtondir}
 &\vw' \in \argmax_{\vw\in \R^d} \frac{1}{2}\|\nabla \loss (\vw)\|^2 , \quad \mathrm{s.t.}\quad \loss(\vw) =\loss(\vw_0). 
  \end{align*}
Let $\nabla^2 \loss$ be the Hessian of $\loss$, and $\lambda_{\max}(\nabla^2 \loss(\vw))$ be the largest eigenvalue of $\nabla^2 \loss(\vw)$.
If $ \nabla \loss(\vw') \neq 0$, then there exists $\lambda_0$ such that
$ 0\leq  \lambda_0 \leq \lambda_{\max} ( \nabla^2 \loss(\vw_0)),$
and one step of gradient descent after teleportation with learning rate $\gamma >0$ gives
\begin{align}
 \vw_1&= \vw' -\gamma  \nabla \loss(\vw') %\nonumber \\
 = \vw' -\gamma \lambda_0 \nabla^2 \loss(\vw')^{-1} \nabla \loss(\vw').\label{eq:teleisnewton}
\end{align}
Let $\vw' = g_0 \cdot \vw_0$. If  $\gamma \leq \frac{1}{\lambda_0}$, $\cL$ is a $\mu$--strongly convex $L$--smooth function, and the Hessian is $G$--Lipschitz,
%  \begin{equation}\label{eq:strconv}
%  L  I    \succeq\nabla^2 \cL(w) \succeq \mu I, \quad \forall w \in \R^d,
% \end{equation}
% (see Lemma~\ref{lem:Newtonstep} for details),
% then under the assumptions of Lemma~\ref{lem:Newtonstep} 
then we have that
\[ \|\vw_1-\vw^*\| \leq    \frac{G}{2\mu}\|g_0 \cdot \vw_0-\vw^*\|^2 + |1-\gamma\lambda_0|\frac{L}{2\mu} \|g_0 \cdot \vw_0-\vw^*\|.\]
\end{proposition}
% For matrix $A, B$, $A \succeq B$ means $A-B$ is positive semidefinite. 
More details about the assumptions and the proof are in Appendix~\ref{sec:newtons-proof}.  Note that due to unknown step size $\lambda_0$, extra care is needed in establishing this convergence rate.  

The above proposition shows that taking one step of teleportation and one gradient step, the result is equal to taking a dampened Newton step~(\eqref{eq:teleisnewton}). Hence, the  convergence rate has a quadratically contracting term 
 $\|g_0 \cdot \vw_0-\vw^*\|^2 $, which is typical of second order methods. In particular, setting $\gamma =1/\lambda_0$ we would have local quadratic convergence. In contrast, without the teleportation step and under the same assumptions, we would have the following linear convergence 
\[ \|\vw_1-\vw^*\| \leq \left(1-\mu \gamma\right) \|\vw_0-\vw^*\| \]
for $\gamma \leq \frac{1}{L}$ using gradient descent. Thus there would be no quadratically contracting term.

\subsection{When is one teleportation enough}
\label{sec:one-teleportation-enough}
Despite the guaranteed improvement in convergence, teleporting before every gradient descent step is computationally expensive. 
Hence we teleport only occasionally. In fact, for certain optimization objectives, every point on the gradient flow has the largest gradient norm in its loss level set after one teleportation \citep{zhao2022symmetry}. 
In past work, this result is limited to convex quadratic functions.
In this section, we give a sufficient condition for when one teleportation results in an optimal trajectory for general loss functions.
Full proofs can be found in Appendix \ref{appendix:one-teleportation-enough}.

Let $V: \mathcal{M} \xrightarrow[]{} T\mathcal{M}$ be a vector field on the manifold $\mathcal{M}$, where $T\mathcal{M}$ denotes the associated tangent bundle. Here we consider the parameter space $\mathcal{M} = \R^n$, although results in this section can be extended to optimization on other manifolds. In this case, we may write $V = v^i \frac{\ro }{\ro w^i}$ using the component functions $v^i: \R^n \xrightarrow[]{} \R$ and coordinates $w^i$. 

% C^\infty(\R^n)
Consider a smooth loss function $\loss: \mathcal{M} \xrightarrow[]{} \R$. 
Let $G$ be a symmetry group of $\loss$, i.e. $\loss(g \cdot \vw) = \loss(\vw)$ for all $\vw \in \mathcal{M}$ and $g \in G$.
Let $\mathfrak{X}$ be the set of all vector fields on $\mathcal{M}$. 
Let $R = r^i \frac{\ro }{\ro w^i}$, where $r^i = -\frac{\ro \loss}{\ro w_i}$, be the reverse gradient vector field. 
Let $\mathfrak{X}_{\perp} = \{A=a^i \frac{\ro }{\ro w^i} \in \mathfrak{X}|~ a^i \in C^\infty(\mathcal{M}) \text{ and } \sum_i a^i (\vw) 
r^i(\vw) = 0, \forall \vw \in \mathcal{M}\}$ be the set of vector fields orthogonal to $R$. 
% replace with metric?
If $G$ is a Lie group, the infinitesimal action of its Lie algebra $\mathfrak{g}$ defines a set of vector fields $\mathfrak{X}_{\mathfrak{g}} \subseteq \mathfrak{X}_{\perp}$. 

A gradient flow is a curve $\gamma: \R \xrightarrow[]{} \mathcal{M}$ where the velocity is given by the value of $R$, i.e. $\gamma'(t) = R_{\gamma(t)}$ for all $t \in \R$. 
% A gradient flow is a curve $\gamma: \R \xrightarrow[]{} \mathcal{M}$ where the velocity is the value of $R$ at each point, i.e. $\gamma'(t) = R_{\gamma(t)}$ for all $t \in \R$. 
% The Lie bracket $[A, R]$ gives the difference between flowing a short time in the symmetry direction then in the gradient direction, and flowing in the gradient direction then in the symmetry direction. 
The Lie bracket $[A, R]$ defines the derivative of $R$ with respect to $A$. 
Flows of $A$ and $R$ commute if and only if $[A, R] = 0$ (Theorem 9.44, \cite{Lee2013}).
That is, teleportation can affect the convergence rate only if $[A, R]\loss \neq 0$ for some $A \in \mathfrak{X}_{\mathfrak{g}}$. 
To simplify notation, we write $([W, R]\loss)(\vw)=0$ for a set of vector fields $W \subseteq \mathfrak{X}$ when $([A, R]\loss)(\vw)=0$ for all  $A \in W$.

We consider a gradient flow optimal if every point on the flow is a critical point of the magnitude of gradient in its loss level set. 
Note that this definition does not exclude the case where points on the flow are minimizers of the magnitude of gradient.
% differential or tangent vectors? $df_{\vw^*}(A) = Af(\vw^*)$
% definition of critical point?
\begin{definition}
\label{def:optimal-point-flow}
    Let $f: \mathcal{M} \xrightarrow[]{} \R,  \vw \mapsto \left\|  \frac{\ro \loss}{\ro \vw} \right\|_2^2 $. A point $\vw \in M$ is optimal with respect to a set of vector fields $W \subseteq \mathfrak{X}_{\perp}$ if $Af(\vw) = 0$ for all $A \in W$. 
    A gradient flow $\gamma: \R \xrightarrow{} \mathcal{M}$ is optimal with respect to $W$ if $\gamma(t)$ is optimal with respect to $W$ for all $t \in \R$.
    % When $W = Y$, we simply call such $\vw$ and $\gamma$ optimal.
\end{definition}

\begin{proposition}
    A point $\vw \in \mathcal{M}$ is optimal with respect to a set of vector fields $W$ if and only if $([W, R]\loss)(\vw) = 0$.
\label{prop:Af-ARf}
\end{proposition}

A sufficient condition for one teleportation to result in an optimal trajectory is that whenever the function $[A, R]\loss$ vanishes at $\vw \in \mathcal{M}$, it vanishes along the entire gradient flow starting at $\vw$.
% also a necessary condition?
\begin{proposition}
    Let $W \subseteq \mathfrak{X}_{\perp}$ be a set of vector fields that are orthogonal to $\frac{\partial\loss}{\partial \vw}$. 
    Assume that for all $\vw \in \mathcal{M}$ such that $([W, R]\loss)(\vw)=0$, we have that $(R[W, R]\loss)(\vw)=0$.
    Then the gradient flow starting at any optimal point with respect to $W$ is optimal with respect to $W$.
\label{prop:one-teleportation}
\end{proposition}

To help check when the assumption in Proposition \ref{prop:one-teleportation} is satisfied, we provide an alternative form of $R[W, R]\loss(\vw)$ when $[W, R]\loss(\vw)=0$. 
\begin{proposition}
    If at all optimal points in $S = \{(M\frac{\partial \loss}{\partial \vw})^i \frac{\ro }{\ro w^i} \in \mathfrak{X}|~ M \in \R^{n \times n}, M^T = -M\}$ ,  
    \begin{align*}
        M_\alpha^j \frac{\ro \loss}{\ro w_k} \frac{\ro \loss}{\ro w_\alpha} \frac{\ro^3 \loss}{\ro w^k \ro w_i \ro w^j} \frac{\ro \loss}{\ro w^i} = 0
    \end{align*}
    for all anti-symmetric matrices $M \in \R^{n \times n}$, then the gradient flow starting at an optimal point in $S$ is optimal in $S$.
\label{prop:RARL2}
\end{proposition}

From Proposition \ref{prop:RARL2}, we see that $R[W, R]\loss(\vw)$ is not automatically 0 when $[W, R]\loss(\vw)=0$. Therefore, even if the group is big enough to have its infinitesimal actions cover the tangent space of the level set ($\mathfrak{X}_{\mathfrak{g}} = \mathfrak{X}_{\perp}$), one teleportation does not guarantee that the gradient flow intersects all future level sets at optimal points. However, for loss functions that satisfy the condition in Proposition \ref{prop:one-teleportation}, teleporting once optimizes the entire trajectory.   This is the case, for example, when $\frac{\ro^3 \loss}{\ro w^k \ro w^i \ro w^j} \frac{\ro \loss}{\ro w^\alpha} = \frac{\ro^3 \loss}{\ro w^k \ro w^i \ro w^\alpha} \frac{\ro \loss}{\ro w^j}$ for all $i, k, j, \alpha$ (Proposition \ref{proposition:anti-sym-equality}). In particular, all quadratic functions meet this condition. 
% A special case is when $\loss$ is a quadratic function, when the third derivative is 0.

% \subsection{Approximate teleportation}
% \bz{to do}

\section{Teleportation for Improving Generalization}

Teleportation was originally proposed to speedup optimization. 
In this section, we explore the suitability of teleportation for improving generalization, which is another important aspect of deep learning. 
We first  review definitions of the sharpness of minima. Then, we introduce a novel notion of the curvature of minima and discuss its implications on generalization. 
By observing how sharpness and curvature of minima are correlated with generalization, we improve generalization by incorporating sharpness and curvature into the objective for teleportation.

\subsection{Sharpness of minima}
Flat minima tend to generalize well \citep{hochreiter1997flat}, typically characterized by numerous small Hessian eigenvalues. 
Although Hessian-based sharpness metrics are known to correlate well with generalization, they are expensive to compute and differentiate through. 
To use sharpness as an objective in teleportation, we consider changes in the loss averaged over random directions. Let $D$ be a set of vectors drawn randomly from the unit sphere $d_i \sim \{d \in \R^n: ||d||=1\}$, and $T$ a list of displacements $t_j \in \R$. Then, we have the following metric \citep{izmailov2018averaging}:
\begin{align}
    \text{Sharpness: } \quad \phi(\vw, T, D) = \frac{1}{|T| |D|}\sum_{t\in T} \sum_{d \in D} \cL(\vw + t d).\quad 
\end{align}
% width by small perturbations 

\begin{wrapfigure}{R}{0.4\textwidth}
\vskip -25pt
\centering
\includegraphics[width=0.4\textwidth]{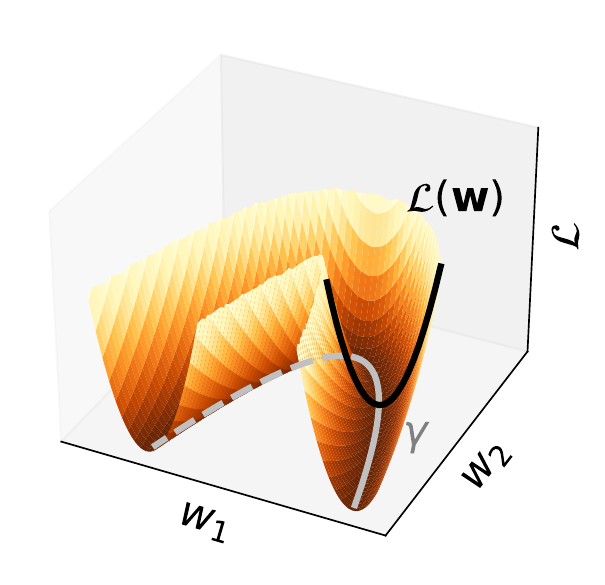}
\caption{Gradient flow ($\loss(\vw)$) and a curve on the minimum ($\gamma$). The curvature of both curves may affect generalization.}
\label{fig:curvatures}
\vskip -20pt
\end{wrapfigure}

\subsection{Curvature of minima}
At a minimum, the loss-invariant or flat directions are zero eigenvectors of the Hessian. 
The curvature along these directions does not directly affect Hessian-based sharpness metrics. 
However, these curvatures may affect generalization, by themselves or by correlating to the curvature along non-flat directions. 
% However, the curvature along these directions can still be meaningful --- it may be correlated with the curvature along other non-flat directions, or have an impact on generalization by itself. 
Unlike the curvature of the loss (curve $\loss(\vw)$ in Figure \ref{fig:curvatures}), the curvature of the minima (curve $\gamma$) is less well studied. We provide a novel method to quantify the curvature of the minima below.

% We first define the curves that lie in the minimum. 
Assume that the loss function $\loss$ has a $G$ symmetry. Consider the curve $\gamma_M: \R \times \R^n \xrightarrow[]{} \R^n$ where $M \in \Lie(G)$ and $\gamma_M(t, \vw) = \exp{(tM)} \cdot \vw$. Then $\gamma(0, \vw) = \vw$, and every point on $\gamma_M$ is in the minimum if $\vw$ is a minimum. 
% The curve $\gamma$ in Figure \ref{fig:curvatures} is an example of such a curve. 
Let $\gamma'=\frac{d\gamma}{dt}$ be the derivative of a curve $\gamma$. 
% Similar, let $\gamma''=\frac{d^2\gamma}{dt^2}$ be the second derivative. 
The curvature of $\gamma$ is $\kappa(\gamma, t) = \frac{\|T'(t)\|}{\|\gamma'(t)\|}$, where $T(t) = \frac{\gamma'(t)}{\|\gamma'(t)\|}$ is the unit tangent vector.
% We want the curvature of this curve at $t=0$. 
% The exponential map is smooth. 
We assume that the action map is smooth, % can we?
since calculating the curvature requires second derivatives and optimizing the curvature via gradient descent requires third derivatives. For multi-layer network with element-wise activations, we derive the group action, $\gamma$, and $\kappa$ in Appendix \ref{sec:appendix-group}.

Since the minimum can have more than one dimension, we measure the curvature of a point $\vw$ on the minimum by averaging the curvature of $k$ curves with randomly selected Lie algebra elements $M_i \in \Lie(G)$.
The resulting new metric is
\begin{align}
    \text{Curvature: } \quad  \psi(\vw, k) = \frac{1}{k} \sum_{i=1}^k \left. \kappa(\gamma_{M_i}(0, \vw), 0)\right..
\end{align}
There are different ways to measure the curvature of a higher-dimensional manifold, such as using the Gaussian curvature of 2D subspaces of the tangent space. However, our method of approximating the mean curvature is easier to compute and suitable as a differentiable objective.

\begin{figure}[h]
\begin{center}
\ \ \ (a) \hfill (b) \hfill (c) \hfill (d) \hfill ~ \\
\includegraphics[width=0.48\columnwidth]{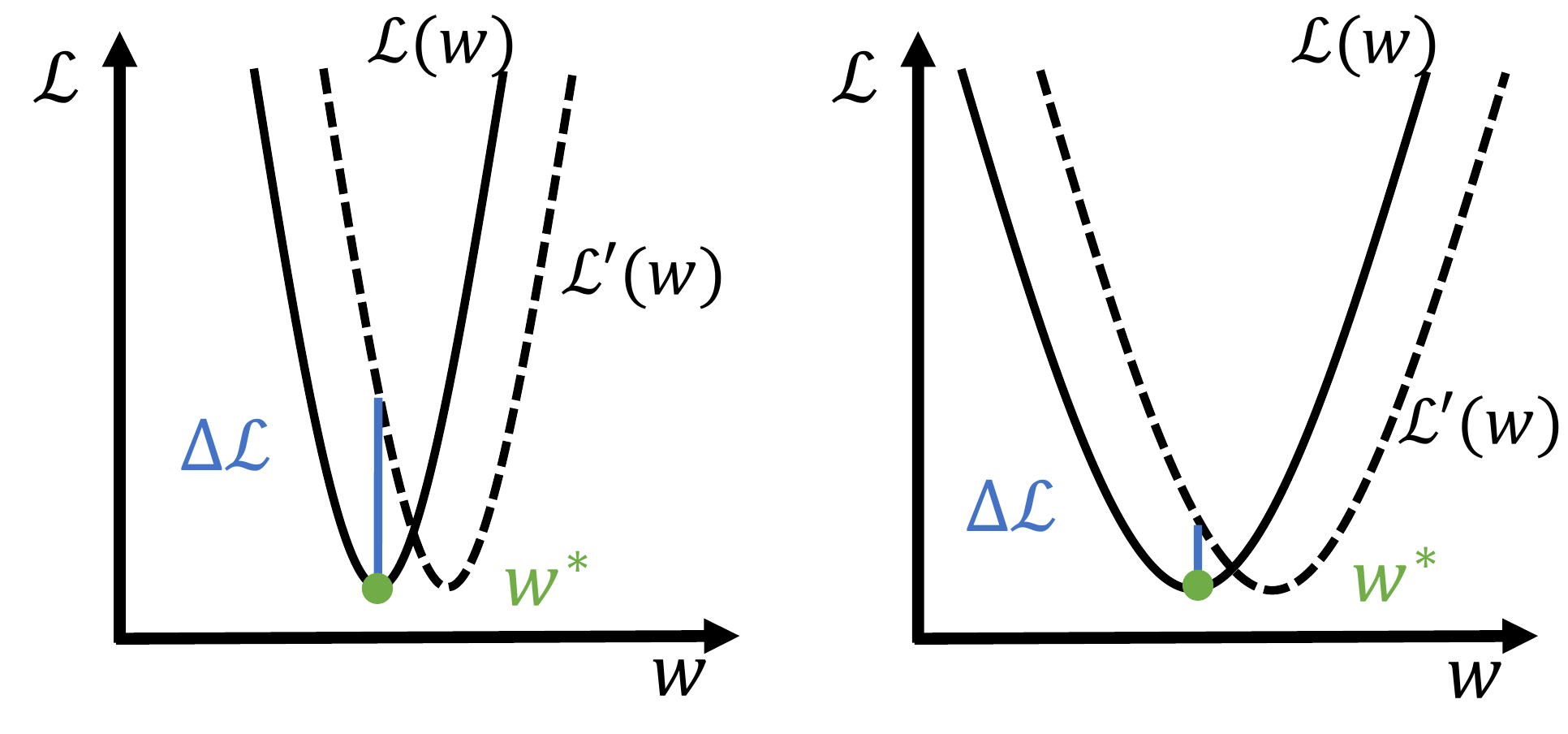}
% \ \ \ (c) \hfill (d) \hfill ~ \\
\includegraphics[width=0.48\columnwidth]{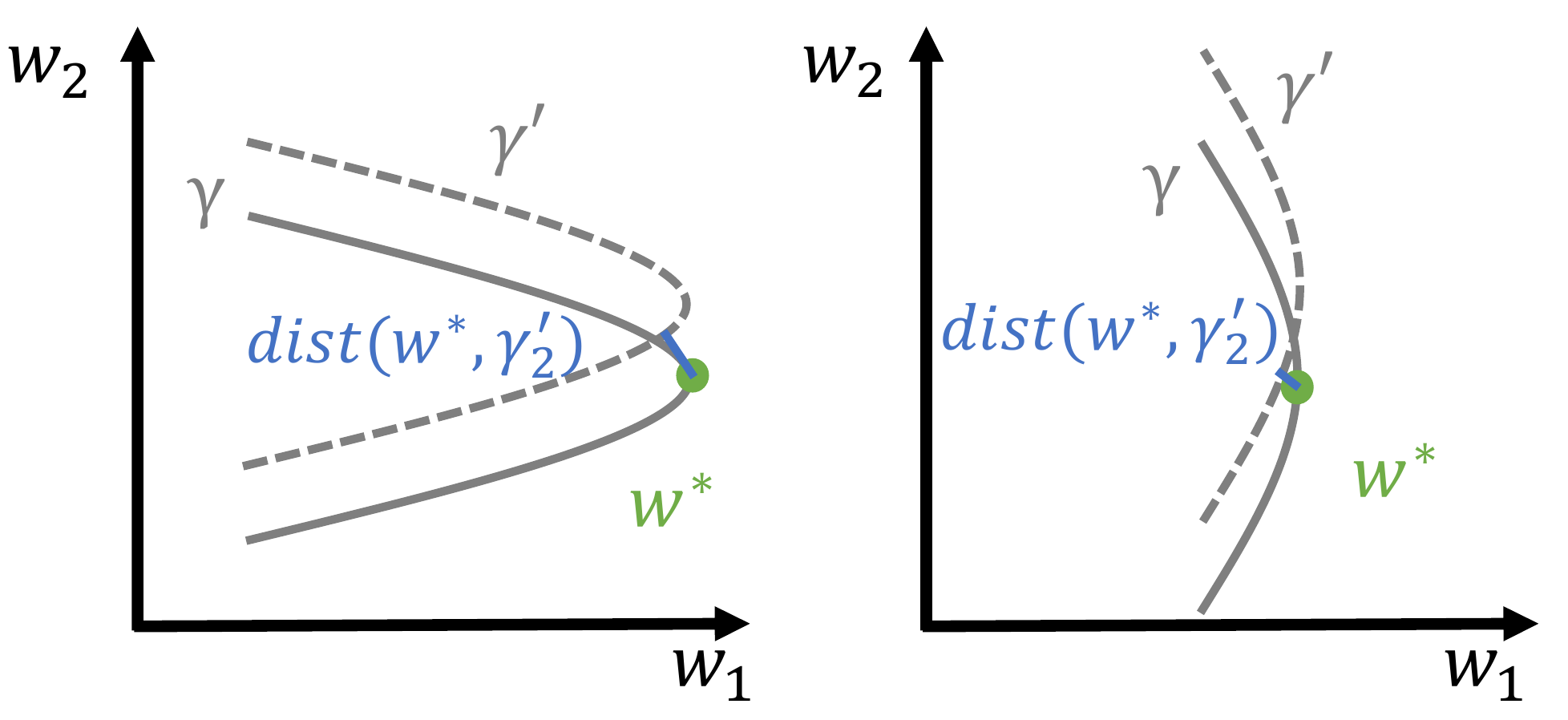}
\caption{Illustration of the effect of sharpness (a,b) and curvature (c,d) of minima on generalization. See Figure \ref{fig:curvatures} for a 3D visualization of the curves $\loss(\vw)$ and $\gamma$. When the loss landscape shifts due to a change in data distribution, sharper minima have larger increase in loss. In the example shown, minima with larger curvature moves further away from the shifted minima. }
\label{fig:curvature-intuition}
\end{center}
\end{figure}

\subsection{Correlation with generalization}
Generalization reflects how loss changes with shifts in data distribution. 
The sharpness of minima is well known to be correlated with generalization. 
Figure \ref{fig:curvature-intuition}(a)(b) visualizes an example of the shift in loss landscape ($\loss(\vw)$), and the change of loss $\Delta \loss$ at a minimizer $\vw^*$ is large when the minimum is sharp.
The relation between the curvature of minimum and generalization is less well studied. 
Figure \ref{fig:curvature-intuition}(c)(d) shows one possible shift of the minimum ($\gamma$). Under this shifting, the minimizer with a larger curvature becomes farther away from the shifted minimum. 
The curve on the minimum can shift in other directions. Appendix \ref{sec:appendix-group-intuition} provides analytical examples of the correlation between curvature and expected distance between the old and shifted minimum. 

We verify the correlation between sharpness, curvatures, and validation loss on MNIST \citep{deng2012mnist}, Fashion-MNIST \citep{xiao2017fashion}, and CIFAR-10 \citep{krizhevsky2009learning}.
On each dataset, we train 100 three-layer neural networks with LeakyReLU using different initializations. 
% The parameters for the metrics are $\eps=100$ for $\phi_1$, $k=200$ for $\phi_2$, $T=[0.1, 1, 2, ..., 20]$ for $\phi_3$, and $k=500$ for $\psi$. Details of training and effect of metric parameters can be found in in Appendix \ref{sec:appendix-group-experiment}. 
Details of the setup can be found in Appendix \ref{sec:appendix-correlation-experiment}. 

Table \ref{table:correlation} shows the Pearson correlation between validation loss and sharpness or curvature (scatter plots in Figure \ref{fig:correlation-sharpness-3layer-leakyrelu} and \ref{fig:correlation-curvature-3layer-leakyrelu} in the appendix).
In all three datasets, sharpness has a strong positive correlation with validation loss, meaning that the average change in loss under perturbations is a good indicator of test performance. %This also confirms that wider minima are more generalizable. 
For the architecture we consider, the curvature of minima is negatively correlated with the validation loss. We observe that the magnitudes of the curvatures are small, which suggests that the minima are relatively flat.

% Correlation among different sharpness (possibly going into appendix). 

\begin{table}
  \caption{Correlation with validation loss}
  \label{table:correlation}
  \centering
  \begin{tabular}{cccccc}
    \toprule
    \multicolumn{3}{c}{sharpness ($\phi$)} & \multicolumn{3}{c}{curvature ($\psi$)}  \\
    \cmidrule(r){1-3} \cmidrule(r){4-6}
    MNIST & Fashion-MNIST & CIFAR-10 & MNIST & Fashion-MNIST & CIFAR-10 \\
    \midrule
    0.704 & 0.790 & 0.899 & -0.050 & -0.232 & -0.167 \\
    \bottomrule
  \end{tabular}
\end{table}

\subsection{Teleportation for improving generalization}
To improve the generalization ability of the minimizer and to gain understanding of the curvature of minima, we teleport parameters to regions with different sharpness and curvature. 
Multi-layer neural networks have $GL(\R)$ symmetry between layers (Appendix \ref{sec:appendix-group-action}).
We parametrize the group by its Lie algebra $T$, and perform gradient ascent on $T$ to maximize the gradient norm at the transformed parameters $|\nabla L \vert_{\exp{(T)} \cdot w}|$. 
% We keep $T$ small by starting from $T=0$ at every gradient ascent step, which justifies using the first-order approximation for the exponential map ($\exp{(T)} \approx I + T$). 
Algorithm \ref{alg:teleport-mlp} in Appendix \ref{sec:appendix-generalization-experiment} demonstrates how to increase curvature $\psi$ by teleporting two layers, with hidden dimension $h$, in an MLP. In experiments, we use an extended version of the algorithm, which teleports all layers by optimizing on a list of $T$’s concurrently.
During teleportation, we perform gradient descent on the group elements to change $\phi$ or $\psi$. Results are averaged over 5 runs.

\begin{figure}[h!]
\centering
% \ \ \ a\hfill b \hfill c\hfill d \hfill ~ \\
% \ \ \ \hspace{10pt} (a) \hfill \hspace{-15pt} (b) \hfill ~ \\
\includegraphics[width=0.4\columnwidth]{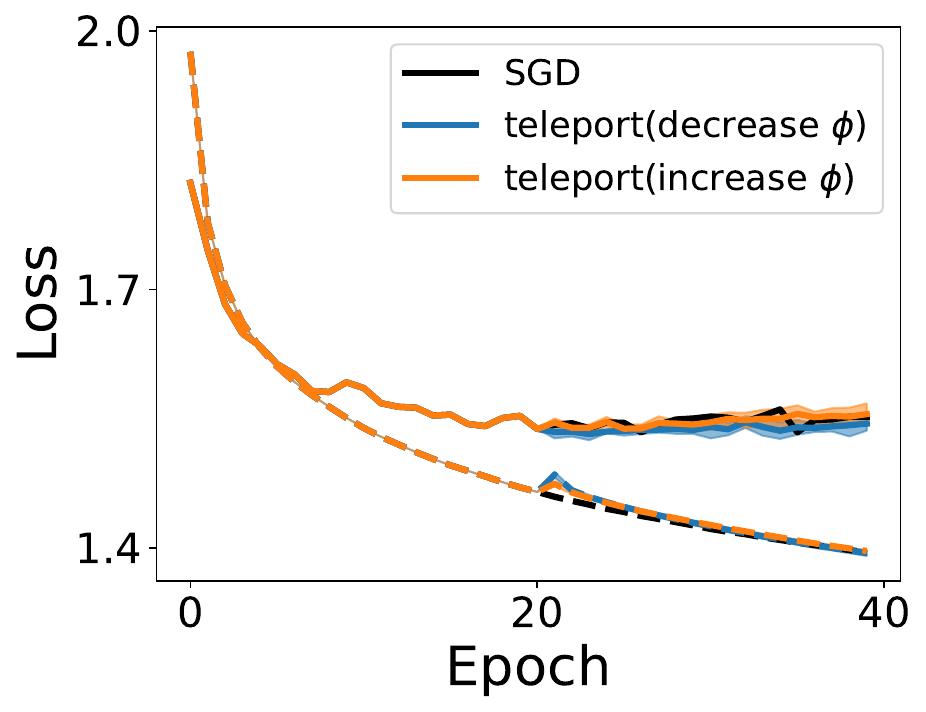}
\hspace{20pt}
\includegraphics[width=0.4\columnwidth]{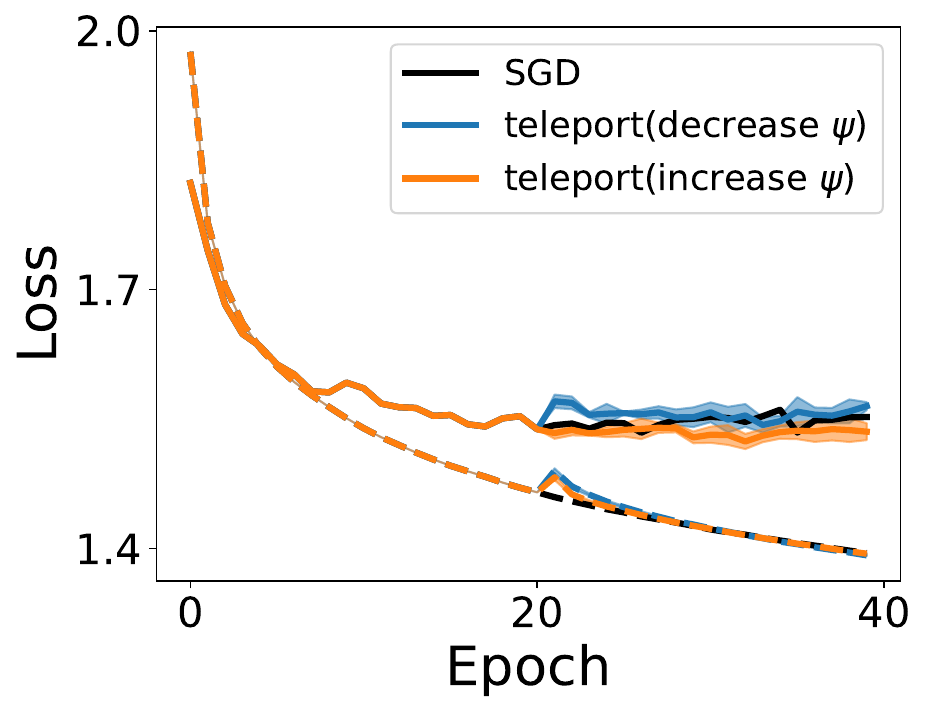}
\caption{Changing sharpness (left) or curvature (right) using teleportation and its effect on generalization on CIFAR-10.
Solid line represents average test loss, and dashed line represent average training loss. 
Teleporting to decrease sharpness improves validation loss slightly. Teleportation changing curvatures has a more significant impact on generalization ability. 
}
\label{fig:teleport-generalization-cifar10}
\end{figure}

Figure \ref{fig:teleport-generalization-cifar10} shows the training curve of SGD on CIFAR-10, with one teleportation at epoch 20. Similar results for AdaGrad can be found in Appendix \ref{sec:appendix-generalization-experiment}. Teleporting to flatter points slightly improves the validation loss, while teleporting to sharper points has no effect.
Since the group action keeps the loss invariant only on the batch of data used in teleportation, the errors incurred in teleportation have a similar effect to a warm restart, which makes the effect of changing sharpness less clear. 

Interestingly, by changing the curvature, teleportation is able to affect generalization. Teleporting to points with larger curvatures helps find a minimum with lower validation loss, while teleporting to points with smaller curvatures has the opposite effect. 
This suggests that at least locally, curvature is correlated with generalization.
% We observe similar results on MNIST and Fashion-MNIST (Appendix \ref{sec:appendix-generalization-experiment-results}). 
Details of the experiment setup can be found in Appendix \ref{sec:appendix-generalization-experiment}. 

\section{Applications to Other Optimization Algorithms}
Having shown teleportation's potential to improve optimization and generalization, we demonstrate its wide applicability by integrating teleportation into different optimizers and meta-learning. 

\subsection{Standard optimizers}
Teleportation improves optimization not only for SGD. 
To show that teleportation works well with other standard optimizers, we train a 3-layer neural network on MNIST using different optimizers with and without teleportation. During training, we teleport once at the first epoch, using 8 minibatches of size 200. Details can be found in Appendix \ref{sec:appendix-optimization-experiment}.

Figure \ref{fig:other-algorithms-loss-vs-epoch} shows that teleportation improves the convergence rate when using AdaGrad, SGD with momentum, RMSProp, and Adam.
The runtime for a teleportation is smaller than the time required to train one epoch, hence teleportation improves convergence rate per epoch at almost no additional cost of time (Figure \ref{fig:other-algorithms-loss-vs-time} in the appendix).

\begin{figure}[h]
\centering
% \ \ \ (a) \hfill (b) \hfill (c) \hfill (d) \hfill ~ \\
% \ \ \ (a)\hfill (b) \hfill ~ \\
\includegraphics[width=0.245\columnwidth]{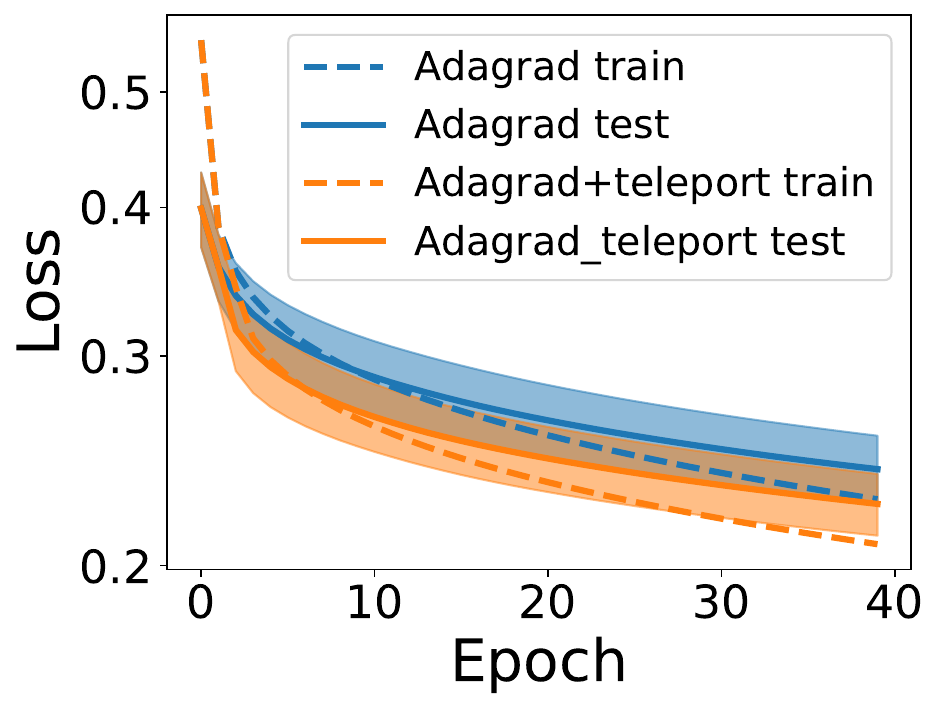}
\includegraphics[width=0.245\columnwidth]{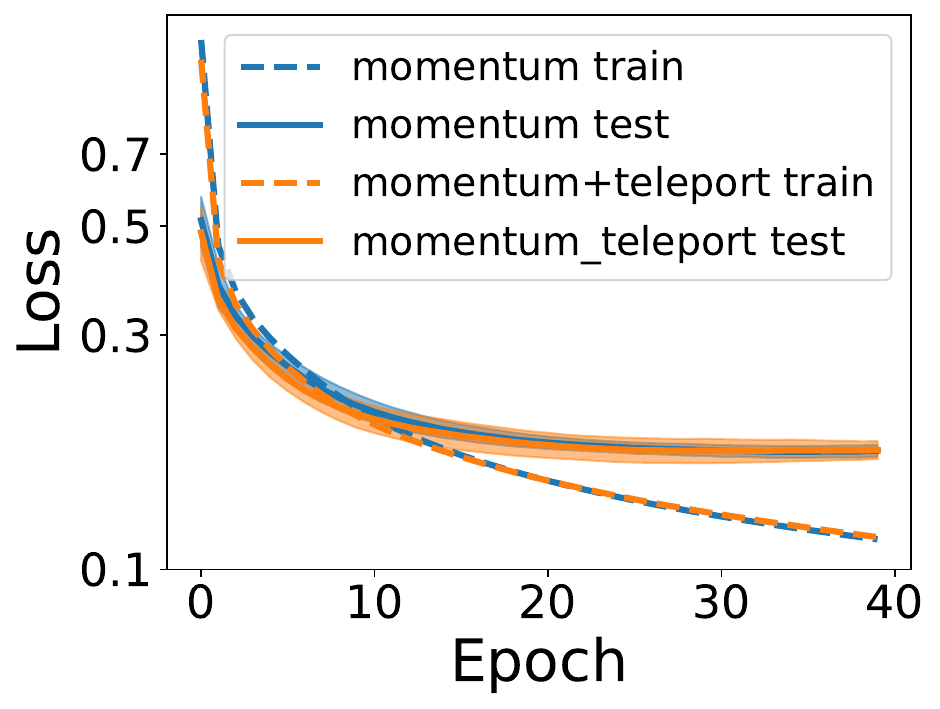}
\includegraphics[width=0.245\columnwidth]{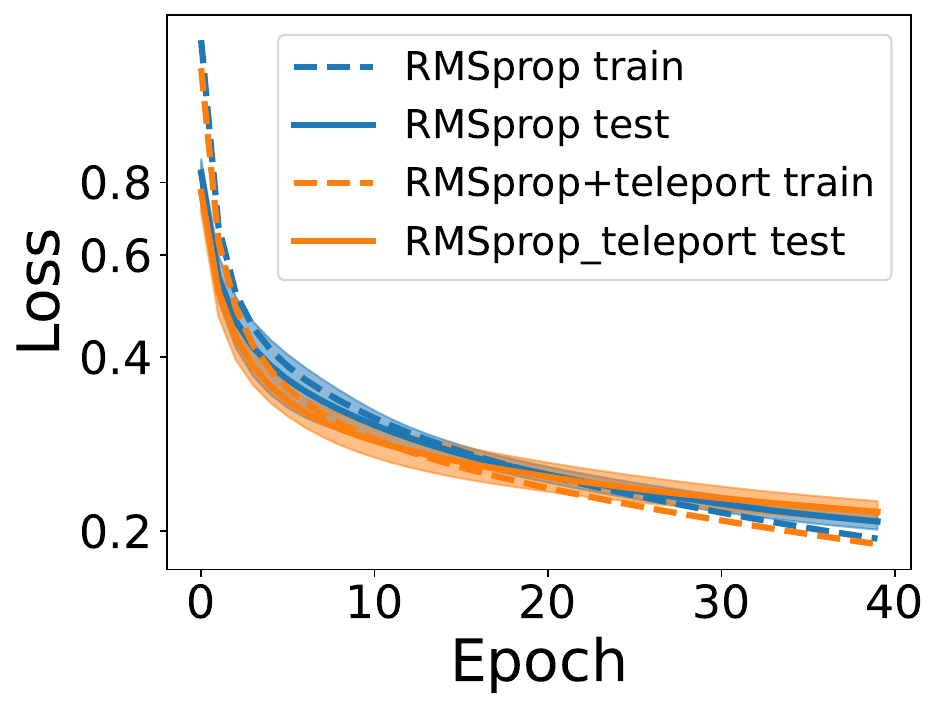}
\includegraphics[width=0.245\columnwidth]{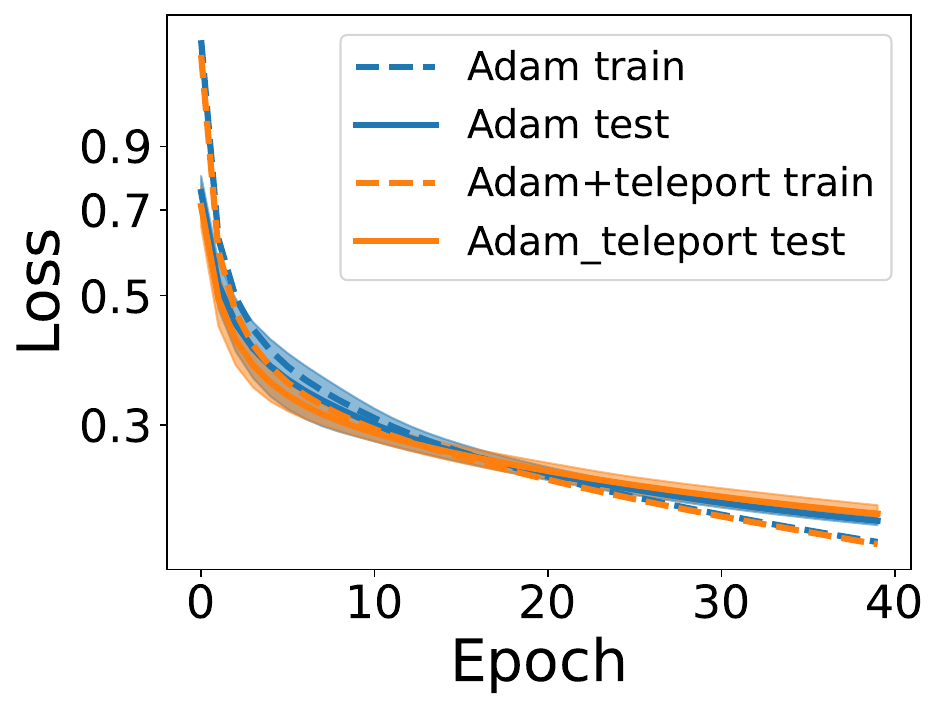}
\caption{Integrating teleportation with AdaGrad, momentum, RMSProp, and Adam improves the convergence rate on MNIST. Solid line represents the average test loss, and dashed line represents the average training loss. Shaded areas are 1 standard deviation of the test loss across 5 runs. }
\label{fig:other-algorithms-loss-vs-epoch}
\end{figure}

\subsection{Learning to teleport}
In optimization-based meta-learning, the parameter update rule or the hyperparameters are learned using a meta-optimizer \citep{andrychowicz2016learning,finn2017model}.
% \citep{andrychowicz2016learning, ravi2016optimization, finn2017model, nichol2018first, chandra2022gradient}
Teleportation introduces an additional degree of freedom in parameter updates. 
% To teleport without implementing optimization on groups, 
We augment existing meta-learning algorithms by learning both the local update and teleportation. This allows us to teleport without implementing the additional optimization step on groups, which reduces computation time.

Let $\vw_t \in \R^d$ be the parameters at time $t$, and $\grad_t = \left.\frac{\partial \loss}{\partial \vw}\right|_{\vw_t}$ be the gradient of the loss $\loss$. 
In gradient descent, the update rule with learning rate $\eta$ is  
\begin{align*}
    \vw_{t+1} = \vw_t - \eta \nabla_t.
\end{align*}

In meta-learning \citep{andrychowicz2016learning}, the update on $\vw_t$ is learned using a meta-learning optimizer $m$, which takes $\grad_t$ as input. Here $m$ is an LSTM model. Denote $h_t$ as the hidden state in the LSTM and $\phi$ as the parameters in $m$. The update rule is
\begin{align*}
    \vw_{t+1} &= \vw_t + f_t \\
    \begin{bmatrix}
    f_t \\
    h_{t+1}
    \end{bmatrix}
    &= m(\nabla_t, h_t, \phi).
\end{align*}

Extending this approach beyond an additive update rule, we learn to teleport. Let $G$ be a group whose action on the parameter space leaves $\loss$ invariant.
We use two meta-learning optimizers $m_1, m_2$ to learn the update direction $f_t \in \R^d$ and the group element $g_t \in G$:
\begin{align*}
    &\vw_{t+1} = g_t \cdot (\vw_t + f_t) \\
    \begin{bmatrix}
    f_t \\
    h_{1_{t+1}}
    \end{bmatrix}
    = m_1(&\nabla_t, h_{1_t}, \phi_1), ~~
    \begin{bmatrix}
    g_t \\
    h_{2_{t+1}}
    \end{bmatrix}
    = m_2(\nabla_t, h_{2_t}, \phi_2).
\end{align*}

% $\theta$: parameters of optimizee

% $\nabla$: gradient of optimizee

% $\eta$: learning rate

% $g$: update on $\theta$

% $m$: meta-learning optimizer (LSTM)

% $h$: hidden state of LSTM

% $\phi$: parameters of LSTM

\paragraph{Experiment setup.}
We train and test on two-layer neural networks  $\loss(W_1,W_2) = \| Y - W_2 \sigma(W_1 X)\|_2$, where $W_2, W_1, X, Y \in \R^{20 \times 20}$, and $\sigma$ is the LeakyReLU function with slope coefficient 0.1. 
Both meta-optimizers are two-layer LSTMs with hidden dimension 300.
We train the meta-optimizers on multiple trajectories created with different initializations, each consisting of 100 steps of gradient descent on $\loss$ with random $X, Y$ and randomly initialized $W$'s. We update the parameters in $m_1$ and $m_2$ by unrolling every 10 steps. The learning rate for meta-optimizers are $10^{-4}$ for $m_1$ and $10^{-3}$ for $m_2$.
We test the meta-optimizers using 5 trajectories not seen in training.
% We train our meta-optimizers on multiple trajectories created with different initializations (Alg. \ref{alg:learning-to-teleport}).

% In addition to training the meta-optimizer with Algorithm \ref{alg:learning-to-teleport} (``LSTM(update,tele)''), we trained two other sets of meta-optimizers as baselines. 
Algorithm \ref{alg:learning-to-teleport} summarizes the training procedure. 
The vanilla gradient descent baseline (``GD'') uses the largest learning rate that does not lead to divergence ($3\times 10^{-4}$). 
The second baseline (``LSTM(update)'') learns the update $f_t$ only and does not perform teleportation ($g_t=I, \forall t$). 
The third baseline (``LSTM(lr,tele)'') learns the group element $g_t$ and the learning rate used to perform gradient descent instead of the update $f_t$.
We keep training until adding more training trajectories does not improve convergence rate.
We use 700 training trajectories for our approach, 600 for the second baseline, and 30 for the third baseline.

\textbf{Results.~~}
By learning both the local update $f_t$ and non-local transformation $g_t$, our meta-optimizer successfully learns to learn faster.
Figure \ref{fig:meta-learning} shows the improvement of our approach from the previous meta-learning method, which only learns $f_t$. Compared to the baselines, learning the two types of updates together (``LSTM(update,tele)'') achieves better convergence rate than learning them separately. Additionally, learning the group element $g_t$ eliminates the need for performing gradient ascent on the group manifold and reduces hyperparameter tuning for teleportation.
As an example of successful integration of teleportation into existing optimization algorithms, this toy experiment demonstrates the flexibility and promising applications of teleportation. 

\begin{minipage}[c]{0.6\linewidth}
\begin{algorithm}[H]
   \caption{Learning to teleport}
   \label{alg:learning-to-teleport}
    \begin{algorithmic}
       \STATE {\bfseries Input:} Loss function $\loss$, learning rate $\eta$, number of epochs $T$, LSTM models $m_1, m_2$ with initial parameters $\phi_1, \phi_2$, unroll step $t_{unroll}$.
       \STATE {\bfseries Output:} Trained parameters $\phi_1$ and $\phi_2$.
       % \REPEAT
       % \STATE Initialize $noChange = true$.
       % \FOR{$i=1$ {\bfseries to} $m-1$}
       \FOR{each training initialization}
       \FOR{$t = 1$ {\bfseries to} $T$}
       \STATE $f_t, h_{1_{t+1}} = m_1(\nabla_t, h_{1_t}, \phi_1)$
       \STATE $g_t, h_{2_{t+1}} = m_2(\nabla_t, h_{2_t}, \phi_2)$
       \STATE $\vw \leftarrow g_t \cdot (\vw + f_t)$
       \IF{$t \mod t_{unroll} = 0$}
       \STATE update $\phi_1, \phi_2$ by back-propogation from the accumulated loss $\sum_{i=t-t_{unroll}}^t \loss(\vw_i)$
       \ENDIF
       \ENDFOR
       \ENDFOR
       % \UNTIL{$noChange$ is $true$}
    \end{algorithmic}
    \end{algorithm}
\end{minipage}
\hspace{0.05\linewidth}
\begin{minipage}[c]{0.35\linewidth}
    \begin{figure}[H]
       \includegraphics[width=1.0\columnwidth]{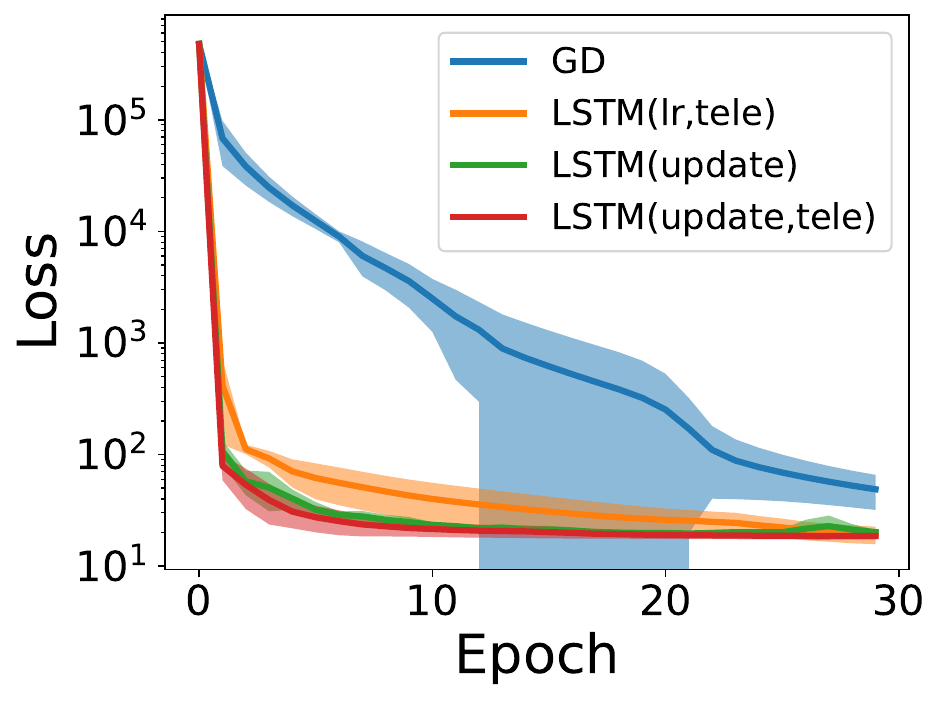}
        \caption{Performance of the trained meta-optimizer on the test set. Learning both local update $f_t$ and nonlocal transformation $g_t$ results in better convergence rate than learning only local updates or learning only teleportation.  }
        \label{fig:meta-learning}
    \end{figure}
\end{minipage}\hfill

% \subsection{Other ideas}
% \begin{itemize}
%     \item learn both group element $g$ and update $f$, using update rule $\vw_{t+1} = g_t \cdot (\vw_t - f_t)$
%     \item learn a gating function to decide whether to teleport (conditional programming)
%     \item similarly, learn which layer to teleport
%     \item add a term in loss function to bias trajectory towards flatter minima
%     \item Learning to learn to teleport by gradient ascent by gradient ascent
%     \item learn a group action or vector field that keeps loss unchanged, so that we can work with any network
%     \item related to projective gradient descent? (extending update rule from simple addition to something else)
% \end{itemize}
\section{Discussion}

% In this section, we provide theoretical guarantees of teleportation. We show that if we teleport to a critical point on the level set at every step, stochastic descent converges to a basin of stationary points. Building on the equivalence of gradient descent update and Newton's method update at critical points, we show that teleportation leads to a mixture of linear and quadratic convergence. Lastly, we show that in certain loss functions, one teleportation guarantees optimality of the entire gradient flow trajectory. 

% We have shown that teleportation improves the convergence of SGD and is equal to taking a dampened Newton step. 

% Distribution of conserved quantities is limited under common initialization schemes \cite{zhao2022symmetries}.  % a parametrization of the subspace of the minima reachable by group actions

% Gradient descent only explores a small portion of the parameter space \cite{gur2018gradient}. Teleportation allows gradient descent to take large steps and explore more globally across the parameter space. 

Teleportation is a powerful tool to search in the loss level sets for parameters with desired properties.
We provide theoretical guarantees that teleportation accelerates the convergence rate of SGD.
% We have also shown that this approach improves the convergence and generalization of SGD and is compatible with other optimiziation algorithms. 
Using concepts in symmetry, we propose a novel notion of curvature and show that incorporating additional teleportation objectives such as changing the curvatures can be beneficial to generalization. 
The close relationship between symmetry and optimization opens up a number of exciting opportunities. Exploring other objectives in teleportation appears to be an interesting future direction. Other possible applications include extending teleportation to different architectures, such as convolutional or graph neural networks, and to different algorithms, such as sampling-based optimization. 

The empirical results linking sharpness and curvatures to generalization are intriguing. However, the theoretical origin of their relation remains unclear. In particular, a precise description of how the loss landscape changes under distribution shifts is not known.
More investigation of the correlation between curvatures and generalization will help teleportation to further improve generalization and take us a step closer to understanding the loss landscape.

% Is teleportation equivalent to changing learning rate?
% In previous comparison of gradient descent with and without teleportation, a fixed learning rate is used. 
% This comparison, however, does not distinguish the source of improvement.
% To answer the question of whether teleportation is equivalent to just increasing stepsize, we use Polyak stepsize (Algorithm \ref{alg:teleport-polyak}). Figure \ref{fig:teleport-polyak} shows that when using Polyak stepsize, teleportation does not improve the convergence rate. 

\section*{Acknowledgements}
This work was supported in part by Army-ECASE award W911NF-23-1-0231, the U.S. Department Of Energy, Office of Science under \#DE-SC0022255, IARPA HAYSTAC Program, CDC-RFA-FT-23-0069, NSF Grants \#2205093, \#2146343, \#2134274, \#2107256, and \#2134178.
% R. Walters is supported by  NSF grants \#2107256 and \#2134178.

\bibliography{iclr2024_conference}

\begin{thebibliography}{43}
\providecommand{\natexlab}[1]{#1}
\providecommand{\url}[1]{\texttt{#1}}
\expandafter\ifx\csname urlstyle\endcsname\relax
  \providecommand{\doi}[1]{doi: #1}\else
  \providecommand{\doi}{doi: \begingroup \urlstyle{rm}\Url}\fi

\bibitem[Al{\'e}ssio(2012)]{alessio2012formulas}
Osmar Al{\'e}ssio.
\newblock Formulas for second curvature, third curvature, normal curvature, first geodesic curvature and first geodesic torsion of implicit curve in n-dimensions.
\newblock \emph{Computer Aided Geometric Design}, 29\penalty0 (3-4):\penalty0 189--201, 2012.

\bibitem[Andriushchenko et~al.(2023)Andriushchenko, Croce, M{\"u}ller, Hein, and Flammarion]{andriushchenko2023modern}
Maksym Andriushchenko, Francesco Croce, Maximilian M{\"u}ller, Matthias Hein, and Nicolas Flammarion.
\newblock A modern look at the relationship between sharpness and generalization.
\newblock \emph{International Conference on Machine Learning}, 2023.

\bibitem[Andrychowicz et~al.(2016)Andrychowicz, Denil, Gomez, Hoffman, Pfau, Schaul, Shillingford, and De~Freitas]{andrychowicz2016learning}
Marcin Andrychowicz, Misha Denil, Sergio Gomez, Matthew~W Hoffman, David Pfau, Tom Schaul, Brendan Shillingford, and Nando De~Freitas.
\newblock Learning to learn by gradient descent by gradient descent.
\newblock \emph{Advances in Neural Information Processing Systems}, 29, 2016.

\bibitem[Armenta \& Jodoin(2021)Armenta and Jodoin]{armenta2021representation}
Marco Armenta and Pierre-Marc Jodoin.
\newblock The representation theory of neural networks.
\newblock \emph{Mathematics}, 9\penalty0 (24), 2021.
\newblock ISSN 2227-7390.

\bibitem[Armenta et~al.(2023)Armenta, Judge, Painchaud, Skandarani, Lemaire, Gibeau~Sanchez, Spino, and Jodoin]{armenta2023neural}
Marco Armenta, Thierry Judge, Nathan Painchaud, Youssef Skandarani, Carl Lemaire, Gabriel Gibeau~Sanchez, Philippe Spino, and Pierre-Marc Jodoin.
\newblock Neural teleportation.
\newblock \emph{Mathematics}, 11\penalty0 (2):\penalty0 480, 2023.

\bibitem[Badrinarayanan et~al.(2015)Badrinarayanan, Mishra, and Cipolla]{badrinarayanan2015symmetry}
Vijay Badrinarayanan, Bamdev Mishra, and Roberto Cipolla.
\newblock Symmetry-invariant optimization in deep networks.
\newblock \emph{arXiv preprint arXiv:1511.01754}, 2015.

\bibitem[Bouthillier et~al.(2021)Bouthillier, Delaunay, Bronzi, Trofimov, Nichyporuk, Szeto, Mohammadi~Sepahvand, Raff, Madan, Voleti, et~al.]{bouthillier2021accounting}
Xavier Bouthillier, Pierre Delaunay, Mirko Bronzi, Assya Trofimov, Brennan Nichyporuk, Justin Szeto, Nazanin Mohammadi~Sepahvand, Edward Raff, Kanika Madan, Vikram Voleti, et~al.
\newblock Accounting for variance in machine learning benchmarks.
\newblock \emph{Proceedings of Machine Learning and Systems}, 3:\penalty0 747--769, 2021.

\bibitem[Chaudhari et~al.(2017)Chaudhari, Choromanska, Soatto, LeCun, Baldassi, Borgs, Chayes, Sagun, and Zecchina]{chaudhari2017entropy}
Pratik Chaudhari, Anna Choromanska, Stefano Soatto, Yann LeCun, Carlo Baldassi, Christian Borgs, Jennifer Chayes, Levent Sagun, and Riccardo Zecchina.
\newblock Entropy-sgd: Biasing gradient descent into wide valleys.
\newblock \emph{International Conference on Learning Representations}, 2017.

\bibitem[Deng(2012)]{deng2012mnist}
Li~Deng.
\newblock The mnist database of handwritten digit images for machine learning research [best of the web].
\newblock \emph{IEEE signal processing magazine}, 29\penalty0 (6):\penalty0 141--142, 2012.

\bibitem[Ding et~al.(2022)Ding, Drusvyatskiy, Fazel, and Harchaoui]{ding2022flat}
Lijun Ding, Dmitriy Drusvyatskiy, Maryam Fazel, and Zaid Harchaoui.
\newblock Flat minima generalize for low-rank matrix recovery.
\newblock \emph{arXiv preprint arXiv:2203.03756}, 2022.

\bibitem[Dinh et~al.(2017)Dinh, Pascanu, Bengio, and Bengio]{dinh2017sharp}
Laurent Dinh, Razvan Pascanu, Samy Bengio, and Yoshua Bengio.
\newblock Sharp minima can generalize for deep nets.
\newblock In \emph{International Conference on Machine Learning}, pp.\  1019--1028. PMLR, 2017.

\bibitem[Dodge et~al.(2020)Dodge, Ilharco, Schwartz, Farhadi, Hajishirzi, and Smith]{dodge2020fine}
Jesse Dodge, Gabriel Ilharco, Roy Schwartz, Ali Farhadi, Hannaneh Hajishirzi, and Noah Smith.
\newblock Fine-tuning pretrained language models: Weight initializations, data orders, and early stopping.
\newblock \emph{arXiv preprint arXiv:2002.06305}, 2020.

\bibitem[Du et~al.(2018)Du, Hu, and Lee]{du2018algorithmic}
Simon~S Du, Wei Hu, and Jason~D Lee.
\newblock Algorithmic regularization in learning deep homogeneous models: Layers are automatically balanced.
\newblock \emph{Neural Information Processing Systems}, 2018.

\bibitem[Entezari et~al.(2022)Entezari, Sedghi, Saukh, and Neyshabur]{entezari2022role}
Rahim Entezari, Hanie Sedghi, Olga Saukh, and Behnam Neyshabur.
\newblock The role of permutation invariance in linear mode connectivity of neural networks.
\newblock \emph{International Conference on Learning Representations}, 2022.

\bibitem[Finn et~al.(2017)Finn, Abbeel, and Levine]{finn2017model}
Chelsea Finn, Pieter Abbeel, and Sergey Levine.
\newblock Model-agnostic meta-learning for fast adaptation of deep networks.
\newblock In \emph{International conference on machine learning}, pp.\  1126--1135. PMLR, 2017.

\bibitem[Foret et~al.(2021)Foret, Kleiner, Mobahi, and Neyshabur]{foret2021sharpness}
Pierre Foret, Ariel Kleiner, Hossein Mobahi, and Behnam Neyshabur.
\newblock Sharpness-aware minimization for efficiently improving generalization.
\newblock In \emph{International Conference on Learning Representations}, 2021.

\bibitem[Ganev \& Walters(2022)Ganev and Walters]{ganev2022quiver}
Iordan Ganev and Robin Walters.
\newblock Quiver neural networks.
\newblock \emph{arXiv preprint arXiv:2207.12773}, 2022.

\bibitem[Ganev et~al.(2022)Ganev, van Laarhoven, and Walters]{ganev2021universal}
Iordan Ganev, Twan van Laarhoven, and Robin Walters.
\newblock Universal approximation and model compression for radial neural networks.
\newblock \emph{arXiv preprint arXiv:2107.02550v2}, 2022.

\bibitem[Grigsby et~al.(2022)Grigsby, Lindsey, Meyerhoff, and Wu]{grigsby2022functional}
J~Elisenda Grigsby, Kathryn Lindsey, Robert Meyerhoff, and Chenxi Wu.
\newblock Functional dimension of feedforward relu neural networks.
\newblock \emph{arXiv preprint arXiv:2209.04036}, 2022.

\bibitem[Hochreiter \& Schmidhuber(1997)Hochreiter and Schmidhuber]{hochreiter1997flat}
Sepp Hochreiter and J{\"u}rgen Schmidhuber.
\newblock Flat minima.
\newblock \emph{Neural computation}, 9\penalty0 (1):\penalty0 1--42, 1997.

\bibitem[Izmailov et~al.(2018)Izmailov, Podoprikhin, Garipov, Vetrov, and Wilson]{izmailov2018averaging}
Pavel Izmailov, Dmitrii Podoprikhin, Timur Garipov, Dmitry Vetrov, and Andrew~Gordon Wilson.
\newblock Averaging weights leads to wider optima and better generalization.
\newblock \emph{Conference on Uncertainty in Artificial Intelligence}, 2018.

\bibitem[Keskar et~al.(2017)Keskar, Mudigere, Nocedal, Smelyanskiy, and Tang]{keskar2017large}
Nitish~Shirish Keskar, Dheevatsa Mudigere, Jorge Nocedal, Mikhail Smelyanskiy, and Ping Tak~Peter Tang.
\newblock On large-batch training for deep learning: Generalization gap and sharp minima.
\newblock \emph{International Conference on Learning Representations}, 2017.

\bibitem[Kim et~al.(2022)Kim, Li, Hu, and Hospedales]{kim2022fisher}
Minyoung Kim, Da~Li, Shell~X Hu, and Timothy Hospedales.
\newblock Fisher sam: Information geometry and sharpness aware minimisation.
\newblock In \emph{International Conference on Machine Learning}, pp.\  11148--11161. PMLR, 2022.

\bibitem[Krizhevsky et~al.(2009)Krizhevsky, Hinton, et~al.]{krizhevsky2009learning}
Alex Krizhevsky, Geoffrey Hinton, et~al.
\newblock Learning multiple layers of features from tiny images.
\newblock 2009.

\bibitem[Kunin et~al.(2021)Kunin, Sagastuy-Brena, Ganguli, Yamins, and Tanaka]{kunin2021neural}
Daniel Kunin, Javier Sagastuy-Brena, Surya Ganguli, Daniel~LK Yamins, and Hidenori Tanaka.
\newblock Neural mechanics: Symmetry and broken conservation laws in deep learning dynamics.
\newblock In \emph{International Conference on Learning Representations}, 2021.

\bibitem[Kwon et~al.(2021)Kwon, Kim, Park, and Choi]{kwon2021asam}
Jungmin Kwon, Jeongseop Kim, Hyunseo Park, and In~Kwon Choi.
\newblock Asam: Adaptive sharpness-aware minimization for scale-invariant learning of deep neural networks.
\newblock In \emph{International Conference on Machine Learning}, pp.\  5905--5914. PMLR, 2021.

\bibitem[Lee(2013)]{Lee2013}
John~M Lee.
\newblock \emph{Introduction to Smooth Manifolds.}
\newblock Graduate Texts in Mathematics, vol 218. Springer, New York, NY, 2013.

\bibitem[Loshchilov \& Hutter(2017)Loshchilov and Hutter]{loshchilov2016sgdr}
Ilya Loshchilov and Frank Hutter.
\newblock Sgdr: Stochastic gradient descent with warm restarts.
\newblock \emph{International Conference on Learning Representations}, 2017.

\bibitem[Meng et~al.(2019)Meng, Zheng, Zhang, Chen, Ma, and Liu]{meng2019mathcal}
Qi~Meng, Shuxin Zheng, Huishuai Zhang, Wei Chen, Zhi-Ming Ma, and Tie-Yan Liu.
\newblock {$\mathcal{G}$-SGD}: Optimizing relu neural networks in its positively scale-invariant space.
\newblock \emph{International Conference on Learning Representations}, 2019.

\bibitem[Neyshabur et~al.(2015)Neyshabur, Salakhutdinov, and Srebro]{neyshabur2015path-sgd}
Behnam Neyshabur, Russ~R Salakhutdinov, and Nati Srebro.
\newblock {Path-SGD}: Path-normalized optimization in deep neural networks.
\newblock In \emph{Advances in Neural Information Processing Systems}, 2015.

\bibitem[Petzka et~al.(2021)Petzka, Kamp, Adilova, Sminchisescu, and Boley]{petzka2021relative}
Henning Petzka, Michael Kamp, Linara Adilova, Cristian Sminchisescu, and Mario Boley.
\newblock Relative flatness and generalization.
\newblock \emph{35th Conference on Neural Information Processing Systems}, 2021.

\bibitem[Ramasinghe et~al.(2022)Ramasinghe, MacDonald, Farazi, Sartachandran, and Lucey]{ramasinghe2022you}
Sameera Ramasinghe, Lachlan MacDonald, Moshiur Farazi, Hemanth Sartachandran, and Simon Lucey.
\newblock How you start matters for generalization.
\newblock \emph{arXiv preprint arXiv:2206.08558}, 2022.

\bibitem[Sagun et~al.(2017)Sagun, Evci, Guney, Dauphin, and Bottou]{sagun2017empirical}
Levent Sagun, Utku Evci, V~Ugur Guney, Yann Dauphin, and Leon Bottou.
\newblock Empirical analysis of the hessian of over-parametrized neural networks.
\newblock \emph{arXiv preprint arXiv:1706.04454}, 2017.

\bibitem[Shelekhov(2021)]{shelekhov2021curvatures}
Aleksandr~Mikhailovich Shelekhov.
\newblock On the curvatures of a curve in n-dimensional euclidean space.
\newblock \emph{Russian Mathematics}, 65\penalty0 (11):\penalty0 46--58, 2021.

\bibitem[{\c{S}}im{\c{s}}ek et~al.(2021){\c{S}}im{\c{s}}ek, Ged, Jacot, Spadaro, Hongler, Gerstner, and Brea]{simsek2021geometry}
Berfin {\c{S}}im{\c{s}}ek, Fran{\c{c}}ois Ged, Arthur Jacot, Francesco Spadaro, Cl{\'e}ment Hongler, Wulfram Gerstner, and Johanni Brea.
\newblock Geometry of the loss landscape in overparameterized neural networks: Symmetries and invariances.
\newblock In \emph{International Conference on Machine Learning}, pp.\  9722--9732. PMLR, 2021.

\bibitem[Stich(2019)]{stich2019}
Sebastian~U. Stich.
\newblock Unified optimal analysis of the (stochastic) gradient method.
\newblock \emph{CoRR}, 2019.

\bibitem[Tarmoun et~al.(2021)Tarmoun, Franca, Haeffele, and Vidal]{tarmoun2021understanding}
Salma Tarmoun, Guilherme Franca, Benjamin~D Haeffele, and Rene Vidal.
\newblock Understanding the dynamics of gradient flow in overparameterized linear models.
\newblock In \emph{International Conference on Machine Learning}, pp.\  10153--10161. PMLR, 2021.

\bibitem[Van~Laarhoven(2017)]{van2017l2}
Twan Van~Laarhoven.
\newblock L2 regularization versus batch and weight normalization.
\newblock \emph{Advances in Neural Information Processing Systems}, 2017.

\bibitem[Wu et~al.(2017)Wu, Zhu, et~al.]{wu2017towards}
Lei Wu, Zhanxing Zhu, et~al.
\newblock Towards understanding generalization of deep learning: Perspective of loss landscapes.
\newblock \emph{arXiv preprint arXiv:1706.10239}, 2017.

\bibitem[Xiao et~al.(2017)Xiao, Rasul, and Vollgraf]{xiao2017fashion}
Han Xiao, Kashif Rasul, and Roland Vollgraf.
\newblock Fashion-mnist: a novel image dataset for benchmarking machine learning algorithms.
\newblock \emph{arXiv preprint arXiv:1708.07747}, 2017.

\bibitem[Zhao et~al.(2022)Zhao, Dehmamy, Walters, and Yu]{zhao2022symmetry}
Bo~Zhao, Nima Dehmamy, Robin Walters, and Rose Yu.
\newblock Symmetry teleportation for accelerated optimization.
\newblock \emph{Advances in Neural Information Processing Systems}, 2022.

\bibitem[Zhao et~al.(2023)Zhao, Ganev, Walters, Yu, and Dehmamy]{zhao2022symmetries}
Bo~Zhao, Iordan Ganev, Robin Walters, Rose Yu, and Nima Dehmamy.
\newblock Symmetries, flat minima, and the conserved quantities of gradient flow.
\newblock \emph{International Conference on Learning Representations}, 2023.

\bibitem[Zhou et~al.(2020)Zhou, Feng, Ma, Xiong, Hoi, et~al.]{zhou2020towards}
Pan Zhou, Jiashi Feng, Chao Ma, Caiming Xiong, Steven Chu~Hong Hoi, et~al.
\newblock Towards theoretically understanding why sgd generalizes better than adam in deep learning.
\newblock \emph{Advances in Neural Information Processing Systems}, 33:\penalty0 21285--21296, 2020.

\end{thebibliography}
\bibliographystyle{iclr2024_conference}

\newpage
\appendix
% \onecolumn
\section*{APPENDIX}
This appendix contains proofs, experiment setups, as well as additional results and discussions. 
Appendix \ref{sec:smooth-non-convex-proof} through \ref{appendix:one-teleportation-enough} contain proofs for theoretical results in Section \ref{sec:optimization}. 
Appendix \ref{sec:appendix-group} provides details about curves induced by symmetry and the curvature of the minimum.
Appendix \ref{appendix:curvature-generalization} discusses possible theoretical approaches to relate curvatures and generalization. This section also contains experiment details on computing correlations and the algorithm that uses teleportation to change curvature. 
Appendix \ref{appendix:other-algorithms} describes experiment setups and different strategies of integrating teleportation into various optimization algorithms.

The code used for our experiments is available at: \url{https://github.com/Rose-STL-Lab/Teleportation-Optimization}.

\section{Teleportation and SGD}
\label{sec:smooth-non-convex-proof}
This section includes a proof for Theorem \ref{theo:telesgdconv}. Additionally, we discuss the theorem's implication when the loss function is strictly convex.
 
\begin{lemma}[Descent Lemma]\label{lem:smoothdescent}
Let $\cL(\vw,\xi)$ be a  $\beta$--smooth function.  It follows that
\begin{equation}\label{eq:smoothsubopt}
\E{\norm{\nabla\cL(\vw,\xi) }^2} \leq 2 \beta (\cL(\vw) - \cL(\vw^*)) + 2\beta ( \cL(\vw^*) -  \E{\inf_\vw\cL(\vw,\xi)}).
\end{equation}
 \end{lemma}
\begin{proof}
Since  $\cL(w,\xi)$ is smooth we have that
\begin{align} \label{eq:smoothnessfuncstar}
\cL(z,\xi) -\cL(\vw,\xi) &\leq  \dotprod{\nabla \cL(\vw,\xi) , \vz-\vw} +\frac{\beta}{2}\norm{\vz-\vw}^2, \quad  \forall \vz,\vw\in\R^d.
\end{align}
By inserting
\[\vz = \vw - \frac{1}{\beta} \nabla \cL(\vw,\xi) \]
 into~\eqref{eq:smoothnessfuncstar} we have that
\begin{align} \label{eq:smoothnessfuncstar2}
 \cL\big(\vw- (1/\beta) \nabla \cL(\vw,\xi), \xi\big) \leq  \cL(\vw,\xi) -\frac{1}{2\beta} \norm{\nabla \cL(\vw,\xi)}^2 .
\end{align}
Re-arranging we have that
\begin{align*}
      \cL(\vw^*,\xi) -  \cL(\vw,\xi)& =   \cL(\vw^*,\xi) -  \inf_\vw \cL(\vw,\xi) +\inf_\vw \cL(\vw,\xi) -  \cL(\vw,\xi)\\
     & \leq  \cL(\vw^*,\xi) -  \inf_\vw \cL(\vw,\xi) + \cL\big(\vw- (1/\beta) \nabla \cL(\vw,\xi), \xi\big)  - \cL(\vw,\xi)  
      \\ & \overset{\eqref{eq:smoothnessfuncstar2}}{\leq}  \cL(\vw^*,\xi) -  \inf_\vw \cL(\vw,\xi) -\frac{1}{2\beta} \norm{\nabla  \cL(\vw,\xi)}^2,
\end{align*} 
where the first inequality follows because $\inf_\vw \cL(\vw,\xi) \leq  \cL(\vw,\xi), \forall \vw.$
Re-arranging the above and taking expectation gives
\begin{align*}
\E{\norm{\nabla  \cL(\vw,\xi)}^2} & \leq 2 \E{\beta (   \cL(\vw^*,\xi) -  \inf_\vw \cL(\vw,\xi)  +  \cL(\vw,\xi) -  \cL(\vw^*,\xi))} \nonumber\\
& \leq  2 \beta\E{   \cL(\vw^*,\xi) -  \inf_\vw \cL(\vw,\xi) +  \cL(\vw,\xi) -  \cL(\vw^*,\xi) } \\
&  \leq  2 \beta (  \cL(\vw)-  \cL(\vw^*)) + 2\beta( \cL(\vw^*)- \E{ \inf_\vw \cL(\vw,\xi)}).  %\qed
\end{align*}
\end{proof}

At each iteration $t \in \mathbb{N}^+$ in SGD, we choose a group element  $g^t \in G$ and 
use teleportation  before each gradient step as follows
% earches for the best gradient descent trajectory by teleporting parameters  $w$ to a different point  using a group action $g^t \in G$.
\begin{equation}%\label{eq:stochgradg}
%x^{t+1}=\mbox{proj}_D\left(g^t\cdot(x^t - \alpha_t g(x^t,\xi^t))\right),
\vw^{t+1}=g^t\cdot \vw^t - \eta \nabla \cL(g^t\cdot \vw^t,\xi^t).
\end{equation}
Here $ \eta$ is a learning rate,  $ \nabla \cL(\vw^t,\xi^t)$ is a gradient of  $\cL(\vw^t,\xi^t) $ with respect to the parameters $\vw$,
 and $\xi^t \sim \mathcal{D}$ is a mini-batch of data sampled i.i.d at each iteration.
\begin{manualtheorem}{\ref{theo:telesgdconv}}%[title]
Let $\cL(\vw, \xi)$ be $\beta$--smooth and let
\[\sigma^2 \eqdef  \cL(\vw^*) -\E{\inf_\vw \cL(\vw,\xi)}.\]
Consider the iterates $\vw^t$ given by~\eqref{eq:stochgradg} where
\begin{equation}
g^t \in \arg\max_{g\in G} \norm{ \nabla\cL(g\cdot \vw^t)}^2.
\end{equation}
If  $\eta = \frac{1}{\beta \sqrt{ T-1}}$ then 
% \begin{align}
% \min_{t=0,\ldots, T-1} \E{\max_{g\in G} \norm{ \nabla\cL(g\cdot w^t)}^2  }
% & \leq  \frac{\beta^2}{T} \left( \cL(w^{0})-\cL^* + \sigma^2\right)    .
% \end{align}
\begin{align}%\label{eq:telesgdconv}
\min_{t=0,\ldots, T-1} \E{ \max_{g\in G}  \norm{ \nabla\cL(g\cdot \vw^t)}^2  }
& \leq \frac{2\beta}{ \sqrt{ T-1}}    \E{\cL(\vw^{0})-\cL(\vw^*)} + \frac{ \beta \sigma^2}{ \sqrt{ T-1}} .
\end{align}
% Consequently the complexity of reaching an $\epsilon$--stationary point is $\mathcal{O}1/\epsilon^2.$
% Furthermore,  if interpolation holds,  that is if $\sigma^2 =0 $ then  
\end{manualtheorem}

\begin{proof}
First note that if $\cL(\vw,\xi)$ is $\beta$--smooth,  then
 $\mathcal{L}(\vw)$ is  also a $\beta$--smooth function,  that is 
 \begin{equation} 
\label{eq:smoothnessfuncb}
 \cL(\vz) - \cL(\vw) - \langle \nabla \cL(\vw), \vz-\vw \rangle \leq \frac{\beta}{2} \norm{\vz-\vw}^2.
 \end{equation}
Using~\eqref{eq:stochgradg} with $\vz=\vw^{t+1}$ and $\vw=g^t\cdot \vw^t$, together with~\eqref{eq:smoothnessfuncb} and the fact that the group action preserves loss, we have that
% ~\eqref{eq:invariant}
\begin{align}
\cL(\vw^{t+1}) & \leq \cL(g^t\cdot \vw^{t}) +\dotprod{\nabla  \cL(g^t\cdot \vw^t), \vw^{t+1} -g^t\cdot \vw^t} +  \frac{\beta}{2} \norm{\vw^{t+1}- g^t\cdot \vw^t}^2\\
& = \cL(\vw^{t}) -\eta_t  \dotprod{ \nabla\cL(g^t\cdot \vw^t), \nabla \cL(g^t\cdot \vw^t,\xi^t)} +  \frac{\beta \eta_t^2}{2} \norm{ \nabla \cL(g^t\cdot \vw^t,\xi^t)}^2.
\end{align}
 Taking expectation conditioned on $\vw^t$, we have that
 \begin{align}
\EE{t}{\cL(\vw^{t+1})}
& \leq \cL(\vw^{t}) -\eta_t  \norm{ \nabla\cL(g^t\cdot \vw^t)}^2 +  \frac{\beta \eta_t^2}{2} \EE{t}{\norm{\nabla \cL(g^t\cdot \vw^t,\xi^t)}^2}.
\end{align}

Now since $\cL(\vw, \xi)$ is $\beta$--smooth,  from Lemma~\ref{lem:smoothdescent} above we have that
% \rob{Use expected smoothness bound, and allow it to depend on teleportation? See SGD General Analysis, Gower et. al. or even better "Better results in the non-convex world SGD" by Khaled and Richtarik}
\begin{equation}\label{eq:smoothsuboptmain}
\E{\norm{\nabla \cL(\vw,\xi)}^2} \leq 2 \beta (\cL(\vw) -\cL(\vw^*)) + 2\beta (\cL(\vw^*) -\E{\inf_\vw \cL(\vw,\xi)})
\end{equation}
Using~\eqref{eq:smoothsuboptmain} with $\vw= g^t \circ \vw^t$ we have that
 \begin{align}
\EE{t}{\cL(\vw^{t+1})}
\leq \cL(\vw^{t}) & -\eta_t  \norm{ \nabla\cL(g^t\cdot \vw^t)}^2 \cr
& + \beta^2 \eta_t^2 \left(   \cL(g^t \cdot \vw^t) - \cL(\vw^*) +\cL(\vw^*) - \E{ \inf_\vw \cL(\vw,\xi)}\right).
\end{align}
Using that $ \cL(g^t \cdot \vw^t) =  \cL( \vw^t),$ taking full expectation and re-arranging terms gives

 \begin{align}\label{eq:temsplo8xs5}
\eta_t \E{ \norm{ \nabla\cL(g^t\cdot \vw^t)}^2 }
& \leq  (1+\beta^2 \eta_t^2 )\E{\cL(\vw^{t})-\cL^*} -\E{\cL(\vw^{t+1}) -\cL^*} +  \beta^2 \eta_t^2  \sigma^2.
\end{align}
Now we use a re-weighting trick introduced in~\cite{stich2019}.  Let $\alpha_t >0$ be a sequence such that $\alpha_t (1+\beta^2 \eta_t^2 ) = \alpha_{t-1} $.  Consequently if $\alpha_{-1} =1$  then $\alpha_t = (1+\beta^2 \eta_t^2 )^{-(t+1)}$ .  Multiplying by both sides of~\eqref{eq:temsplo8xs5} by $\alpha_t$ thus gives
 \begin{align}\label{eq:temsplo8xs56}
\alpha_t  \eta_t  \E{ \norm{ \nabla\cL(g^t\cdot \vw^t)}^2 }
& \leq  \alpha_{t-1} \E{\cL(\vw^{t})-\cL^*} -\alpha_t  \E{\cL(\vw^{t+1}) -\cL^*} + \alpha_t  \beta^2 \eta_t^2  \sigma^2.
\end{align}
Summing up from $t=0, \ldots, T-1$,  and using telescopic cancellation,  gives
 \begin{align}\label{eq:temsplo8xs567}
\sum_{t=0}^{T-1}\alpha_t  \eta_t  \E{ \norm{ \nabla\cL(g^t\cdot \vw^t)}^2 }
& \leq   \E{\cL(\vw^{0})-\cL^*} +  \beta^2 \sigma^2 \sum_{t=0}^{T-1} \alpha_t \eta_t^2  
\end{align}
Let $A =\sum_{t=0}^{T-1}\alpha_t  \eta_t.$ Dividing both sides by $A$ gives
\begin{align}
\min_{t=0,\ldots, T-1} \E{  \norm{ \nabla\cL(g^t\cdot \vw^t)}^2 } & \leq  \frac{1}{\sum_{t=0}^{T-1}\alpha_t  \eta_t}\sum_{t=0}^{T-1}\alpha_t  \eta_t  \norm{ \nabla\cL(g^t\cdot \vw^t)}^2  \nonumber \\
& \leq\frac{  \E{\cL(\vw^{0})-\cL^*} +  \beta^2 \sigma^2 \sum_{t=0}^{T-1} \alpha_t \eta_t^2  }{\sum_{t=0}^{T-1}\alpha_t  \eta_t}.\label{eq:tempnkiz8ehz}
\end{align}
Finally,  if $\eta_t  \equiv \eta$ then 
\begin{align}
 \sum_{t=0}^{T-1}\alpha_t  \eta_t & = \eta  \sum_{t=0}^{T-1} (1+\beta^2 \eta_t^2 )^{-(t+1)} 
= \frac{\eta}{1+\beta^2 \eta^2}  \frac{1- (1+\beta^2 \eta^2)^{-T}}{1- (1+\beta^2 \eta^2)^{-1}}  \\
& = \frac{1- (1+\beta^2 \eta^2)^{-T}}{\beta^2 \eta} \label{eq:temionoxeint}
% \\
% &= \frac{\eta}{1+\beta^2 \eta^2}  \frac{1+\beta^2 \eta^2}{1+\beta^2 \eta^2- 1}  = \frac{\eta}{\beta^2 \eta^2} = \frac{1}{\beta^2 \eta}. 
\end{align}

To bound the term with the $-T$ power,  we use that
\[(1+\beta^2 \eta^2)^{-T}  \leq \frac{1}{2} \quad \implies \quad \frac{\log(2)}{\log(1+\beta^2 \eta^2)} \leq  T.\]
To simplify the above expression we can use
\[\frac{x}{1+x}\leq \log(1+x) \leq x, \quad \mbox{ for }x \geq -1,\]
thus
\[\frac{\log(2)}{\log(1+\beta^2 \eta^2)} \leq  \frac{1+\beta^2 \eta^2}{ \beta^2 \eta^2} \leq T.\]
Using the above
we have that
\begin{align*}
 \sum_{t=0}^{T-1}\alpha_t  \eta_t & \geq
  \frac{1}{2\beta^2 \eta}, \quad \mbox{for } T \geq   \frac{1+\beta^2 \eta^2}{ \beta^2 \eta^2} 
 \end{align*}
Using this lower bound in~\eqref{eq:tempnkiz8ehz} gives
\begin{align*}
\min_{t=0,\ldots, T-1}  \E{ \norm{ \nabla\cL(g^t\cdot \vw^t)}^2  }
& \leq 2\beta^2 \eta   \E{\cL(\vw^{0})-\cL^*} + \eta \beta^2 \sigma^2,
\quad \mbox{for }T \geq   \frac{1+\beta^2 \eta^2}{ \beta^2 \eta^2} .
\end{align*}
Now note that
\[T \geq  \frac{1+\beta^2 \eta^2}{  \beta^2 \eta^2}  \Leftrightarrow \beta^2 \eta^2 (T-1) \geq   1
\Leftrightarrow  \eta \geq \frac{1}{\beta \sqrt{ (T-1)}}.\]
Thus finally setting  $ \eta = \frac{1}{\beta \sqrt{ T-1}}$ gives the result~\eqref{eq:telesgdconv}. 

% \rob{**thinking ****}
% If interpolation holds then $\sigma^2=0$ and we have from~\eqref{eq:temionoxeint} and 
% ~\eqref{eq:tempnkiz8ehz} that
% \begin{align}
% \min_{t=0,\ldots, T-1} \E{  \norm{ \nabla\cL(g^t\cdot w^t)}^2 } 
% & \leq \frac{\beta^2 \eta}{1- (1+\beta^2 \eta^2)^{-T}} \E{\cL(w^{0})-\cL^*}  .\label{eq:tempnkiz8ehz434}
% \end{align}
% Now choosing $\eta = \frac{1}{\beta (T-1)}$ we have that
% \[ \frac{\beta^2 \eta}{1- (1+\beta^2 \eta^2)^{-T}} =
% \frac{1}{T-1}\frac{1}{1- (1+\frac{1}{T-1})^{-T}} =
% \frac{1}{T-1}\frac{1}{1- (\frac{T}{T-1})^{-T}}
% \] 
% \qed 

\end{proof}

\out{
% \rob{Being new material}

\begin{manualtheorem}{\ref{theo:error-bound} (Error Bound)}
    Consider the setting of Theorem~\ref{theo:telesgdconv}. If in addition to $\beta$--smoothness we also assume that the $\mu$-PL (Polyak-\L ojasiewicz) inequality holds, that is  %Ł
    \begin{eqnarray}\label{eq:PL}
      \loss(w)-\inf \loss \leq \frac{1}{2 \mu} \|\nabla \loss(w)\|^2,
    \end{eqnarray}
   then the  $w^t$ iterates given by~\eqref{eq:stochgradg} with  learning rate $\eta = \frac{1}{\beta \sqrt{ T-1}}$ 
converge according to
\begin{align}\label{eq:PLconv}
\min_{t=0,\ldots, T-1} \E{  \max_{g\in G} \norm{g\circ w^t- w^*(g\circ w^t)}  }
& \leq \frac{2\beta}{ \mu}\frac{1}{\sqrt{ T-1}}    \E{\cL(w^{0})-\cL(w^*)} + \frac{\beta}{ \mu}\frac{  \sigma^2}{ \sqrt{ T-1}},
\end{align}
where $w^*(g\circ w^t)$ is the projection of $g\circ w^t$ onto to set of minimizers of $\cL(w).$
\end{manualtheorem}
\begin{proof}
From Appendix A in~\cite{PLschmidt} we have that if $\cL$ is $\beta$--smooth, then PL inequality is equivalent to the \emph{error bound} inequality given by
\begin{eqnarray}\label{eq:errorbnd}
   \mu \norm{w-w^*(w)} \leq \|\nabla \loss(w)\|,
\end{eqnarray}
where $w^*(w)$ is the projection of $w$ onto the set of minimizers of $\cL(w).$
    The proof now follows by directly applying the above     
   %  \begin{equation}
   % \mu( \norm{g\circ w-w^* }    \leq  \norm{ \nabla\cL(g\cdot w^t)}.
   %  \end{equation}
   %  Using the above 
    in~\eqref{eq:telesgdconv} and dividing through by $\mu.$
\end{proof}

Our result in~\eqref{eq:PLconv} now shows that, for loss functions that are smooth and are $\mu$--PL, we have that the expected loss converges to the optimal loss for every point on the orbit of $w^t$. Furthermore, assuming the loss function is $\mu$--PL can be reasonable for neural networks that are overparametrized.
Indeed, in~\cite{LIU202285} the authors showed that, for overparametrized neural networks, when the task is regression, the loss function can often be $\mu$--PL.

% \rob{end new matrial}
}

% \begin{manualproposition}{\ref{prop:convex-loss}}
\begin{proposition}
\label{prop:convex-loss}
Assume that $\loss:\R^n \xrightarrow[]{} \R$ is strictly convex and twice continuously differentiable. 
% Assume also that $G$ acts transitively on all loss level sets. 
Assume also that for any two points $\vw_a, \vw_b \in \R^n$ such that $\loss(\vw_a) = \loss(\vw_b)$, there exists a $g \in G$ such that $\vw_a = g \cdot \vw_b$.
At two points $\vw_1, \vw_2 \in \R^n$, if 
$\max_{g\in G}  \norm{ \nabla\cL(g\cdot \vw_1)}^2 = \norm{ \nabla\cL(\vw_2)}^2$,
then $\cL(\vw_1) \leq \cL(\vw_2)$.
% \end{manualproposition}
\end{proposition}

\begin{proof}
Let $S(x)=\{\vw: \loss(\vw) = x\}$ be the level sets of $\loss$, and $X = \{\loss(\vw): \vw \in \R^n\}$ be the image of $\loss$. Since $G$ acts transitively on the level sets of $\loss$, $\max_{g\in G} \norm{ \nabla\cL(g\cdot \vw)}^2 = \max_{\vw \in S(x)} \norm{ \nabla\cL(\vw)}^2$.
To simplify notation, we define a function $F:X \xrightarrow[]{} \R$, $F(x) = \max_{\vw \in S(x)} \norm{ \nabla\cL(\vw)}^2$.
Since $\nabla \loss(\vw)$ is continuously differentiable, the directional derivative of $F$ is defined.
%
% Proof that $F'$ is defined:
% Let $w^*_x = \argmax_{w \in S(x)} \norm{\grad \loss(w)}^2$. Since $\nabla \loss(w)$ is continuous, $\norm{\grad \loss(w)}^2$ is also continuous. If $w^*_x$ is unique, then 
% \begin{align}
%     \lim_{h \xrightarrow[]{} 0} w^*_{x+h} = w^*_x.
% \end{align}
% In other words, $w^*_{x+h}$ is in the neighborhood of $w^*_x$.
% The derivative of $F$ is the fastest rate of change of $\norm{\grad \loss(w)}^2$ (with respect to $x$), among all maxima of $\norm{\grad \loss(w)}^2$ in $S(x)$:
% \begin{align}
%     F'(x) &= \lim_{h \xrightarrow[]{} 0} \frac{\max_{w \in S(x)} \norm{ \nabla\cL(w)}^2 - \max_{w \in S(x+h)} \norm{ \nabla\cL(w)}^2}{h} \cr
%     &= \max_{w \in w^*_x} \br{
%     \max_{v \in \R^n:~ v^T \grad \loss(w) = 1}
%     v \cdot \frac{\partial}{\partial w} \norm{\grad \loss(w)}^2
%     }.
% \end{align} 
% The constraints on $v$ ensures that the rate of change of $\loss$ along $v$ is 1.
%
Additionally, since $\loss$ is continuous and its domain $\R^n$ is connected, its image $X$ is also connected. 
% https://www.google.com/search?q=Continuous+image+of+connected+space+is+connected
This means that for any $\vw_1, \vw_2 \in \R^n$ and $\min(\loss(\vw_1), \loss(\vw_2)) \leq y \leq \max(\loss(\vw_1), \loss(\vw_2))$, there exists a $\vw_3 \in \R^n$ such that $\loss(\vw_3) = y$.

Next, we show that $F(\cdot)$ is strictly increasing by contradiction.

Suppose that $\cL(\vw_1) < \cL(\vw_2)$ and $F(\loss(\vw_1)) \geq F(\loss(\vw_2))$. By the mean value theorem, there exists a $\vw_3$ such that $\loss(\vw_1) < \loss(\vw_3) < \loss(\vw_2)$ and the directional derivative of $F$ in the direction towards $\loss(\vw_2)$ is non-positive: $\partial_{\loss(\vw_2)-\loss(\vw_3)} F(\loss(\vw_3)) \leq 0$. 
% $\left. \frac{\partial F(x)}{\partial x} \right|_{\loss(w_3)}  = \frac{F(\loss(w_2)) - F(\loss(w_1))}{\loss(w_2) - \loss(w_1)} \leq 0$.
Let $\vw_3^* \in \argmax_{\vw \in S(\loss(\vw_3))} \norm{ \nabla\cL(\vw)}^2$ be a point that has the largest gradient norm in $S(\loss(\vw_3))$. 
Then at $\vw_3^*$, $\norm{ \nabla\cL}^2$ cannot increase along the gradient direction. % add proof?
However, this means
\begin{align}
\label{eq:convex-contradiction}
    \grad \loss(\vw_3^*) \cdot \frac{\partial }{\partial \vw}\norm{\grad \loss(\vw_3^*)}^2 
    = \grad \loss(\vw_3^*)^T H \grad \loss(\vw_3^*)
    \leq 0.
\end{align}
Since we assumed that $\loss$ is convex and $\loss(\vw_3^*)$ is not a minimum ($\loss(\vw_3^*) > \loss(\vw_1)$), we have that $\grad \loss(\vw_3^*) \neq 0$.
Therefore, \eqref{eq:convex-contradiction} contradicts with $\loss$ being strictly convex, and we have $F(\loss(\vw_1)) < F(\loss(\vw_2))$.

We have shown that $\cL(\vw_1) < \cL(\vw_2)$ implies $F(\loss(\vw_1)) < F(\loss(\vw_2))$. 
Taking the contrapositive and switching $\vw_1$ and $\vw_2$, $F(\loss(\vw_1)) \leq F(\loss(\vw_2))$ implies $\cL(\vw_1) \leq \cL(\vw_2)$. 
Equivalently, $\max_{g\in G} \norm{ \nabla\cL(g\cdot \vw_1)}^2 \leq \max_{g\in G} \norm{ \nabla\cL(g\cdot \vw_2)}^2$ implies that $\cL(\vw_1) \leq \cL(\vw_2)$. 

Finally, since
\begin{align}
    \max_{g\in G}  \norm{ \nabla\cL(g\cdot \vw_1)}^2 
    = \norm{ \nabla\cL(\vw_2)}^2 %\cr
    \leq \max_{g\in G} \norm{ \nabla\cL(g\cdot \vw_2)}^2,
\end{align}
we have $\cL(\vw_1) \leq \cL(\vw_2)$.
\end{proof}

% \begin{corollary}
% \label{corollary:convex-increasing-grad}
%     Let $\loss$ be twice continuously differentiable and convex. Let $S(x)=\{w: \loss(w) = x\}$ be a level set of $\loss$. If $\max_{w \in S(x)} \norm{ \nabla\cL(w)}^2 \geq \max_{w \in S(y)} \norm{ \nabla\cL(w)}^2$, then $x \geq y$.
% \end{corollary}
% \bz{$F$ is not differentiable but the directional derivative might be defined. }
% \rob{Maybe we could use Danskin's Theorem. Let $\phi(x,w) :=  \norm{ \nabla\cL(w)}^2 + 1_{\cL(w) =x}$, where $1$ is the indicator function here. Then we would have to restrict the maximization in $w$ to some ball, e.g. $B_D : = \{w \; : \; \| w\|^2 \leq D  \}$. Then finally by Danskin's theorem the function $F(x) = \max_{w\in B_D} \phi(x,w)$ is differentiable if $\phi(x,w)$ is convex in $x$ and if there exists only one solution $w^* = \arg\max_{w\in B_D} \phi(x,w)$. } 

\section{Teleportation and Newton's method}
\label{sec:newtons-proof}

% {\bf Rob: The following Lemma is refinement of Proposition 5.7 in the Neurips teleport paper.}

\begin{lemma}[One step of Newton's Method]
\label{lem:Newtonstep}
Let $f(x)$ be a $\mu$--strongly convex  and $L$--smooth function, that is, we have a global lower bound on the Hessian given by
\begin{equation}\label{eq:strconv}
 L  I    \succeq\nabla^2 f(x) \succeq \mu I, \quad \forall x \in \R^n.
\end{equation}
Furthermore, if the Hessian is also $G$--Lipschitz
\begin{equation}\label{eq:Hesslip}
\|\nabla^2 f(x) - \nabla^2 f(y)\| \leq G \|x-y\|\end{equation}
 then  Newton's method 
\[x^{k+1} = x^k - \lambda_k\nabla^2 f(x^k)^{-1}  \nabla f(x^k) \] 
 has a mixed linear and quadratic convergence according to 
\begin{equation}
 \|x^{k+1} -x^*\| \leq    \frac{G}{2\mu}\|x^{k}-x^*\|^2 + |1-\lambda_k|\frac{L}{2\mu} \|x^{k}-x^*\|.
\end{equation}
%In particular if $\|x^0 -x^*\| \leq \frac{\mu}{L},$
%then for $k\geq 1$ we have that
%\begin{equation}\label{eq:newton2kconv}
% \|x^{k} -x^*\| \leq \frac{1}{2^{2^{k}}} \frac{\mu}{L}.
%\end{equation}
\end{lemma}
\begin{proof}
\begin{align*}
 x^{k+1} -x^* &= x^k -x^* -\lambda_k \nabla^2 f(x^k)^{-1}\left(\nabla f(x^k)-\nabla f(x^*)\right)\\
  &=  x^k -x^* -\lambda_k \nabla^2 f(x^k)^{-1}\int_{s=0}^1 \nabla^2 f(x^k+s(x^{*}-x^k))(x^k-x^{*})ds \quad \mbox{(Mean value theorem)}\\
  &= \nabla^2 f(x^k)^{-1}\int_{s=0}^1\left( \nabla^2 f(x^k) - \lambda_k \nabla^2 f(x^k+s(x^{*}-x^k))\right)(x^k-x^{*}) ds \\
  &= \nabla^2 f(x^k)^{-1}\int_{s=0}^1\left( \nabla^2 f(x^k) - \nabla^2 f(x^k+s(x^{*}-x^k)) \right. \\
  &\hspace{90pt} \left. +(1-\lambda_k)  \nabla^2 f(x^k+s(x^{*}-x^k))   \right)(x^k-x^{*}) ds
\end{align*}
Let $\delta_k :=  \|x^{k+1} -x^*\| $.
Taking norms we have that
\begin{align*}
\delta_{k+1}&\leq 
 \|\nabla^2 f(x^k)^{-1}\|\int_{s=0}^1\left(\| \nabla^2 f(x^k) -  \nabla^2 f(x^k+s(x^{*}-x^k))\| \right. \\
  &\hspace{100pt} \left. +|1-\lambda_k\| \|\nabla^2 f(x^k+s(x^{*}-x^k))\| \right) \delta_k ds\\
 & \overset{\eqref{eq:Hesslip}+\eqref{eq:strconv}}{\leq} \frac{G}{\mu}
 \int_{s=0}^1 s\|x^{k}-x^*\|^2 ds  + |1-\lambda_k|\frac{L}{\mu} \int_{s=0}^1 s\|x^{k}-x^*\| ds \\
 & = \frac{G}{2\mu}\|x^{k}-x^*\|^2 + |1-\lambda_k|\frac{L}{2\mu} \|x^{k}-x^*\|.
\end{align*}
%Finally if $\|x^0 -x^*\| \leq \frac{\mu}{L},$ then
%by induction and assuming that~\eqref{eq:newton2kconv} holds we have that
%\[\|x^{k+1} -x^*\| \quad \leq  \quad  \frac{L}{2\mu}\|x^{k}-x^*\|^2 \quad\leq  \quad \frac{L}{2\mu} \frac{1}{2^{2^{k}}}\frac{1}{2^{2^{k}}} \left(\frac{\mu}{L}\right)^2
%\quad < \quad  \frac{1}{2^{2^{k+1}}}\frac{\mu}{L},\]
%which concludes the induction proof. \qed
\end{proof}
The assumptions on for this proof can be relaxed, since we only require the Hessian is Lipschitz and lower bounded in a $\frac{\mu}{2L}$--ball around $x^*$.

\begin{manualproposition}{\ref{lem:newtondir}}
[Quadratic term in convergence rate]
% [Teleportation gives Newton's direction]%\label{lem:newtondir}
Let $\loss$ be strictly convex and let $w_0\in \R^d$.  Let 
\begin{equation}\label{eq:lemmanewtondir}
 w' \in \argmax_{w\in \R^d} \frac{1}{2}\|\nabla \loss (w)\|^2 \quad \mbox{subject to}\quad \loss(w) =\loss(w_0). 
  \end{equation}
If $ \nabla \loss(w') \neq 0$ then there exists $\lambda_0$ such that
\[ 0\leq  \lambda_0 \leq \lambda_{\max} ( \nabla^2 \loss(w_0))\]
and one step of gradient descent with learning rate $\gamma >0$ gives
\begin{align}
 w_1&= w' -\gamma  \nabla \loss(w') \nonumber \\
 &= w' -\gamma \lambda_0 \nabla^2 \loss(w')^{-1} \nabla \loss(w').
\end{align}
Consequently,  letting $w' = g_0 \circ w_0$, and if  $\gamma \leq \frac{1}{\lambda_0}$ then under the assumptions of Lemma~\ref{lem:Newtonstep} we have that
\[ \|w_1-w^*\| \leq    \frac{G}{2\mu}\|g_0 \circ w_0-x^*\|^2 + |1-\gamma\lambda_0|\frac{L}{2\mu} \|g_0 \circ w_0-w^*\|.\]
\end{manualproposition}
\begin{proof}
The Lagrangian associated to~\eqref{eq:lemmanewtondir} is given by
\[L(w, \lambda ) =\frac{1}{2}\|\nabla \loss (w)\|^2 + \lambda (\loss(w_0) -\loss(w)). \]
Taking the derivative in $w$ and setting it to zero gives
\begin{equation}\label{eq:tempo8zh8hz4}
 \nabla_w L(w, \lambda_0 ) = 0 \; \implies \;  \nabla^2 \loss(w) \nabla \loss (w) - \lambda_0 \nabla \loss(w) =0. 
\end{equation}
Re-arranging we have that
\[ \nabla \loss (w)  = \lambda_0  \nabla^2 \loss(w)^{-1} \nabla \loss (w). \]
If $ \nabla \loss(w') \neq 0 $ then from the above we have that
\[ \|\nabla \loss (w) \|^2 = \lambda_0  \nabla \loss (w) ^\top   \nabla^2 \loss(w)^{-1} \nabla \loss (w) > 0.  \]
Since $\nabla^2 \loss(w)^{-1} $ is positive definite we have that  $ \nabla \loss (w) ^\top   \nabla^2 \loss(w)^{-1} \nabla \loss (w) \geq 0$,  and consequently 
$\lambda_0 >0.$ Finally from~\eqref{eq:tempo8zh8hz4} we have that $\lambda_0$ is an eigenvalue of $\nabla^2 \loss(w)$ and thus it must be smaller or equal to the largest eigenvalue of $\nabla^2 \loss(w)$.

%We could further characterize $\lambda_0$ using~\eqref{eq:tempo8zh8hz4} from which we have
%\[ (\nabla^2 \loss(w_0)-\lambda_0 I) \nabla \loss (w) = \]
%
%\begin{align*}
% \lambda_0  &= \frac{ \|\nabla \loss (w) \|^2 }{ \nabla \loss (w) ^\top   \nabla^2 \loss(w_0)^{-1} \nabla \loss (w)} \; \leq \; \max_{v} \frac{\|v\|^2}{v^\top  \nabla^2 \loss(w_0)^{-1} v} \\
% &= \frac{1}{\min_v  \frac{v^\top  \nabla^2 \loss(w_0)^{-1} v}{\|v\|^2}} = \frac{1}{\lambda_{\min}( \nabla^2 \loss(w_0)^{-1}) } \\
% & = \frac{1}{\lambda_{\max} ( \nabla^2 \loss(w_0))^{-1}} =  \lambda_{\max} ( \nabla^2 \loss(w_0)).
%\end{align*}

\end{proof}

\out{
\subsection{Quadratics}
\bz{Remove?}

Now consider the setting where $\loss(w) = \frac{1}{2}w^\top A w$ where $A$ is positive definite.
\begin{lemma}[Exact update]
Let $A\in \R^{d\times d}$ be positive definite.  Let $U\Lambda U^\top = A$ be the eigendecomposition of $A$.  The solution to
\begin{equation}\label{eq:tempquadproj}
w' = \argmax_{w\in \R^d} \| A w\|^2 \quad \mbox{subject to}\quad  w^\top A w = w_0^\top A w_0
\end{equation} 
is given by $w'=0$ if $w_0^\top A w_0.$ Alternatively if $w_0^\top A w_0 \neq 0$ we have that 
\[w'\; = \;\]
\end{lemma}
\begin{proof}
Let $z :=  \Lambda U^\top w.$
Using the eigendecomposition of $A$ we have taht 
\[ \| A w\|^2 =  \| \Lambda U^\top w \|^2 =  \sum_{i}  z_i^2.\]
Furthermore
\[ \frac{1}{2}w^\top A w  =  \frac{1}{2}w^\top U\Lambda U^\top  w  = \sum_i \frac{1}{\lambda_i} z_i^2. \]
Let $c =  \frac{1}{2}w_0^\top A w_0.$
Consequently~\eqref{eq:tempquadproj} is equivalent to\
\begin{align*}
w' = \argmax_{w\in \R^d}  \sum_{i}  z_i^2.\quad \mbox{subject to}\quad   \sum_i  \frac{1}{\lambda_i} z_i^2= c.
\end{align*}
The Lagrangian of the above is given by
\[ L(w,\tau) =  z_i^2 + \tau ( c - \sum_i \lambda_i^{-1} z_i^2). \]
Differentiating and setting to zero gives
\begin{equation}\label{eqL:zeoinz9e}
\nabla_i L(w,\tau)  =0  \quad \Leftrightarrow \quad  z_i - \tau (\lambda_i^{-1} z_i) =0 \quad \Leftrightarrow \quad    z_i (\lambda_i^{-1} -\tau) =0,
\end{equation}
for $i=1,\ldots, d.$ From which we conclude that 
\begin{equation}\label{eq:cond1seoinipnzr}
 \boxed{z_i =0 \quad \mbox{or} \quad \tau = \lambda_i^{-1}}
\end{equation}
Furthermore multiplying each equation in~\eqref{eqL:zeoinz9e} by $\frac{z_i}{\lambda_i}$ and summing up from $i=1,\ldots, d$ we have that
\[ \sum_i  z_i^2 (\lambda_i^{-1}-\tau)  = \sum_i  \lambda_i^{-1} z_i^2  - \tau \sum_i  z_i^2 = c -\tau \sum_i  z_i^2  =0 .\]
Isolating $\tau$ in the above gives 
\begin{equation}\label{eq:cond2seoinipnzr} \boxed{\sum_i  z_i^2 \;=\;  \frac{c}{  \tau }}.\end{equation}
The conditions~\eqref{eq:cond1seoinipnzr} and~\eqref{eq:cond2seoinipnzr} characterize all the stationary points. The maximum is give by $z$ such that $  \sum_{i} z_i^2$ is large and $z$ satisfies~\eqref{eq:cond1seoinipnzr} and~\eqref{eq:cond2seoinipnzr}.
Due to~\eqref{eq:cond2seoinipnzr} we should choose $\tau$ as small as possible for  $\sum_i  z_i^2$ to be large.  Due to~\eqref{eq:cond2seoinipnzr} this corresponds to choosing $\tau = 1/\lambda_{\max}.$ Consequently
\[z_{i_{\max}}^2 = c \lambda_{\max} \quad \Leftrightarrow \quad  z_{i_{\max}} = \sqrt{c\lambda_{\max}}\]
and $z_i =0$ for $i \neq i_{\max}.$
Consequently
\[w' =U \Lambda^{-1} z =u_{i_{\max}} \lambda_{\max} z_{i_{\max}} =u_{i_{\max}} \lambda_{\max}^{-1} \sqrt{c\lambda_{\max}} = 
 u_{i_{\max}} \sqrt{\frac{c}{\lambda_{\max}}}. \]
\end{proof}
}

\section{Is one teleportation enough to find the optimal trajectory?}
\label{appendix:one-teleportation-enough}

This section contains proofs for the results in Section \ref{sec:one-teleportation-enough}. For readability, we repeat some of the definitions here. 

Consider the parameter space $\mathcal{M} = \R^n$.
Let $V: \mathcal{\R}^n \xrightarrow[]{} T\mathcal{\R}^n$ be a vector field on $\R^n$, where $T\mathcal{\R}^n$ denotes the associated tangent bundle. We will write $V = v^i \frac{\ro }{\ro w^i}$ using the component functions $v^i: \R^n \xrightarrow[]{} \R$ and coordinates $w^i$. 

% C^\infty(\R^n)
Let $\loss: \mathcal{M} \xrightarrow[]{} \R$ be a smooth loss function. 
Let $G$ be a symmetry group of $\loss$, i.e. $\loss(g \cdot \vw) = \loss(\vw)$ for all $\vw \in \mathcal{M}$ and $g \in G$.
Let $\mathfrak{X}$ be the set of all vector fields on $\mathcal{M}$. 
Let $R = r^i \frac{\ro }{\ro w^i}$, where $r^i = -\frac{\ro \loss}{\ro w_i}$, be the reverse gradient vector field. 
Let $\mathfrak{X}_{\perp} = \{A=a^i \frac{\ro }{\ro w^i} \in \mathfrak{X}|~ a^i \in C^\infty(\mathcal{M}) \text{ and } \sum_i a^i (\vw) 
r^i(\vw) = 0, \forall \vw \in \mathcal{M}\}$ be the set of vector fields orthogonal to $R$. 
% replace with metric?
If $G$ is a Lie group, the infinitesimal action of its Lie algebra $\mathfrak{g}$ defines a set of vector fields $\mathfrak{X}_{\mathfrak{g}} \subseteq \mathfrak{X}_{\perp}$. 

A gradient flow is a curve $\gamma: \R \xrightarrow[]{} \mathcal{M}$ where the velocity is the value of $R$ at each point, i.e. $\gamma'(t) = R_{\gamma(t)}$ for all $t \in \R$. 
The Lie bracket $[A, R]$ defines the derivative of $R$ with respect to $A$. 
To simplify notation, we write $([W, R]\loss)(\vw)=0$ for a set of vector fields $W \subseteq \mathfrak{X}$ when $([A, R]\loss)(\vw)=0$ for all  $A \in W$.

% \begin{definition}
% \label{def:optimal-point-flow}
%     A point $\vw \in M$ is optimal in a set of vector fields $W \subseteq Y$ if $Af(\vw) = 0$ for $f: M \xrightarrow[]{} \R,  \vw \mapsto \left\|  \frac{\ro \loss}{\ro \vw} \right\|_2^2 $ and all $A \in W$. 
%     A gradient flow $\gamma: \R \xrightarrow{} M$ is \textnormal{optimal in $W$} if $\gamma(t)$ is optimal in $W$ for all $t \in \R$.
%     % When $W = Y$, we simply call such $\vw$ and $\gamma$ optimal.
% \end{definition}
% \begin{definition}
%     An exact teleportation using a group $G$ is $\mathcal{T}_G: M \xrightarrow[]{} M$ moves a point to an optimal point the set of vector fields defined by infinitesimal actions of the Lie algebras of $G$.
% \end{definition}

% Even when $A$ and $R$ do not commute, there may exist some $\vw$ such that $[A, R]\loss(\vw) = 0$. 
\begin{manualproposition}{\ref{prop:Af-ARf}}
    A point $\vw \in M$ is optimal in a set of vector fields $W$ if and only if $[A, R]\loss(\vw) = 0$ for all $A \in W$.
% \label{prop:Af-ARf}
\end{manualproposition}
\begin{proof}
    Note that $A\loss = a^i \frac{\partial \loss}{\partial w^i} = 0$. We have
    \begin{align}
        [A, R]\loss &= A R\loss - R A\loss
        = A \left(r^i \frac{\ro \loss}{\ro w_i}\right) - 0 
        = -A \left\|\frac{\ro \loss}{\ro \vw}\right\|_2^2
        = -Af.
    \end{align}
    The result then follows from Definition \ref{def:optimal-point-flow}.
    % At the target teleportation destination $\vw$, $\left\|\frac{\ro \loss}{\ro \vw}\right\|_2^2$ is at a local maximum on a loss level set, hence $[A, R]\loss(\vw) = 0$. If $[A, R]\loss(\vw) = 0$ for all $A$, then $\vw$ is a local extremum. 
\end{proof}

\begin{manualproposition}{\ref{prop:one-teleportation}}
    Let $W \subseteq \mathfrak{X}_{\perp}$ be a set of vector fields that are orthogonal to the gradient of $\loss$. If $[A, R]\loss(\vw)=0$ for all $A \in W$ implies that $R([A, R]\loss)(\vw)=0$ for all $A \in W$, then the gradient flow starting at an optimal point in $W$ is optimal in $W$.
% \label{prop:one-teleportation}
\end{manualproposition}
\begin{proof}
% replace vector fields with differential
Consider the gradient flow $\gamma$ that starts at an optimal point in $W$. 
The derivative of $[A, R]\loss$ along $\gamma$ is
\begin{align}
    % \left. \frac{d}{dt}[A, R]\loss(\gamma(t)) \right|_{t=0} 
    \frac{d}{dt}[A, R]\loss(\gamma(t)) 
    = \gamma'(t) ([A, R]\loss) (\gamma(t)) 
    = -R[A, R]\loss(\gamma(t)).
\end{align}

Since $\gamma(0)$ is an optimal point, $[A, R]\loss(\gamma(0)) = 0$ for all $A \in W$ by Proposition \ref{prop:Af-ARf}.
By assumption, if $[A, R]\loss(\gamma(t))=0$ for all $A \in W$, then $R([A, R]\loss)(\gamma(t))=0$ for all $A \in W$.
Therefore, both the value and the derivative of $[A, R]\loss$ stay 0 along $\gamma$.
Since $[A, R]\loss(\gamma(t)) = 0$ for all $t \in \R$, $\gamma$ is optimal in $W$.
\end{proof}

To help check when Proposition \ref{prop:one-teleportation} is satisfied, we provide an alternative form of $R[A, R]\loss(\vw)$ under the assumption that $[A, R]\loss(\vw)=0$. 
% We provide an additional method to check whether $R[A, R]\loss(\vw)=0$ in the next proposition. 
We will use the following lemmas in the proof.

\begin{lemma}
\label{lemma:ortho-vector-antisymmetric-1}
    For two vectors $\vv, \vw \in \R^n$, if $\vv^T \vw = 0$ and $\vw \neq \mathbf{0}$, then there exists an anti-symmetric matrix $M \in \R^{n \times n}$ such that $\vv = M \vw$.
\end{lemma}
% Notes: 
% 1. ChatGPT is useless.
% 2. Similar technique as the proof of Proposition C.3 in ICLR2023 gradient flow paper.
% 3. Alternative proof: the map from the lie algebra to the vector field of infinitesimal actions is surjective (for all compact groups?)

% Every $\vw$ that is orthogonal to $\vv$ defines an infinitesimal transformation that perserves the norm of $\vv$. Therefore, there exists an infinitesimal action that has value $\vw$ at $\vv$. The infinitesimal action is $M \vv$ for some anti-symmetric $M$.
\begin{proof}
Let $\vw_0 = [1, 0, ..., 0]^T \in \R^n$. Consider a list of $n-1$ anti-symmetric matrices $M_i \in \R^{n \times n}$, where 
\begin{align}
    M_{ij}^{~k} =
    \begin{cases}
        -1, & \text{if $j=1$ and $k=i+1$}\\
        1, & \text{if $j=i+1$ and $k=1$}\\
        0, & \text{otherwise}
     \end{cases}
\end{align}
In matrix form, the $M_i$'s are
\begin{align}
    M_1 = 
    \begin{bmatrix}
        0 & -1 & 0 & ... & 0 \\
        1 & 0 & 0 & ... & 0 \\
        0 & 0 & 0 & ... & 0 \\
          &   & ... &  & \\
        0 & 0 & 0 & ... & 0
    \end{bmatrix},  
    M_2 = 
    \begin{bmatrix}
        0 & 0 & -1 & ... & 0 \\
        0 & 0 & 0 & ... & 0 \\
        1 & 0 & 0 & ... & 0 \\
          &   & ... &  & \\
        0 & 0 & 0 & ... & 0
    \end{bmatrix},  ..., 
    M_{n-1} = 
    \begin{bmatrix}
        0 & 0 & 0 & ... & -1 \\
        0 & 0 & 0 & ... & 0 \\
        0 & 0 & 0 & ... & 0 \\
          &   & ... &  & \\
        1 & 0 & 0 & ... & 0
    \end{bmatrix}.
\end{align}
Since $M_i$'s are anti-symmetric, $M_i \vw_0$ is orthogonal to $\vw_0$. The norm of $M_i \vw_0 = \mathbf{e}_{i+1}$ is 1. Additionally, $M_i \vw_0$ is orthogonal to $M_j \vw_0$ for $i \neq j$:
\begin{align}
    (M_i \vw_0)^T (M_j \vw_0) = \mathbf{e}_{i+1}^T \mathbf{e}_{j+1} = \delta_{ij}.
\end{align}
Denote $\vw_0^\bot = \{\vx \in \R^n: \vx^T\vw_0 = 0 \}$ as the orthogonal complement of $\vw_0$.
Then $M_i \vw_0$ forms a basis of $\vw_0^\bot$. Next, we extend this to an arbitrary $\vw \in \R^n$.

Let $\hat{\vw} = \frac{\vw}{\|\vw\|_2}$. Since $\hat{\vw}$ has norm 1, 
% and the group action of $O(n)$ is transitive on $S^{n-1}$,
there exists an orthogonal matrix $R$ such that $\hat{\vw} = R\vw_0$. Let $M_i' = RM_iR^T$. Then $M_i'$ is anti-symmetric:
\begin{align}
    (RM_iR^T)^T = RM^T_iR^T = -RM_iR^T.
\end{align}
It follows that $M_i'\hat{\vw}$ is orthogonal to $\hat{\vw}$. The norm of $M_i' \hat{\vw}$ is $\|(RM_iR^T)(R\vw_0)\| = \|R M_i \vw_0\| = \|M_i \vw_0\| = 1$. Additionally, $M_i' \hat{\vw}$ is orthogonal to $M_j' \hat{\vw}$ for $i \neq j$:
\begin{align}
    (M_i' \hat{\vw})^T (M_j' \hat{\vw}) 
    &= (RM_iR^T R\vw_0)^T (RM_jR^T R\vw_0) \cr
    &= \vw_0^T R^T R M_i^T R^T RM_jR^T R\vw_0 \cr
    &= \vw_0^T M_i^T M_j \vw_0 \cr
    &= \delta_{ij}.
\end{align}
Therefore, $M_i' \hat{\vw}$ spans $\hat{\vw}^\bot = \vw^\bot$. This means that any vector $\vv \in \vw^\bot$ can be written as a linear combination of $M_i' \hat{\vw}$. That is, there exists $k_1, ..., k_n \in \R$, such that $\vv = \sum_i k_i (M_i' \hat{\vw})$. To find the anti-symmetric $M$ that takes $\vw$ to $\vv$, note that 
\begin{align}
    \vv %= \sum_i k_i (M_i' \hat{\vw})
    = \left(\sum_i k_i M_i' \right) \hat{\vw} 
    = \left( \| \vw\|_2^{-1} \sum_i k_i M_i' \right) \vw.
\end{align}
Since the sum of anti-symmetric matrices is anti-symmetric, and the product of an anti-symmetric matrix and a scalar is also anti-symmetric, $\|\vw\|_2^{-1} \sum_i k_i M_i'$ is anti-symmetric. 
\end{proof}

\begin{lemma}
\label{lemma:ortho-vector-antisymmetric-3}
    Let $\vv \in \R^{n}$ be a nonzero vector. Then the two sets $\{M\vv: M \in \R^{n \times n}, M^T=-M\}$ and $\{\vw \in \R^n: \vw^T \vv = 0\}$ are equal. 
\end{lemma}
\begin{proof}
    Let $A = \{M\vv: M \in \R^{n \times n}, M^T=M^{-1}\}$ and $B = \{\vw \in \R^n: \vw^T \vv = 0\}$. Since $(M\vv)^T \vv = 0$ for all anti-symmetric $M$, every element in $A$ is in $B$. By Lemma \ref{lemma:ortho-vector-antisymmetric-1}, every element in $B$ is in $A$. Therefore $A=B$.
\end{proof}

Let $S = \{(M\frac{\partial \loss}{\partial \vw})^i \frac{\ro }{\ro w^i} \in \mathfrak{X}|~ M \in \R^{n \times n}, M^T = -M\}$ be the set of vector fields constructed by multiplying the gradient by an anti-symmetric matrix. 
Recall that $R = -\frac{\ro \loss}{\ro w_i} \frac{\ro }{\ro w^i}$ is the reverse gradient vector field, and $\mathfrak{X}_{\perp} = \{a^i \frac{\ro }{\ro w^i}|~ \sum_i a^i (\vw) 
\frac{\ro \loss(\vw)}{\ro w^i} = 0, \forall \vw \in \mathcal{M}\}$ is the set of all vector fields orthogonal to $R$. 
From Lemma \ref{lemma:ortho-vector-antisymmetric-3}, we have $S=\mathfrak{X}_{\perp}$. Therefore, a point $\vw$ is an optimal point in $S$ if and only if $\vw$ is an optimal point in $\mathfrak{X}_{\perp}$.

We are now ready to prove the following proposition, which provides another way to check the condition in Proposition \ref{prop:one-teleportation}.

% \begin{lemma}
% \label{lemma:ortho-vector-antisymmetric-2}
%     % Assume that $[A, R]\loss(\vw)=0$ for all $A \in Y$ at a point $\vw$.
%     If $[A, R]\loss(\vw) = 0$ for all $A \in S$, then $\frac{\ro \loss}{\ro \vw}$ is an eigenvector of the Hessian of $\loss$.
% \end{lemma}
% \begin{proof}
%     If $[A, R]\loss(\vw) = 0$ for a vector field $A$, then the directional derivative of $\left\|\frac{\ro \loss}{\ro \vw}\right\|_2^2$ along the direction of $A$ is 0 (proof of Proposition \ref{prop:Af-ARf}). 
%     Since $[A, R]\loss(\vw) = 0$ for all $A \in S$, the directional derivative of $\left\|\frac{\ro \loss}{\ro \vw}\right\|_2^2$ along the direction of all $A \in S$ is 0.
%     This means that the directional derivative of $\left\|\frac{\ro \loss}{\ro \vw}\right\|_2^2$ along all directions orthogonal to $\frac{\ro \loss}{\ro \vw}$ is 0 by Lemma \ref{lemma:ortho-vector-antisymmetric-1}.
%     Therefore, $\frac{\ro \loss}{\ro \vw}$ is an eigenvector of the Hessian of $\loss$ by Lemma C.4 in \cite{zhao2022symmetry}.
% \end{proof}

\begin{manualproposition}{\ref{prop:RARL2}}
    If at all optimal points in $S$,  
    \begin{align}
        M_\alpha^j \frac{\ro \loss}{\ro w_k} \frac{\ro \loss}{\ro w_\alpha} \frac{\ro^3 \loss}{\ro w^k \ro w_i \ro w^j} \frac{\ro \loss}{\ro w^i} = 0
    \end{align}
    for all anti-symmetric matrix $M \in \R^{n \times n}$, then the gradient flow starting at an optimal point in $S$ is optimal in $S$.
% \label{prop:RARL2}
\end{manualproposition}

\begin{proof}
Expanding $R[A, R]\loss$, we have
\begin{align}
    R[A, R]\loss &= R\left(A \left(r^i \frac{\ro \loss}{\ro w^i}\right) - 0 \right) \cr
    &= r^k \frac{\ro}{\ro w^k} \left(a^j \frac{\ro}{\ro w_j} \left(r^i \frac{\ro \loss}{\ro w^i}\right) \right) \cr
    &= r^k \frac{\ro}{\ro w^k} \left(a^j  \left(\frac{\ro r^i}{\ro w^j} \frac{\ro \loss}{\ro w^i} + r^i \frac{\ro}{\ro w^j} \frac{\ro \loss}{\ro w^i}\right) \right) \cr
    &= -r^k \frac{\ro}{\ro w^k} \left(a^j  \left(\left(\frac{\ro}{\ro w^j}\frac{\ro \loss}{\ro w_i}\right) \frac{\ro \loss}{\ro w^i} + \frac{\ro \loss}{\ro w_i} \frac{\ro}{\ro w^j} \frac{\ro \loss}{\ro w^i}\right) \right) \cr
    &= -2r^k \frac{\ro}{\ro w^k} \left(a^j \frac{\ro^2 \loss}{\ro w_i \ro w^j} \frac{\ro \loss}{\ro w^i} \right)\cr
    &= -2r^k \left(
    \frac{\ro a^j}{\ro w^k} \frac{\ro^2 \loss}{\ro w_i \ro w^j} \frac{\ro \loss}{\ro w^i}
    + a^j \frac{\ro}{\ro w^k} \left(\frac{\ro^2 \loss}{\ro w_i \ro w^j} \frac{\ro \loss}{\ro w^i} \right)
    \right) \cr
    &= 2 \frac{\ro \loss}{\ro w_k}
    \frac{\ro a^j}{\ro w^k} \frac{\ro^2 \loss}{\ro w_i \ro w^j} \frac{\ro \loss}{\ro w^i}
    + 2 \frac{\ro \loss}{\ro w_k} a^j \frac{\ro}{\ro w^k} \left(\frac{\ro^2 \loss}{\ro w_i \ro w^j} \frac{\ro \loss}{\ro w^i} \right)
\label{eq:RARL-expand-1}
\end{align}
Assume that $\vw$ is an optimal point in $S$. %That is, $[A, R]\loss(\vw) = 0$ for all $A \in S$.
By Lemma \ref{lemma:ortho-vector-antisymmetric-3}, $\vw$ is also an optimal point in $\mathfrak{X}_{\perp}$. By Lemma C.4 in \cite{zhao2022symmetry}, $\frac{\ro \loss}{\ro \vw}$ is an eigenvector of $\frac{\ro^2 \loss}{\ro w_i \ro w^j}$. % Z span T_w M
Therefore, $\frac{\ro^2 \loss}{\ro w_i \ro w^j}\frac{\ro \loss}{\ro w^i} = \lambda \frac{\ro \loss}{\ro w^j}$ for some $\lambda \in \mathbb{C}$.
Additionally, $a^j = M_\alpha^j \frac{\ro \loss}{\ro w_\alpha}$ and $\frac{\ro a^j}{\ro w^k} = M_\alpha^j \frac{\ro^2 \loss}{\ro w_\alpha \ro w^k}$. We are now ready to simplify both terms in \eqref{eq:RARL-expand-1}. 

For the first term in \eqref{eq:RARL-expand-1},
\begin{align}
    \frac{\ro \loss}{\ro w_k}
    \frac{\ro a^j}{\ro w^k} \frac{\ro^2 \loss}{\ro w_i \ro w^j} \frac{\ro \loss}{\ro w^i}
    &= \frac{\ro \loss}{\ro w_k} M_\alpha^j \frac{\ro^2 \loss}{\ro w_\alpha \ro w^k} \frac{\ro^2 \loss}{\ro w_i \ro w^j} \frac{\ro \loss}{\ro w^i} \cr
    &= M_\alpha^j 
    \left(\frac{\ro^2 \loss}{\ro w_\alpha \ro w^k}\frac{\ro \loss}{\ro w_k}\right) 
    \left(\frac{\ro^2 \loss}{\ro w_i \ro w^j} \frac{\ro \loss}{\ro w^i} \right) \cr
    &= M_\alpha^j \left(\lambda_1 \frac{\ro \loss}{\ro w_\alpha}\right) \left(\lambda_2 \frac{\ro \loss}{\ro w^j} \right) \cr
    &= \lambda_1 \lambda_2 M_\alpha^j \frac{\ro \loss}{\ro w_\alpha} \frac{\ro \loss}{\ro w^j} \cr
    &= 0
\end{align}
The last equality holds because $M$ is anti-symmetric.

For the second term in \eqref{eq:RARL-expand-1},
\begin{align}
    \frac{\ro \loss}{\ro w_k} a^j \frac{\ro}{\ro w^k} \left(\frac{\ro^2 \loss}{\ro w_i \ro w^j} \frac{\ro \loss}{\ro w^i} \right)
    &= \frac{\ro \loss}{\ro w_k} a^j \left(\frac{\ro^3 \loss}{\ro w^k \ro w_i \ro w^j} \frac{\ro \loss}{\ro w^i} + \frac{\ro^2 \loss}{\ro w_i \ro w^j} \frac{\ro^2 \loss}{\ro w^k\ro w^i} \right) \cr
    &= \frac{\ro \loss}{\ro w_k} M_\alpha^j \frac{\ro \loss}{\ro w_\alpha} 
    \left(\frac{\ro^3 \loss}{\ro w^k \ro w_i \ro w^j} \frac{\ro \loss}{\ro w^i} + \frac{\ro^2 \loss}{\ro w_i \ro w^j} \frac{\ro^2 \loss}{\ro w^k\ro w^i} \right) \cr
    &= M_\alpha^j \frac{\ro \loss}{\ro w_k} \frac{\ro \loss}{\ro w_\alpha} \frac{\ro^3 \loss}{\ro w^k \ro w_i \ro w^j} \frac{\ro \loss}{\ro w^i} + \lambda_1 \lambda_2 M_\alpha^j \frac{\ro \loss}{\ro w_\alpha} \frac{\ro \loss}{\ro w^j} \cr
    &= M_\alpha^j \frac{\ro \loss}{\ro w_k} \frac{\ro \loss}{\ro w_\alpha} \frac{\ro^3 \loss}{\ro w^k \ro w_i \ro w^j} \frac{\ro \loss}{\ro w^i}
\end{align}
% Similar to the first term in \eqref{eq:RARL-expand-1}, the second term in the second-to-last line equals 0.

In summary, 
\begin{align}
    R[A, R]\loss = 2M_\alpha^j \frac{\ro \loss}{\ro w_k} \frac{\ro \loss}{\ro w_\alpha} \frac{\ro^3 \loss}{\ro w^k \ro w_i \ro w^j} \frac{\ro \loss}{\ro w^i}.
\end{align}
Since we assumed that $[A, R]\loss(\vw) = 0$, when $R[A, R]\loss(\vw) = 0$ for all $A \in S$, the gradient flow starting at an optimal point in $S$ is optimal in $S$.
\end{proof}

\begin{proposition}
\label{proposition:anti-sym-equality}
    If $\frac{\ro^3 \loss}{\ro w^k \ro w^i \ro w^j} \frac{\ro \loss}{\ro w^\alpha} = \frac{\ro^3 \loss}{\ro w^k \ro w^i \ro w^\alpha} \frac{\ro \loss}{\ro w^j}$
    holds for all $i, k, j, \alpha$, then $M_\alpha^j \frac{\ro \loss}{\ro w_k} \frac{\ro \loss}{\ro w_\alpha} \frac{\ro^3 \loss}{\ro w^k \ro w_i \ro w^j} \frac{\ro \loss}{\ro w^i} = 0$
    holds for all anti-symmetric matrices $M \in \R^{n \times n}$.
\end{proposition}
\begin{proof}
If $\frac{\ro^3 \loss}{\ro w^k \ro w^i \ro w^j} \frac{\ro \loss}{\ro w^\alpha} = \frac{\ro^3 \loss}{\ro w^k \ro w^i \ro w^\alpha} \frac{\ro \loss}{\ro w^j}$ for all $i, k, j, \alpha$, then 
\begin{align}
    &\quad\quad M_\alpha^j \frac{\ro \loss}{\ro w_k} \frac{\ro \loss}{\ro w_\alpha} \frac{\ro^3 \loss}{\ro w^k \ro w_i \ro w^j} \frac{\ro \loss}{\ro w^i} \cr
    &= \sum_{i, k, \alpha < j} M_\alpha^j \frac{\ro \loss}{\ro w_k} \frac{\ro \loss}{\ro w_\alpha} \frac{\ro^3 \loss}{\ro w^k \ro w_i \ro w^j} \frac{\ro \loss}{\ro w^i} + \sum_{i, k, \alpha > j} M_\alpha^j \frac{\ro \loss}{\ro w_k} \frac{\ro \loss}{\ro w_\alpha} \frac{\ro^3 \loss}{\ro w^k \ro w_i \ro w^j} \frac{\ro \loss}{\ro w^i} \cr 
    &= \sum_{i, k, \alpha < j} M_\alpha^j \frac{\ro \loss}{\ro w_k} \frac{\ro \loss}{\ro w_\alpha} \frac{\ro^3 \loss}{\ro w^k \ro w_i \ro w^j} \frac{\ro \loss}{\ro w^i} + \sum_{i, k, j > \alpha} M_j^\alpha \frac{\ro \loss}{\ro w_k} \frac{\ro \loss}{\ro w_j} \frac{\ro^3 \loss}{\ro w^k \ro w_i \ro w^\alpha} \frac{\ro \loss}{\ro w^i} \cr 
    &= \sum_{i, k, \alpha < j} M_\alpha^j \frac{\ro \loss}{\ro w_k} \frac{\ro \loss}{\ro w_\alpha} \frac{\ro^3 \loss}{\ro w^k \ro w_i \ro w^j} \frac{\ro \loss}{\ro w^i} + \sum_{i, k, j > \alpha} -M_\alpha^j \frac{\ro \loss}{\ro w_k} \frac{\ro \loss}{\ro w_j} \frac{\ro^3 \loss}{\ro w^k \ro w_i \ro w^\alpha} \frac{\ro \loss}{\ro w^i} \cr 
    &= \sum_{i, k, \alpha < j} M_\alpha^j \frac{\ro \loss}{\ro w_k} \frac{\ro \loss}{\ro w^i} \left(\frac{\ro \loss}{\ro w_\alpha} \frac{\ro^3 \loss}{\ro w^k \ro w_i \ro w^j} - \frac{\ro \loss}{\ro w_j} \frac{\ro^3 \loss}{\ro w^k \ro w_i \ro w^\alpha}\right) \cr 
    &= 0,
\end{align}
where the first equality uses that the diagonal of an anti-symmetric matrix is 0, the second equality swaps $\alpha$ and $j$ in the second term, the third equality uses that $M$ is anti-symmetric.
% Suppose that $\frac{\ro^3 \loss}{\ro w^k \ro w^i \ro w^j} \frac{\ro \loss}{\ro w^\alpha} \neq \frac{\ro^3 \loss}{\ro w^k \ro w^i \ro w^\alpha} \frac{\ro \loss}{\ro w^j}$ for some $i, k, j, \alpha$. Then $\frac{\ro \loss}{\ro w_k} \neq 0$, because otherwise the third derivatives are 0 and both side are equal to 0. Let $M_\alpha^j = 1$, 
\end{proof}

\paragraph{Example (Quadratic function)}
Consider the quadratic function $\loss(\vw) = \frac{1}{2}\vw^T A \vw + \mathbf{b}^T \vw + \mathbf{c}$, where $A \in \R^{n \times n}$ is symmetric, $\mathbf{b}, \mathbf{c} \in \R^n$, and $\vw \in \R^n$. 
Two examples of quadratic functions are the ellipse $\loss_e(w_1, w_2) = \frac{1}{2}(w_1^2+\lambda^2 w_2^2)$ and the Booth function $\loss_{b}(w_1, w_2) = (w_1 + 2w_2 - 7)^2 + (2w_1 + w_2 - 5)^2$.
Since the third derivative of $\loss$ is 0, 
% $R[A, R]\loss(\vw)=0$ when $[A, R]\loss(\vw)=0$ 
one teleportation guarantees optimal trajectory.% (Proposition \ref{prop:RARL2}).

\section{Group actions and curves on minima}
\label{sec:appendix-group}

\subsection{Group actions for MLP}
\label{sec:appendix-group-action}
Consider a multi-layer neural network with elementwise activation function $\sigma$. The output of the $m^{th}$ layer is $h_m = \sigma(W_m h_{m-1})$, where $W_m \in \R^{d_{m}\times d_{m-1}}$ is the weight, $h_{m-1} \in \R^{d_{m-1}\times k}$ is the output of the $m-1^{th}$ layer, and $h_0 \in \R^{d_0 \times k}$ is the data. 

Assuming that $\sigma\pa{g_m W_{m-1} h_{m-2}}$ is invertible, for $g_m \in \mathrm{GL}_{d_{m-1}}(\R)$, the following transformation is a loss-preserving group action: 
\begin{align}
    g_m \cdot W_k = \left\{
    \begin{array}{lc}
        W_m \sigma\pa{W_{m-1} h_{m-2}} \sigma\pa{g_m W_{m-1} h_{m-2}}^{-1} & k = m \\
        g_m W_{m-1} & k = m-1 \\
        W_k & k \not\in \{m, m-1\}
    \end{array}
    \right.
    \label{eq:Wm-g-action-main}
\end{align}
Usually, the assumption does not hold \citep{zhao2022symmetries}. Hence the above transformation may not preserve loss or be a valid group action. 
Nevertheless, we observe in practice that the change in the loss value is often small after such transformations on parameters. We therefore refer to equation (\ref{eq:Wm-g-action-main}) as an approximate symmetry and adopt it in the teleportation algorithm.
Due to the possibility that $\sigma\pa{g_m W_{m-1} h_{m-2}}$ is not invertible, we use pseudoinverses in implementations.

% There are two ways to define a $\mathrm{GL}_{d_{m-1}}(\R)$ symmetry acting on $W_m$ and $W_{m-1}$. 
% Unless stated otherwise, we use the second group action since it does not require $\sigma$ to be invertible. We use pseudoinverses in experiments.

% \paragraph{Group action 1}
% \hspace{-10pt} \cite{zhao2022symmetry}.
% Assume that $h_{m-2}$ is invertible and $\sigma$ is bijective. For $g_m \in \mathrm{GL}_{d_{m-1}}(\R)$,
% \begin{align}
%     g_m \cdot W_k = \left\{
%     \begin{array}{lc}
%         W_m g_m^{-1} & k = m \\
%         \sigma^{-1}\pa{g_m \sigma\pa{W_{m-1} h_{m-2}} } h_{m-2}^{-1} & k = m-1 \\
%         W_k & k \not\in \{m, m-1\}
%     \end{array}
%     \right.
%     \label{eq:Wm-g-action-main}
% \end{align}

% \paragraph{Group action 2}
% \hspace{-10pt} \cite{zhao2022symmetries}.
% Assume that $g_m \sigma\pa{W_{m-1} h_{m-2}}$ is invertible. For $g_m \in \mathrm{GL}_{d_{m-1}}(\R)$, % the assumption never holds
% \begin{align}
%     g_m \cdot W_k = \left\{
%     \begin{array}{lc}
%         W_m \sigma\pa{W_{m-1} h_{m-2}} \sigma\pa{g_m W_{m-1} h_{m-2}}^{-1} & k = m \\
%         g_m W_{m-1} & k = m-1 \\
%         W_k & k \not\in \{m, m-1\}
%     \end{array}
%     \right.
%     \label{eq:Wm-g-action-main}
% \end{align}

% % \subsection{Lie algebras, the exponential map, and infinitesimal generators of group actions}
% % We first define the curves that lie in the minimum. Consider the curve $\gamma_M: \R \times \R^n \xrightarrow[]{} \R^n$ where $M \in T_I G$ and $\gamma_M(t, \vw) = \exp{(tM)} \cdot \vw$. Then $\gamma(0, \vw) = \vw$. 

\subsection{Curvature}
The curvature of a curve $\gamma: \R \xrightarrow[]{} \R^n$ is $\kappa(t) = \frac{\|T'(t)\|}{\|\gamma'(t)\|}$, where $T(t) = \frac{\gamma'(t)}{\|\gamma'(t)\|}$ is the unit tangent vector. 
The curvature can be written as a function of $\gamma'$ and $\gamma''$ \citep{alessio2012formulas, shelekhov2021curvatures}:
\begin{align}
\label{eq:curvature-appendix}
    \kappa(t) = \frac{\br{\|\gamma'\|^2 \|\gamma''\|^2 - (\gamma' \cdot \gamma'')^2 }^\frac{1}{2}}{\|\gamma'\|^3}.
\end{align}
% https://en.wikipedia.org/wiki/Curvature#General_expressions

\subsection{The derivative of curvature}
To compute the derivative of $\kappa(t)$, we first list the derivatives of a few commonly used terms:
\begin{align}
    \frac{d}{dt} \|\gamma'\|^2 
    &= \frac{d}{dt} ({\gamma'_1}^2 + {\gamma'_2}^2 + {\gamma'_3}^2 + ...)
    = 2\gamma'_1 \gamma''_1 + 2\gamma'_2 \gamma''_2 + 2\gamma'_3 \gamma''_3 + ...
    = 2 \gamma' \cdot \gamma'' \cr
    \frac{d}{dt} \|\gamma''\|^2 
    &= \frac{d}{dt} ({\gamma''_1}^2 + {\gamma''_2}^2 + {\gamma''_3}^2 + ...)
    = 2\gamma''_1 \gamma'''_1 + 2\gamma''_2 \gamma'''_2 + 2\gamma''_3 \gamma'''_3 + ...
    = 2 \gamma'' \cdot \gamma''' \cr
    \frac{d}{dt} (\gamma' \cdot \gamma'')
    &= \frac{d}{dt} (\gamma'_1 \gamma''_1 + \gamma'_2 \gamma''_2 + \gamma'_3 \gamma''_3...)
    = \gamma'_1 \gamma'''_1 + \gamma''_1 \gamma''_1 + ...
    = \|\gamma''\|^2 + \gamma' \cdot \gamma'''
\end{align}

The derivatives of the numerator and denominator of $\kappa$ are:
\begin{align}
    \frac{d}{dt} \br{\|\gamma'\|^2 \|\gamma''\|^2 - (\gamma' \cdot \gamma'')^2 }^\frac{1}{2}
    &= \frac{1}{2} \br{\|\gamma'\|^2 \|\gamma''\|^2 - (\gamma' \cdot \gamma'')^2 }^{-\frac{1}{2}}
    \frac{d}{dt} \br{\|\gamma'\|^2 \|\gamma''\|^2 - (\gamma' \cdot \gamma'')^2 } \cr 
    &= \frac{1}{2} \br{\|\gamma'\|^2 \|\gamma''\|^2 - (\gamma' \cdot \gamma'')^2 }^{-\frac{1}{2}} \cr
    &\quad\quad\quad \br{\|\gamma'\|^2 \frac{d}{dt} \|\gamma''\|^2 + \|\gamma''\|^2 \frac{d}{dt} \|\gamma'\|^2   
    - 2 (\gamma' \cdot \gamma'') \frac{d}{dt}(\gamma' \cdot \gamma'') } \cr 
    &= \frac{1}{2} \br{\|\gamma'\|^2 \|\gamma''\|^2 - (\gamma' \cdot \gamma'')^2 }^{-\frac{1}{2}} \cr
    &\quad\quad\quad \br{2\|\gamma'\|^2 (\gamma'' \cdot \gamma''') + 2\|\gamma''\|^2 (\gamma' \cdot \gamma'')
    - 2 (\gamma' \cdot \gamma'') (\|\gamma''\|^2 + \gamma' \cdot \gamma''') } \cr 
    &= \br{\|\gamma'\|^2 \|\gamma''\|^2 - (\gamma' \cdot \gamma'')^2 }^{-\frac{1}{2}}
    \br{\|\gamma'\|^2 (\gamma'' \cdot \gamma''')
    - (\gamma' \cdot \gamma'') (\gamma' \cdot \gamma''') },
\end{align}
and
\begin{align}
    \frac{d}{dt} \|\gamma'\|^3 
    = \frac{d}{dt} (\|\gamma'\|^2)^\frac{3}{2} 
    = \frac{3}{2} (\|\gamma'\|^2)^\frac{1}{2} \frac{d}{dt} \|\gamma'\|^2
    = \frac{3}{2} (\|\gamma'\|^2)^\frac{1}{2} (2 \gamma' \cdot \gamma'')
    = 3 \|\gamma'\| (\gamma' \cdot \gamma'').
\end{align}

Using the derivatives above, the derivative of $\kappa$ is
\begin{align}
\label{eq:curvature-derivative}
    \kappa'(t) &= \frac{
    \left[\frac{d}{dt} \br{\|\gamma'\|^2 \|\gamma''\|^2 - (\gamma' \cdot \gamma'')^2 }^\frac{1}{2}\right] \|\gamma'\|^3 - 
     \br{\|\gamma'\|^2 \|\gamma''\|^2 - (\gamma' \cdot \gamma'')^2 }^\frac{1}{2} \left[\frac{d}{dt} \|\gamma'\|^3\right]
    }{\|\gamma'\|^6} \cr 
    &= \frac{\splitfrac{\br{\|\gamma'\|^2 \|\gamma''\|^2 - (\gamma' \cdot \gamma'')^2 }^{-\frac{1}{2}}
    \br{\|\gamma'\|^2 (\gamma'' \cdot \gamma''')
    - (\gamma' \cdot \gamma'') (\gamma' \cdot \gamma''') } \|\gamma'\|^3}
    {- 
    \br{\|\gamma'\|^2 \|\gamma''\|^2 - (\gamma' \cdot \gamma'')^2 }^\frac{1}{2} 3 \|\gamma'\| (\gamma' \cdot \gamma'')}
    }{\|\gamma'\|^6} \cr 
    &= \frac{\splitfrac{\br{\|\gamma'\|^2 \|\gamma''\|^2 - (\gamma' \cdot \gamma'')^2 }^{-\frac{1}{2}}
    \br{\|\gamma'\|^2 (\gamma'' \cdot \gamma''')
    - (\gamma' \cdot \gamma'') (\gamma' \cdot \gamma''') } \|\gamma'\|^2}
    {- 
    \br{\|\gamma'\|^2 \|\gamma''\|^2 - (\gamma' \cdot \gamma'')^2 }^\frac{1}{2} 3 (\gamma' \cdot \gamma'')}
    }{\|\gamma'\|^5}. \cr 
\end{align}
  
\subsection{The derivatives of curves on minima}
Consider the curve $\gamma_M: \R \times \R^n \xrightarrow[]{} \R^n$ where $M \in \Lie(G)$ and 
% $\gamma_M(t, \vw) = \exp{(tM)} \cdot \vw$. 
\begin{align}
    \gamma_M(t, \vw) = \exp{(tM)} \cdot \vw.
\end{align}
% Then $\gamma(0, \vw) = \vw$. 
In this section, we derive $\gamma'$, $\gamma''$, and $\gamma'''$, which are needed to compute the curvature $\kappa(t)$ and its derivative $\kappa'(t)$. We are interested in $\kappa$ and $\kappa'$ at $\vw$, or equivalently, at $t=0$.
To find the derivatives of $\gamma$ at $t=0$, we write the group action in the following form:
\begin{align}
\label{eq:taylor-general}
    \gamma(t) = \sum_{n=0}^\infty \frac{f(n)}{n!}t^n.
\end{align}
By the uniqueness of Taylor polynomial, the derivatives are $\gamma^{(n)}(0) = f(n)$. In the rest of this subsection, we expand the group action to find $f(n)$.

% Consider the following function $f$, where $\sigma$ is an element-wise function:
% \begin{align}
%     f(U, V, X) = U \sigma(VX).
% \end{align}
Consider two consecutive layers $U \sigma(VX)$ in a neural network, where $U \in \R^{m \times h}, V \in \R^{h \times n}$ are weights, $X \in \R^{h \times k}$ is the output from the previous layer, and $\sigma$ is an elementwise activation function. 
Choosing $G=GL_h(\R)$, one group action that leaves the output of these two layers unchanged is:
\begin{align}
    g \cdot (U, V, X) = (g \cdot U, g \cdot V, g \cdot X) = (Ug^{-1}, \sigma^{-1} (g \sigma (VX)) X^{-1}, X).
\end{align}

Let  
\begin{align}
    g = \exp(tM) = \sum_{k=0}^\infty \frac{1}{k!} (tM)^k,
\end{align}
where $M\in \Lie(G)$ is in the Lie algebra of $G$.
The action of $g$ yields
\begin{align}
    g \cdot (U,V, X) 
    &= (U\exp(-tM), \sigma^{-1} (\exp(tM) \sigma (VX)) X^{-1}, X).
    % &= (U - \eps^i U T_i, \sigma^{-1}(\sigma(VX) + \eps^i T_i \sigma(VX))X^{-1}, X)
\label{eq:linearization1}
\end{align}

Next, we expand $\gamma(t) = g \cdot (U, V)$. The Taylor expansion for $g \cdot U$ is
\begin{align}
\label{eq:taylor-U}
    U\exp(-tM) 
    &= U \sum_{k=0}^\infty \frac{1}{k!} (-tM)^k \cr 
    &= U - tUM + \frac{t^2}{2!} UM^2 - \frac{t^3}{3!} UM^3 + O(t^4).
\end{align}

The Taylor expansion for $g \cdot V$ is
\begin{align}
    & \sigma^{-1} (\exp(tM) \sigma (VX)) X^{-1} \cr
    =& \sigma^{-1} \left( \left(\sum_{k=0}^\infty \frac{1}{k!} (tM)^k\right) \sigma (VX) \right) X^{-1} \cr 
    =& \sigma^{-1} \left( \sigma (VX) + \sum_{k=1}^\infty \frac{1}{k!} (tM)^k \sigma (VX) \right) X^{-1} \cr 
    =& \left[ \sigma^{-1}(\sigma(VX)) + 
    \sum_{j=1}^\infty 
    \left(\sum_{k=1}^\infty \frac{1}{k!} (tM)^k \sigma (VX) \right)^{\odot j}
    \odot \left. \frac{\partial^j \sigma^{-1}(A)}{\partial A^j} \right|_{A=\sigma(VX)} 
    \right] X^{-1} \cr 
    =& V + \left[
    \sum_{j=1}^\infty 
    \left(\sum_{k=1}^\infty \frac{1}{k!} (tM)^k \sigma (VX) \right)^{\odot j}
    \odot \left. \frac{\partial^j \sigma^{-1}(A)}{\partial A^j} \right|_{A=\sigma(VX)} 
    \right] X^{-1},
\end{align}
where $\odot$ denotes element-wise product: $(A \odot B)_{mn} = A_{mn} B_{mn}$, and the superscript $^{\odot}$ denotes elementwise power: $(A^{\odot j})_{mn} = (A_{mn})^j$. The Taylor expansion is of each element individually, because $\sigma$ is element-wise.

Since our goal is to find the first 3 derivatives of $\gamma$, we are only interested in the terms up to $t^3$. Letting
\begin{align}
    \sum_{k=1}^\infty \frac{1}{k!} (tM)^k 
    = tM + t^2\frac{M^2}{2} + t^3\frac{M^3}{6} + O(t^4)
\end{align}
and considering only the $j=1,2,3$ terms, we have 
\begin{align}
\label{eq:taylor-V}
    &\sigma^{-1} (\exp(tM) \sigma (VX)) X^{-1} \cr 
    =& V + \left[
    \sum_{j=1}^\infty 
    \left((tM + t^2\frac{M^2}{2} + t^3\frac{M^3}{6}) \sigma (VX) \right)^{\odot j}
    \odot \left. \frac{\partial^j \sigma^{-1}(A)}{\partial A^j} \right|_{A=\sigma(VX)} 
    \right] X^{-1}  + O(t^4) \cr
    =& V + \left[
    \left((tM + t^2\frac{M^2}{2} + t^3\frac{M^3}{6}) \sigma (VX) \right)
    \odot \left. \frac{\partial \sigma^{-1}(A)}{\partial A} \right|_{A=\sigma(VX)} \right.\cr
    &\quad\quad\quad + 
    \left((tM + t^2\frac{M^2}{2} + t^3\frac{M^3}{6}) \sigma (VX) \right)^{\odot 2}
    \odot \left. \frac{\partial^2 \sigma^{-1}(A)}{\partial A^2} \right|_{A=\sigma(VX)} \cr
    &\quad\quad\quad + \left. 
    \left((tM + t^2\frac{M^2}{2} + t^3\frac{M^3}{6}) \sigma (VX) \right)^{\odot 3}
    \odot \left. \frac{\partial^3 \sigma^{-1}(A)}{\partial A^3} \right|_{A=\sigma(VX)} 
    \right] X^{-1}  + O(t^4) \cr
    =& V + t \left( \left( M \sigma (VX) \right)
    \odot \frac{1}{\sigma'(VX)} \right) X^{-1} \cr
    &\quad\quad + \frac{t^2}{2} \left( 
    \left( M^2 \sigma (VX) \right) \odot \frac{1}{\sigma'(VX)} 
    - 2 (M \sigma (VX))^{\odot 2}  \odot \frac{\sigma''(VX)}{\sigma'(VX)^3}
    \right) X^{-1}\cr
    &\quad\quad + \frac{t^3}{6} \left( 
    \left( M^3 \sigma (VX) \right) \odot \frac{1}{\sigma'(VX)} 
    - 6 (M \sigma (VX)) \odot (M^2 \sigma (VX)) \odot \frac{\sigma''(VX)}{\sigma'(VX)^3}
    \right.\cr 
    &\quad\quad\quad\quad\quad\quad \left. 
    + 6 (M \sigma (VX))^{\odot 3}  \odot \left. \frac{\partial^3 \sigma^{-1}(A)}{\partial A^3} \right|_{A=\sigma(VX)} 
    \right) X^{-1}\cr
    &\quad\quad + O(t^4).
\end{align}
Matching terms in \eqref{eq:taylor-U} and \eqref{eq:taylor-V} with \eqref{eq:taylor-general}, we have the expressions for $\gamma'$, $\gamma''$, and $\gamma'''$. This allows us to compute the curvature and its derivative using \eqref{eq:curvature-appendix} and \eqref{eq:curvature-derivative}.

\section{Sharpness, Curvatures, and Their Relation to Generalization}
\label{appendix:curvature-generalization}
\subsection{Alternative definitions of sharpness}
A common definition of flat minimum is based on the number of eigenvalues of the Hessian which are small. Minimizers with a large number of large eigenvalues tend to have worse generalization ability \citep{keskar2017large}. Let $\lambda_i(H)(\vw)$ be the $i^{th}$ largest eigenvalue of the Hessian of the loss function evaluated at $\vw$. We can quantify the notion of sharpness by the number of eigenvalues larger than a threshold $\eps \in \R^{>0}$: % complex eigenvalues?
\begin{align}
    \phi_1(\vw, \eps) = \left| \{\lambda_i(H)(\vw): \lambda_i > \eps\} \right|.
\end{align}
A related sharpness metric uses the logarithm of the product of the $k$ largest eigenvalues \citep{wu2017towards},
\begin{align}
    \phi_2(\vw, k) = \sum_{i=1}^k \log \lambda_i(H)(\vw).
\end{align}
Both metrics require computing the eigenvalues of the Hessian. As a result, optimizing on these metrics during teleportation is prohibitively expensive. 
Hence, in this paper we use the average change in loss averaged over random directions ($\phi$) as objective in generalization experiments. 

\subsection{More intuition on curvatures and generalization}
\label{sec:appendix-group-intuition}

\subsubsection{Example: curvature affects average displacement of minima}

Consider an optimization problem with two variables $w_1, w_2 \in \R$. Assume that the minimum is a curve $\gamma: \R \to \R^2$ in the two-dimensional parameter space. For a point $\vw_0$ on $\gamma$, we estimate its generalization ability by computing the expected distance between $\vw_0$ and the new minimum obtained by shifting $\gamma$. 

We consider the following two curves as examples:
\begin{align}
    \gamma_1:& \R \to \R^2, t \mapsto (t, k_1 t^2) \cr
    \gamma_2:& [0, 2\pi] \to \R^2, \theta \mapsto (k_2\cos(\theta), k_2\sin(\theta) + k_2),
\end{align}
with $k_1, k_2 \in \R^{\neq 0}$. The curve $\gamma_1$ is a parabola with curvature $\kappa_1 = 2k_1$ at $\vw_0 = (0, 0)$. The curve $\gamma_2$ is a circle, with curvature $\kappa_2 = \frac{1}{k_2}$ at $\vw_0$.
Note that $\gamma_1$ is the only polynomial approximation with integer power ($\gamma(t) = (t, k|t|^n), n \in \mathbb{Z}^+$) where the curvature at $\vw_0$ depends on $k$. When $n<1$, the value of $\vw_0$ is undefined. When $n=1$, the first derivative at $\vw_0$ is undefined. When $n>2$, $\kappa(\vw_0) = 0$. 

Assume that a distribution shift in data causes $\gamma$ to shift by a distance $r$, and that the direction of the shift is chosen uniformly at random over all possible directions. 
Viewing from the perspective of the curve, this is equivalent to shifting $\vw_0$ by distance $r$. 

The distance between a point $\vw$ and a curve $\gamma$ is
\begin{align}
    dist(\vw, \gamma) = \min_{\vw' \in \gamma_2} \|\vw' - \vw\|_2.
\end{align}

Let $S_r$ be the circle centered at the origin with radius $r$. The expected distance between the old solution $\vw_0$ and shifted curve is
\begin{align}
\label{eq:curvature-displacement-example}
    \mathbb{E}_{\vw \in S_r}[dist(\vw, \gamma)]
    = \frac{\int_{S_r} dist(\vw, \gamma) ds}{\int_{S_r} ds}
    = \frac{\int_0^{2\pi} dist((r\cos{\theta}, r\sin{\theta}), \gamma) r d\theta}{\int_0^{2\pi} r d\theta}.
\end{align}

In the limit of zero curvature, $\gamma$ is a straight line $\gamma(t) = (t, 0)$. In this case, the expected distance is 
\begin{align}
    \mathbb{E}_{\vw \in S_r}[dist(\vw, \gamma)]
    = \frac{\int_0^{2\pi} |r\sin{\theta}| r d\theta}{2 \pi r}
    = \frac{2r}{\pi}
    \approx 0.637 r.
\end{align}

Figure \ref{fig:curvature-displacement}(b)(c) shows that the expected distance's dependence on $\kappa$. Using both curves $\gamma_1$ and $\gamma_2$, the generalization ability of $\vw_0$ depends on the curvature at $\vw_0$. However, the type of dependence is affected by the type of curve used. In other words, the curvatures at points around $\vw_0$ affect how the curvature at $\vw_0$ affects generalization. Therefore, from these results alone, one cannot deduce whether minima with sharper curvatures generalize better or worse. 
To find a more definitive relationship between curvature and generalization, further investigation on the type of curves on the minimum is required.

% The value of $k$ approximates the curvature at $\vw_0$.
% Points with smaller distance to a minimum typically has lower loss value.  
% Therefore, larger curvatures may indicate less robustness of a point on a minimum. 
% This gives a slightly more precise example that sharper points on the minima have worse generalization ability. 

We emphasize that this example only serves as an intuition for connecting curvature to generalization.
As a future direction, it would be interesting to consider different families of parametric curves, higher dimensional parameter spaces, and deforming in addition to shifting the minima.

\begin{figure}[h!]
\begin{center}
\ \ \ (a) \hfill (b) \hfill (c) \hfill ~ \\
\includegraphics[width=0.3\columnwidth]{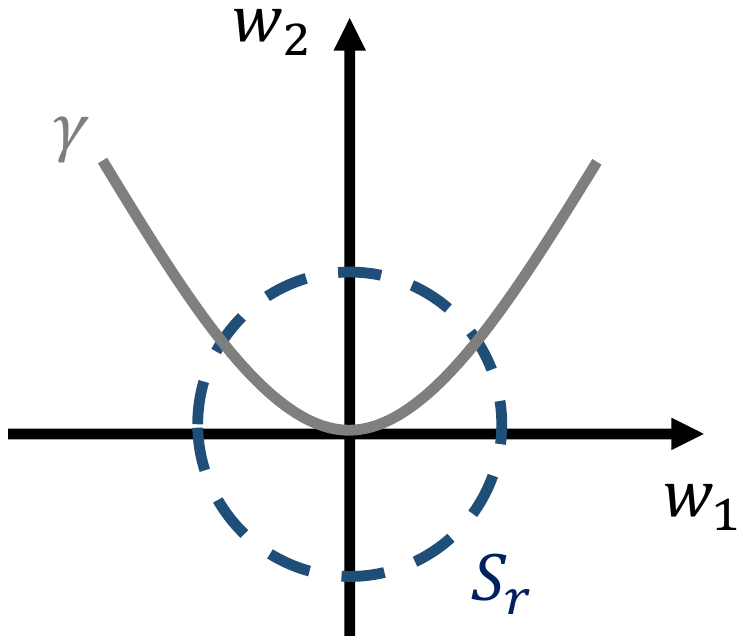}
\includegraphics[width=0.31\columnwidth]{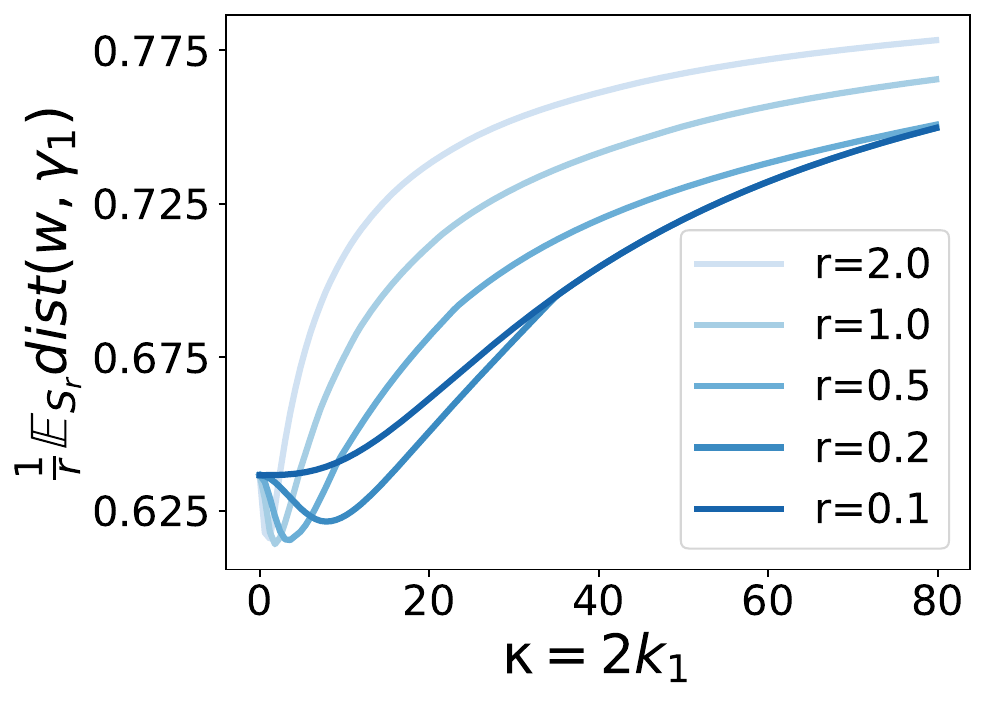}
\includegraphics[width=0.3\columnwidth]{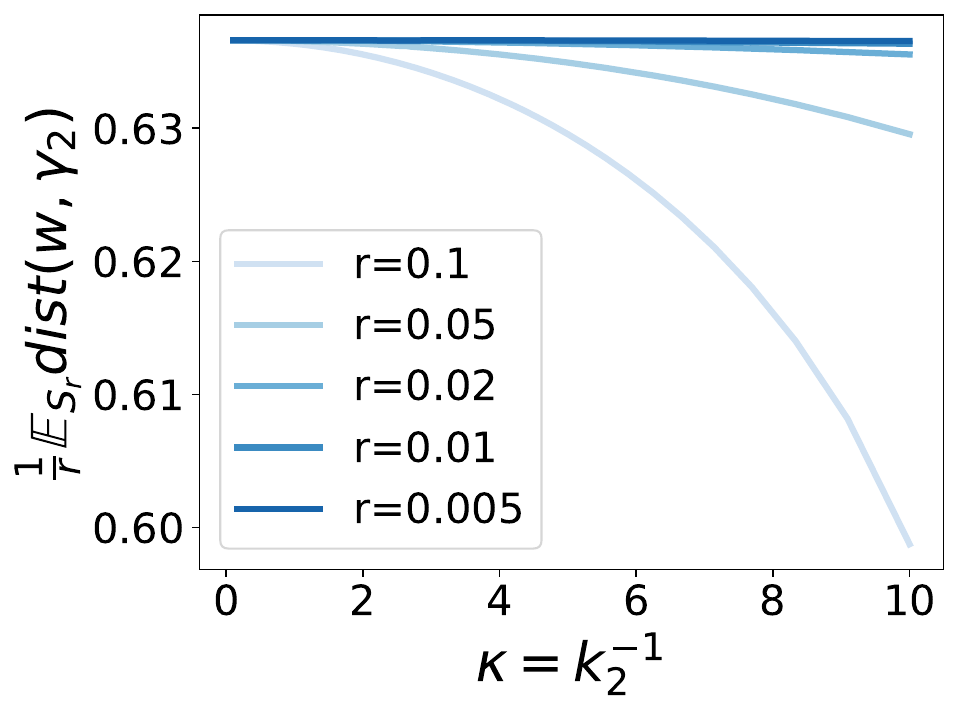}
\caption{(a) Illustration of the parameter space, the minimum ($\gamma$), and all shifts with distance $r$ ($S_r$). (b) Expected distance between $\vw_0$ and the new minimum as a function of $\kappa$, for quadratic approximation $\gamma_1$. (c) Expected distance between $\vw_0$ and the new minimum as a function of $\kappa$, for constant curvature approximation $\gamma_2$. The expected distance is scaled by $r$ so that the curves can be plotted together.} % the average distance at each point on $S_r$ is roughly proportional to $r$
\label{fig:curvature-displacement}
\end{center}
\end{figure}

\subsubsection{Higher dimensions}
Figure \ref{fig:curvature-3D} visualizes a curve obtained from a 2D minimum. However, it is not immediately clear what curves look like on a higher-dimensional minimum. A possible way to extend previous analysis is to consider sectional curvatures. 
% If the intersection of the plane and the minimum is a curve, similar analysis can be done on the plane. 

\begin{figure}[h!]
\begin{center}
\includegraphics[width=0.6\columnwidth]{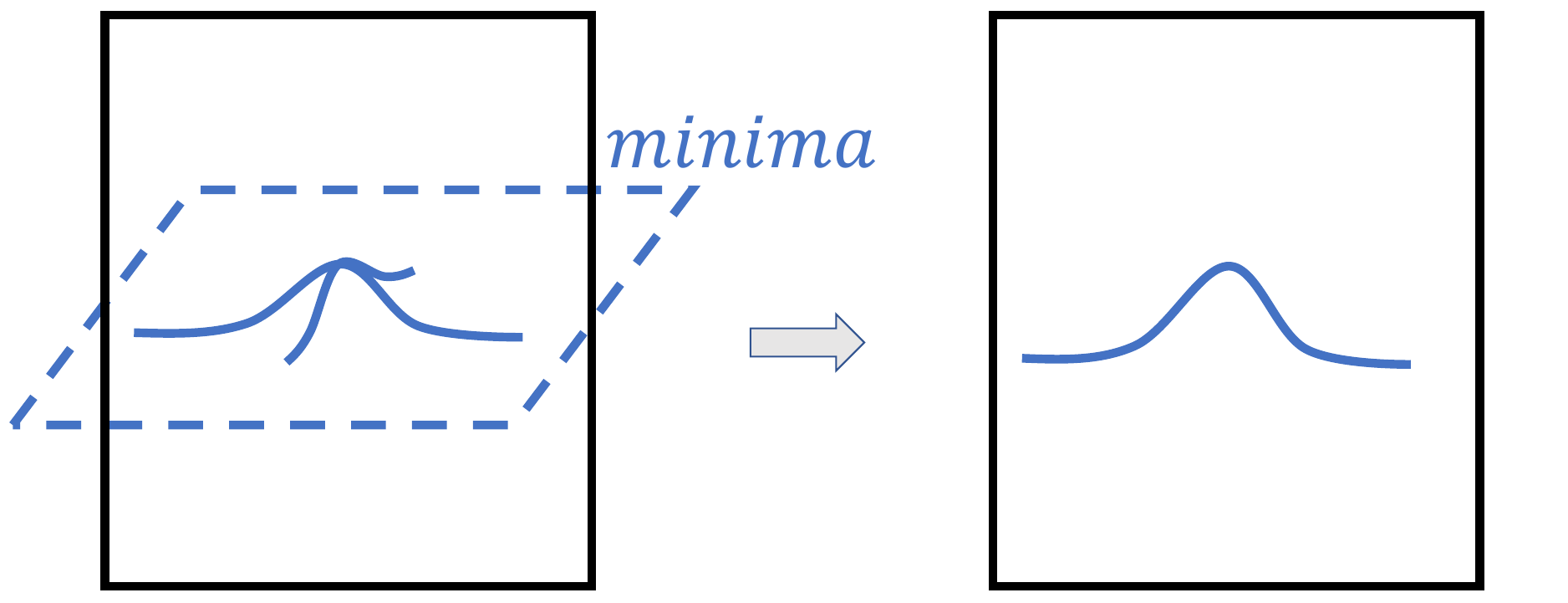}
\caption{Left: a 2D minima in a 3D parameter space. Right: a 2D subspace of the parameter space and a curve on the minima (the intersection of the minima and the subspace). }
\label{fig:curvature-3D}
\end{center}
\end{figure}

% \subsubsection{Example: correlation between curvature and sharpness in a 2-layer network}
% \bz{To do. Use the 1D $U, V$ example in ICLR submission.}

\subsection{Computing correlation to generalization}
\label{sec:appendix-correlation-experiment}
% \paragraph{Experiments in Section 4.3.}
We generate the 100 different models used in Section 4.3 by training randomly initialized models.
For all three datasets (MNIST, FashionMNIST, and CIFAR-10), we train on 50,000 samples and test on a different set of 10,000 samples. 
% We use a three-layer model and cross-entropy loss for classification with minibatches of size 20. 
The labels for classification tasks belongs to 1 of 10 classes.

For a batch of flattened input data $X \in \R^{d \times 20}$ and labels $Y \in \R^{20}$, the loss function is
$\loss(W_1,W_2,W_3, X, Y) = \text{CrossEntropy} \pa{W_3 \sigma(W_2 \sigma(W_1 X)), Y}$, 
where $W_3 \in \R^{10 \times h_2}$, $W_2 \in \R^{h_2 \times h_1}$, $W_1 \in \R^{h_1 \times d}$ are the weight matrices, and $\sigma$ is the LeakyReLU activation with slope coefficient 0.1. 
For MNIST and Fashion-MNIST, $d = 28^2$, $h_1 = 16$, and $h_2 = 10$. For CIFAR-10, $d=32^3 \times 3$, $h_1 = 128$, and $h_2 = 32$.
The learning rate for stochastic gradient descent is $0.01$ for MNIST and Fashion-MNIST, and $0.02$ for CIFAR-10. We train each model using mini-batches of size 20 for 40 epochs. 
% g / torch.norm(g, p='fro', dim=None) * 0.01 + torch.eye(dim[m+2]) * 1e0

When computing the sharpness $\phi$, we choose the displacement list T that gives the highest correlation. The displacements used in this paper are $T = {0.001, 0.011, 0.021, ..., 0.191}$ for MNIST, and $T = {0.001, 0.011, 0.021, ..., 0.191}$ for Fashion-MNIST and CIFAR-10. We evaluate the change in loss over $|D| = 200$ random directions. For curvature $\psi$, we average over $k=1$ curves generated by random Lie algebras (invertible matrices in this case).

Figure \ref{fig:correlation-sharpness-3layer-leakyrelu} and \ref{fig:correlation-curvature-3layer-leakyrelu} visualizes the correlation result in Table \ref{table:correlation}. Each point represents one model. 

\begin{figure}[h!]
\centering
% \ \ \ a\hfill b \hfill c\hfill d \hfill ~ \\
\ \ \ (a) \hfill (b) \hfill (c) \hfill ~ \\
\includegraphics[width=0.31\columnwidth]{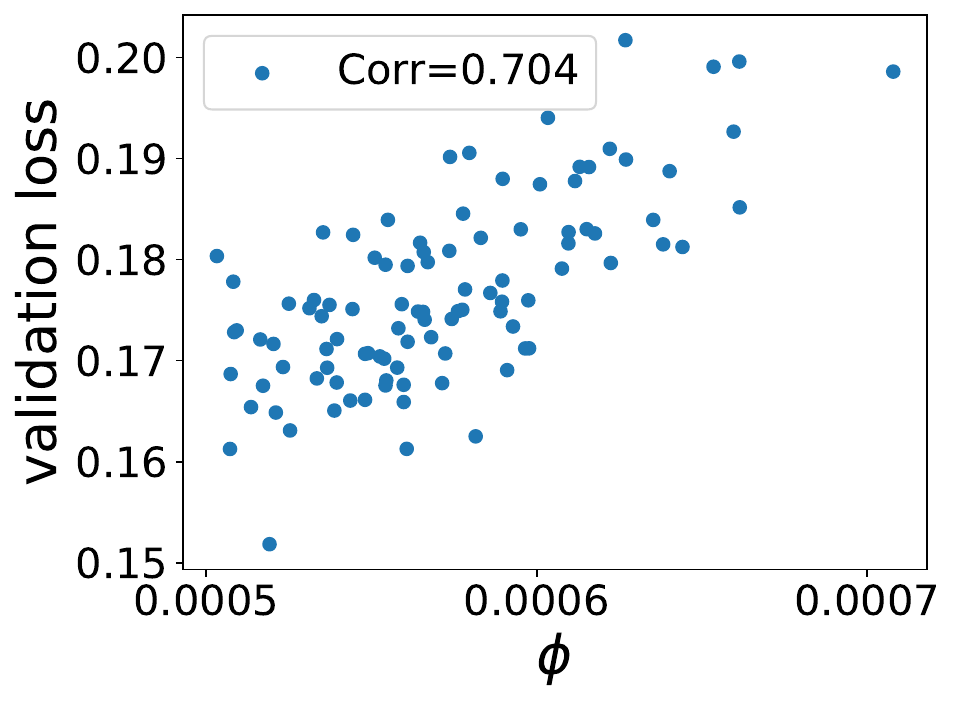}
\includegraphics[width=0.334\columnwidth]{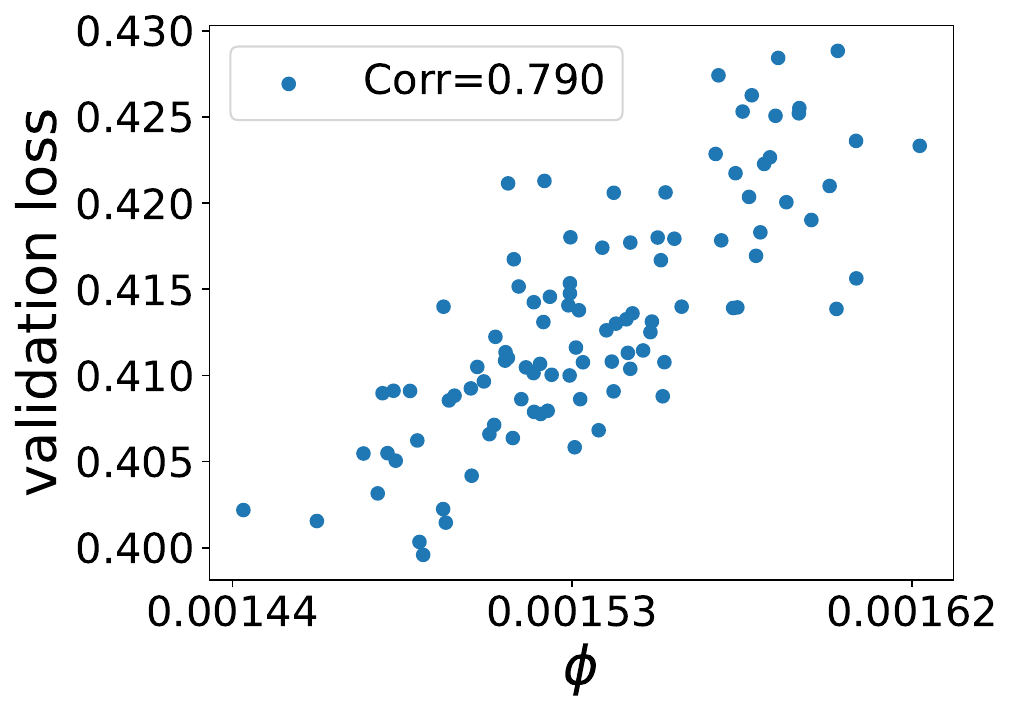}
\includegraphics[width=0.313\columnwidth]{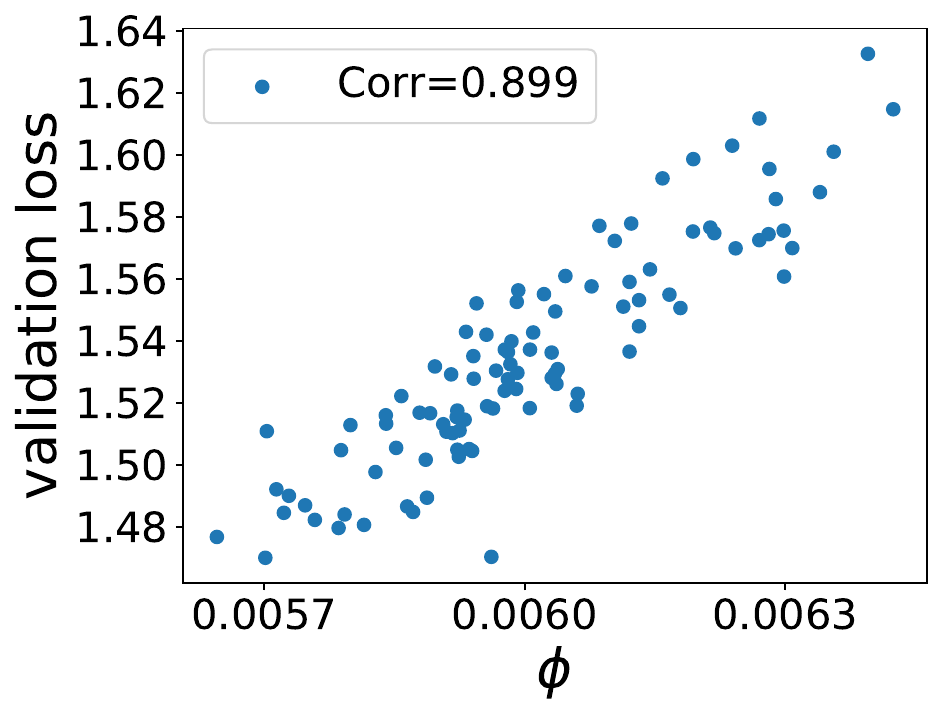}
\caption{Correlation between sharpness and validation loss on MNIST (left), Fashion-MNIST (middle), and CIFAR-10 (right). Sharpness and generalization are strongly correlated.
}
\label{fig:correlation-sharpness-3layer-leakyrelu}
\end{figure}

\begin{figure}[h!]
\centering
% \ \ \ a\hfill b \hfill c\hfill d \hfill ~ \\
\ \ \ (a) \hfill (b) \hfill (c) \hfill ~ \\
\includegraphics[width=0.315\columnwidth]{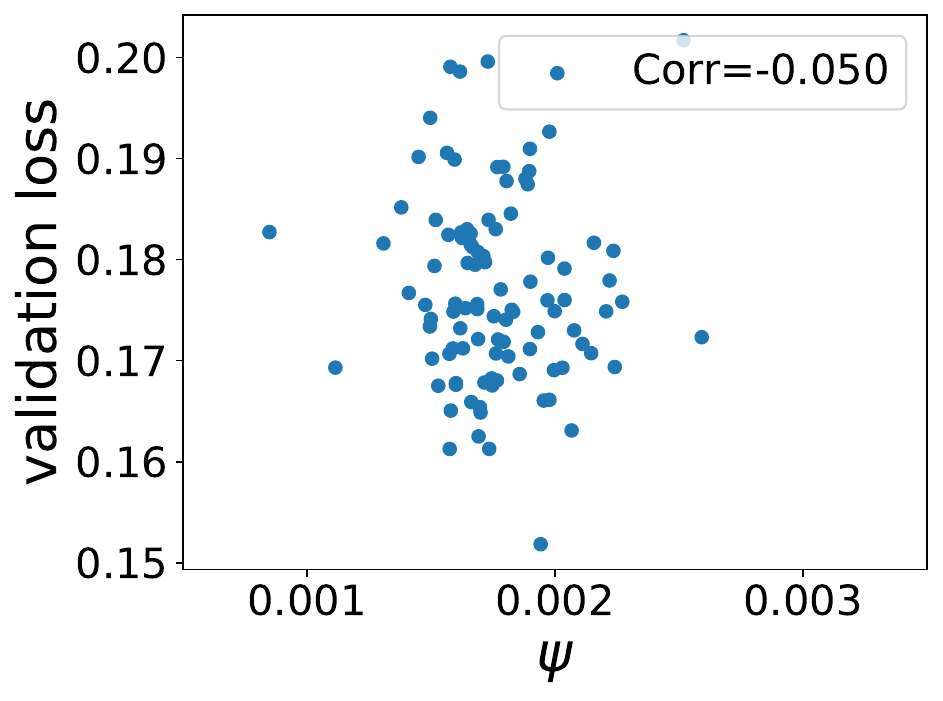}
\includegraphics[width=0.325\columnwidth]{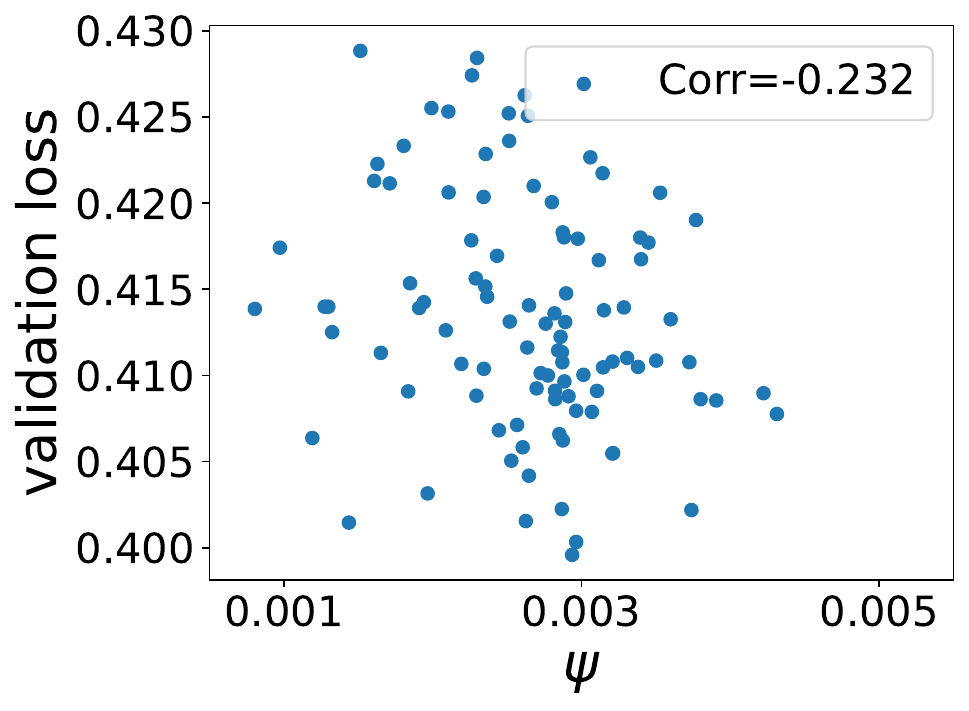}
\includegraphics[width=0.343\columnwidth]{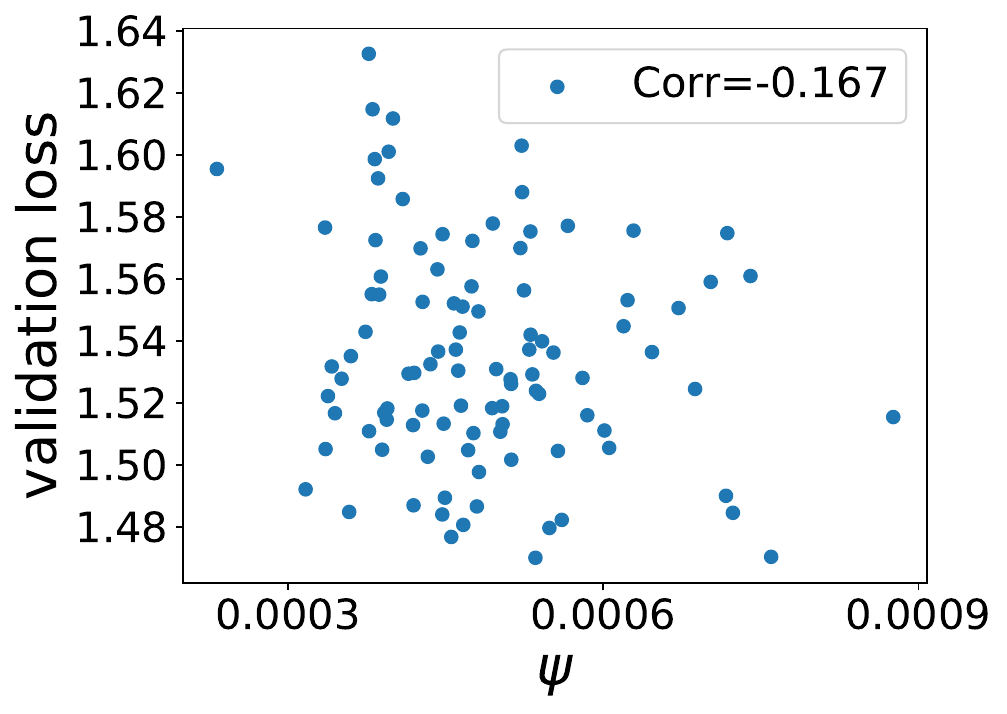}
\caption{Correlation between curvature and validation loss on MNIST (left), Fashion-MNIST (middle), and CIFAR-10 (right). There is a weak negative correlation in all three datasets.
}
\label{fig:correlation-curvature-3layer-leakyrelu}
\end{figure}

\subsection{Additional details for generalization experiments} 
\label{sec:appendix-generalization-experiment}

Algorithm \ref{alg:teleport-mlp} shows an example on how to perform a teleportation with an MLP.

\begin{algorithm}[H]
   \caption{Changing curvature using teleportation}
   \label{alg:teleport-mlp}
    \begin{algorithmic}
       \STATE {\bfseries Input:} loss function $L(w)$, parameters before teleportation $w_0$, teleportation learning rate $\eta_{teleport}$, number of teleportation steps $n_{teleport}$.
       \STATE {\bfseries Output:} parameters after teleportation $w_{n_{teleport}}$.

       \FOR{$t = 0$ {\bfseries to} $n_{teleport} - 1$}
       \STATE initialize $T = 0_{h \times h}$
       \STATE set $w_{t}' = (I_{h \times h} + T) \cdot w_t$
       \STATE compute $grad = \frac{d | \psi(w_{t}') |}{dT} $
       \STATE set $T_t = \eta_{teleport} \times grad$
       \STATE set $w_{t+1} = (I + T_t) \cdot w_t$
       \ENDFOR

       \STATE {\bfseries Return} $w_{n_{teleport}}$
    \end{algorithmic}
\end{algorithm}

% In experiments in Section 4.4, we perform gradient descent on the group elements to increase or decrease $\phi$ and $\psi$. 
% We use the same train/validation split of MNIST and architecture as in Section 4.3.
On CIFAR-10, we run SGD using the same three-layer architecture as in Section \ref{sec:appendix-correlation-experiment}, but with a smaller hidden size $h_1 = 32$ and $h_2 = 10$. At epoch 20 which is close to convergence, we teleport using 5 batches of data, each of size 2000. 
During each teleportation for $\phi$, we perform 10 gradient ascent (or descent) steps on the group element.
During each teleportation for $\psi$, we perform 1 gradient ascent (or descent) step on the group element. 
The learning rate for the optimization on group elements is $5 \times 10^{-2}$. 

To investigate how teleportation affects generalization for other optimizers, we repeat the same experiment but replace SGD with AdaGrad. Figure \ref{fig:teleport-generalization-cifar10-adagrad} shows the training curve of AdaGrad on CIFAR-10, averaged across 5 runs. Similar to SGD, changing curvature via teleportation affects the validation loss, while changing sharpness has negligible effects. Teleporting to points with larger curvatures helps find minima with slightly lower validation loss. Teleporting to points with smaller curvatures increases the gap between training and validation loss. 

\begin{figure}[h!]
\centering
\includegraphics[width=0.4\columnwidth]{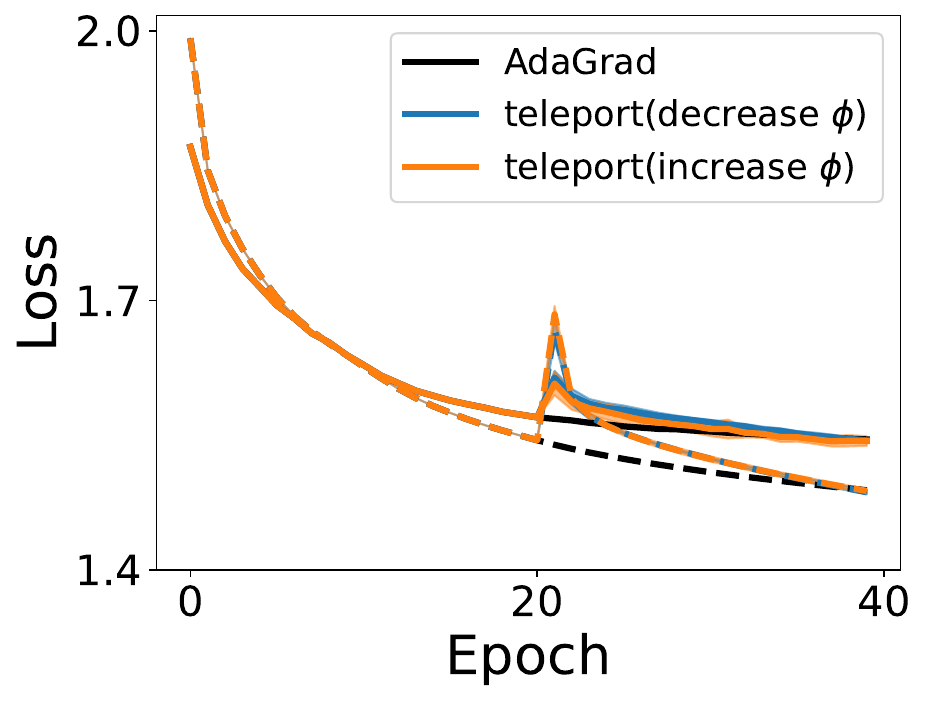}
\hspace{20pt}
\includegraphics[width=0.4\columnwidth]{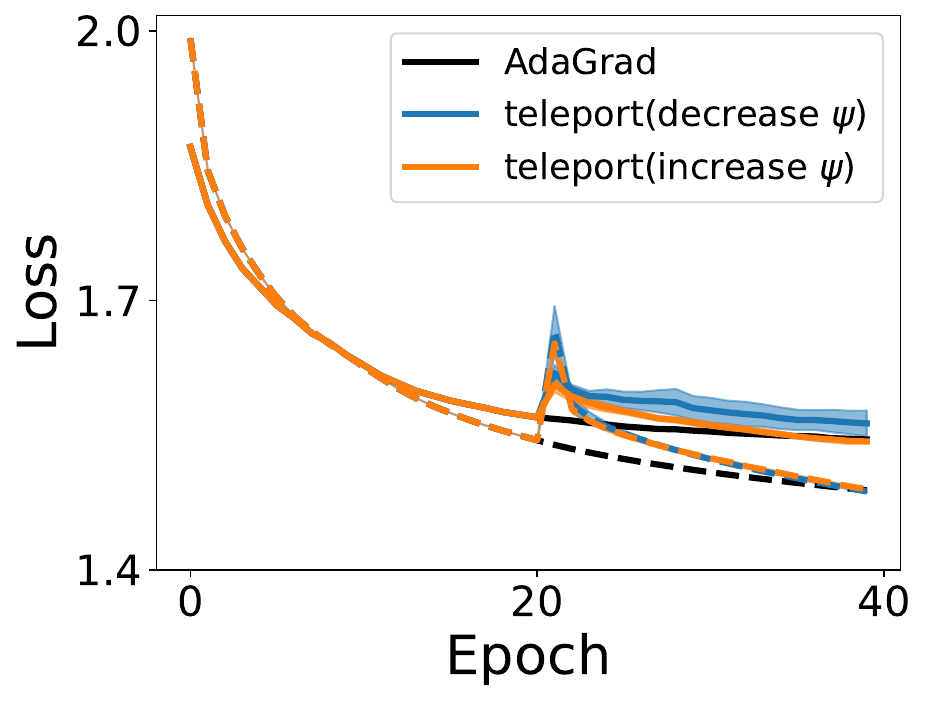}
\caption{Changing sharpness (left) or curvature (right) using teleportation and its effect on generalizability of AdaGrad solutions on CIFAR-10.
Solid line represents average test loss, and dashed line represent average training loss. }
\label{fig:teleport-generalization-cifar10-adagrad}
\end{figure}

%%
% \newpage
\section{Integrating teleportation with other gradient-based algorithms}
\label{appendix:other-algorithms}

\out{
\subsection{Integrating teleportation with other gradient-based algorithms}

\paragraph{Polyak stepsize.} 
In previous comparison of gradient descent with and without teleportation, a fixed learning rate is used. 
% This comparison, however, does not distinguish the source of improvement.
To answer the question of whether teleportation is equivalent to just increasing stepsize, we use Polyak stepsize (Algorithm \ref{alg:teleport-polyak}). Figure \ref{fig:teleport-polyak} shows that when using Polyak stepsize, teleportation does not improve the convergence rate. 

\begin{figure}[h]
\centering
\ \ \ a\hfill b \hfill c\hfill d \hfill ~ \\
\includegraphics[width=0.24\textwidth]{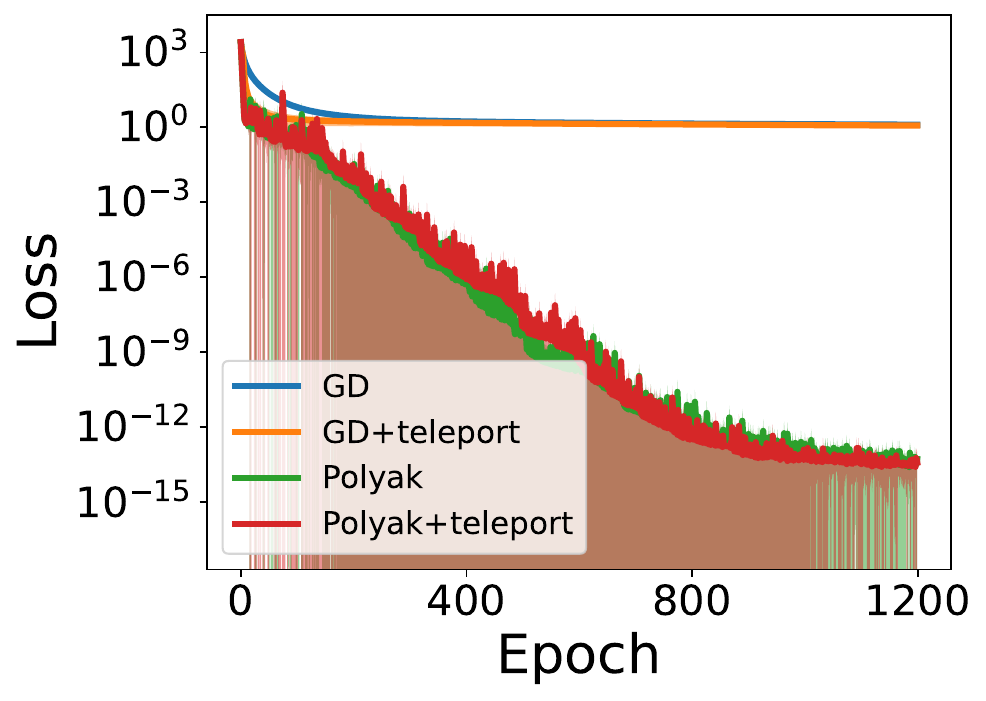}
\includegraphics[width=0.24\textwidth]{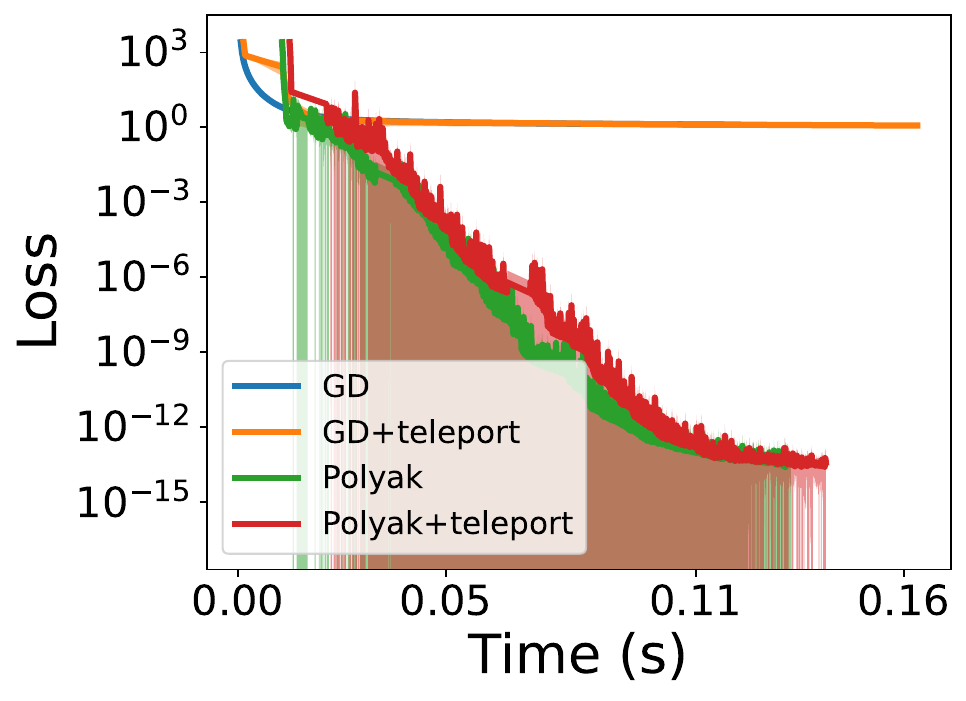}
\includegraphics[width=0.24\textwidth]{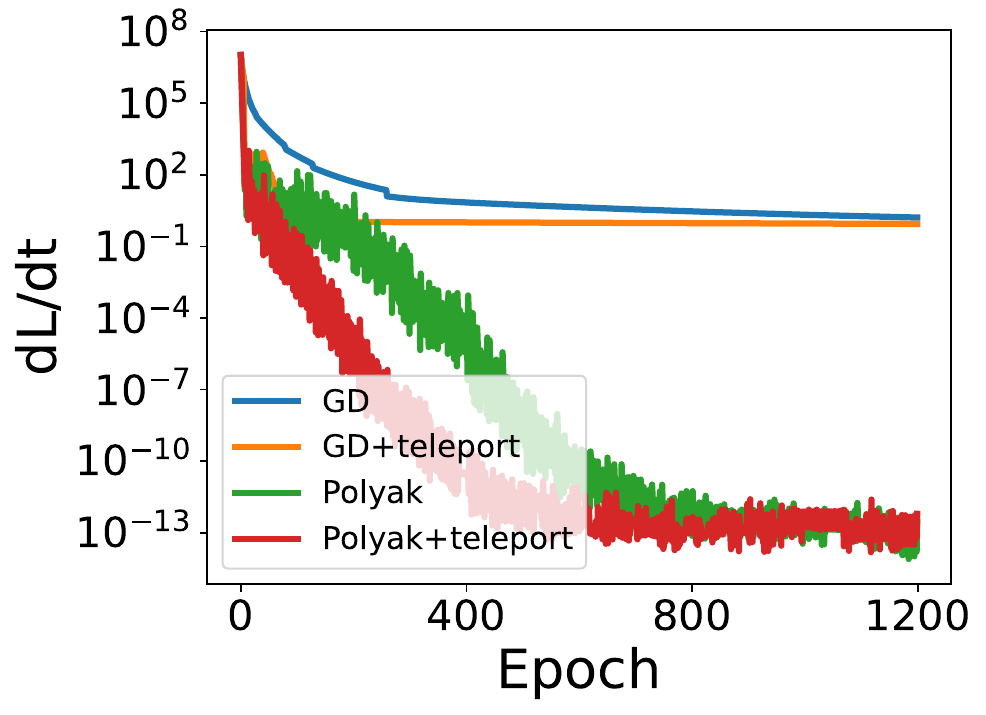}
\includegraphics[width=0.24\textwidth]{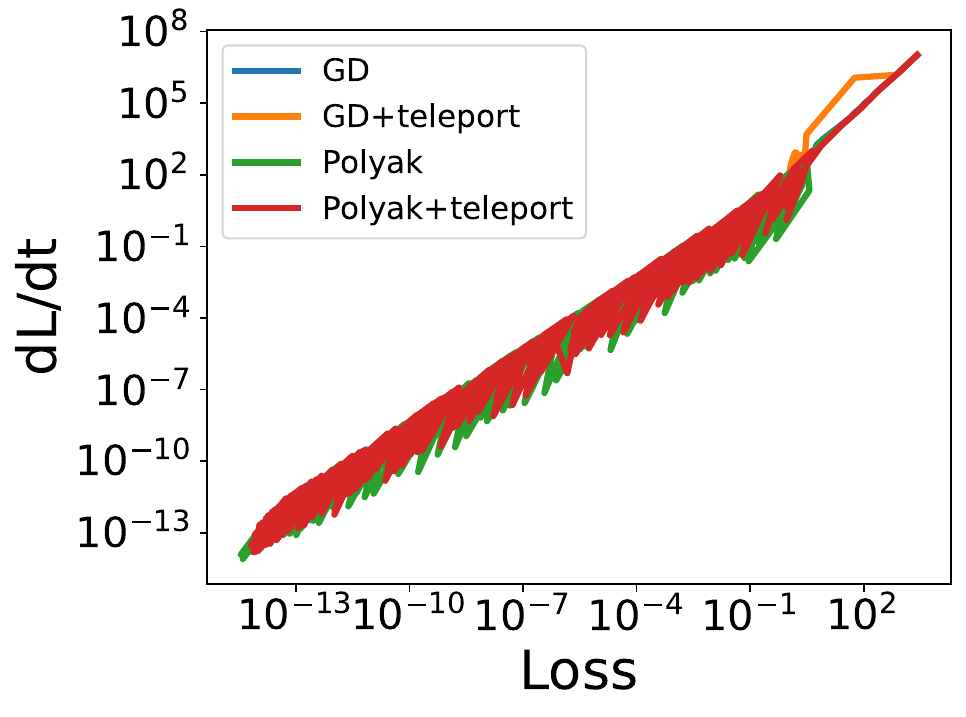}
\caption{Polyak.}
\label{fig:teleport-polyak}
\end{figure}

\paragraph{Momentum.} 
We compare three strategies of integrating teleportation with momentum: teleporting both parameters and momentum (Algorithm \ref{alg:teleport-momentum}), teleporting parameters but not momentum (Algorithm \ref{alg:teleport-momentum} with line 6 removed), and reset momentum to 0 after a teleportation (Algorithm \ref{alg:teleport-momentum} with line 6 replaced by $\vv_t \leftarrow 0$). 

The training curves of teleporting momentum in different ways are similar (Figure \ref{fig:teleport-momentum}).
Using the best method to teleport momentum (reset), gradient descent + teleportation with or without momentum is almost the same (Figure \ref{fig:teleport-momentum-GD}).
% \bz{Results for momentum seem to be sensitive to random seed, hence not reliable. Will look into it later.}

\begin{figure}[h]
\centering
\ \ \ a\hfill b \hfill c\hfill d \hfill ~ \\
\includegraphics[width=0.24\textwidth]{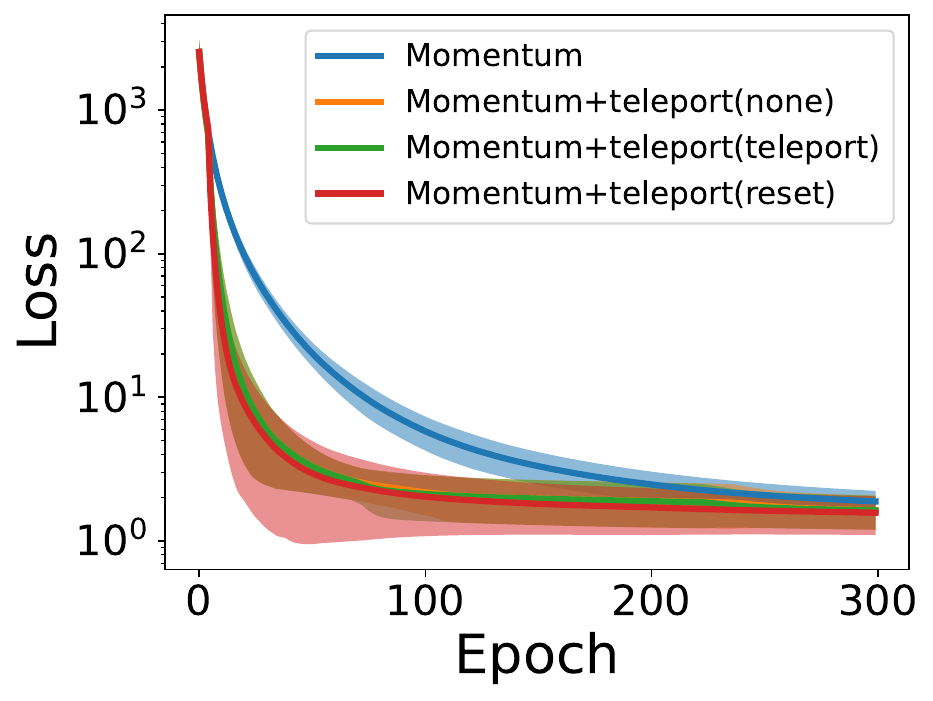}
\includegraphics[width=0.24\textwidth]{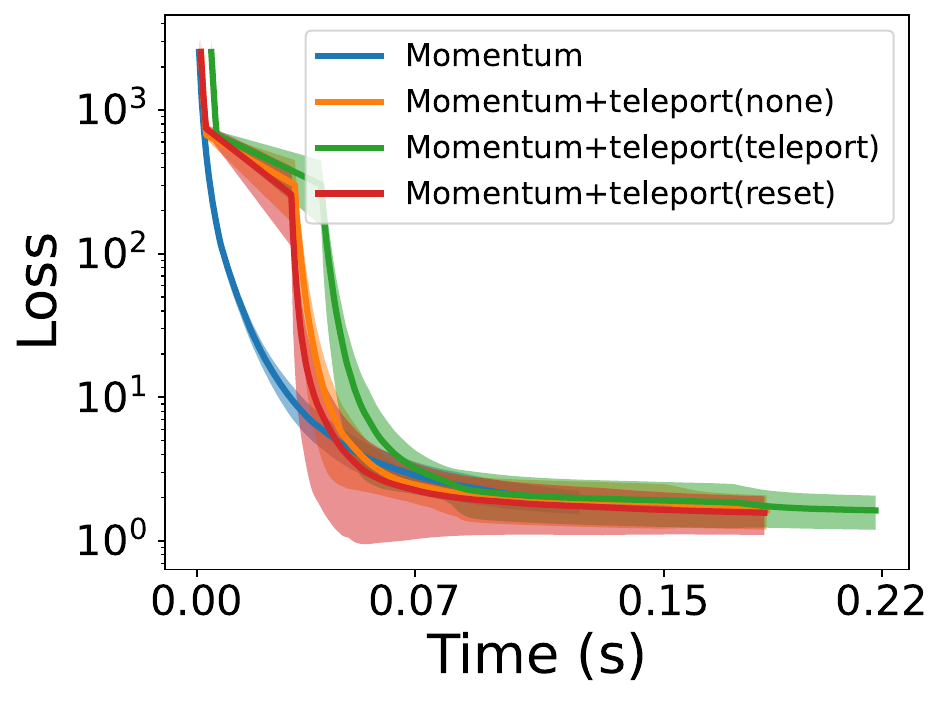}
\includegraphics[width=0.24\textwidth]{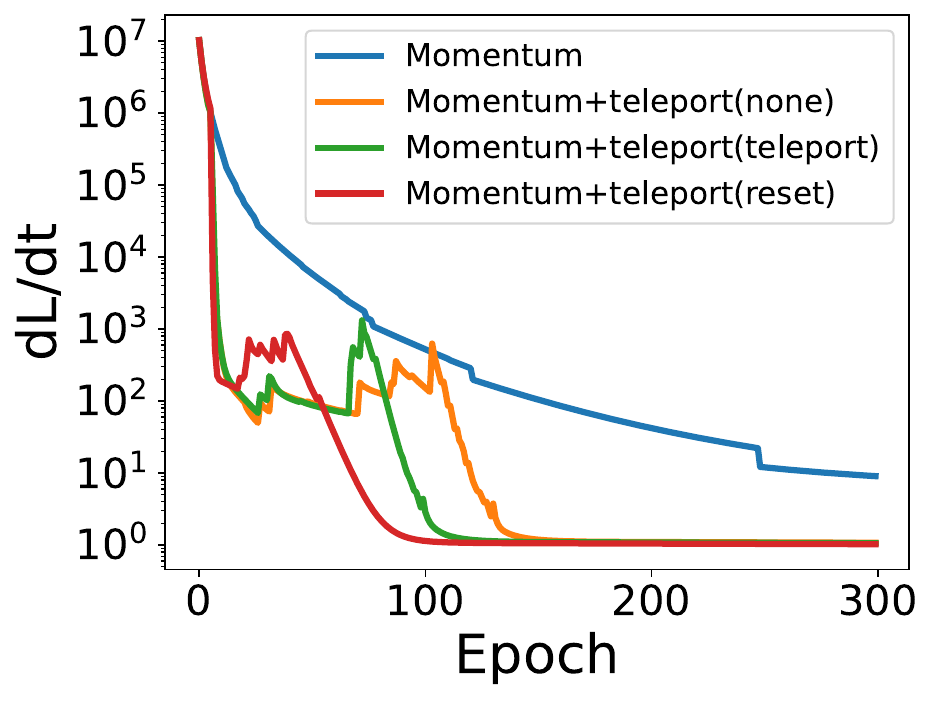}
\includegraphics[width=0.24\textwidth]{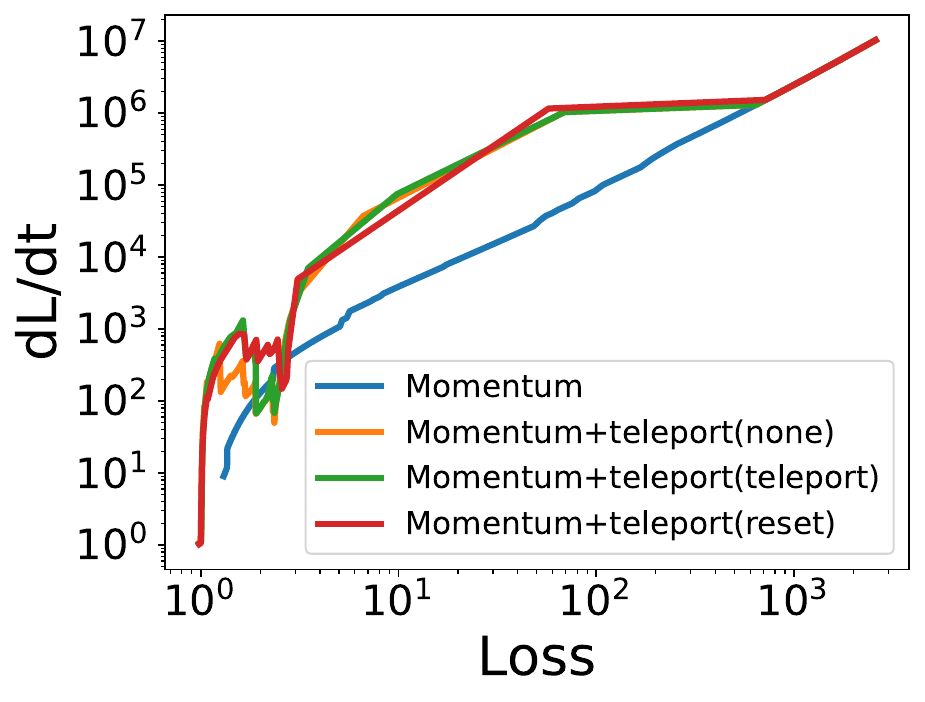}
\caption{Teleporting momentum.}
\label{fig:teleport-momentum}
\end{figure}

\begin{figure}[h]
\centering
\ \ \ a\hfill b \hfill c\hfill d \hfill ~ \\
\includegraphics[width=0.24\textwidth]{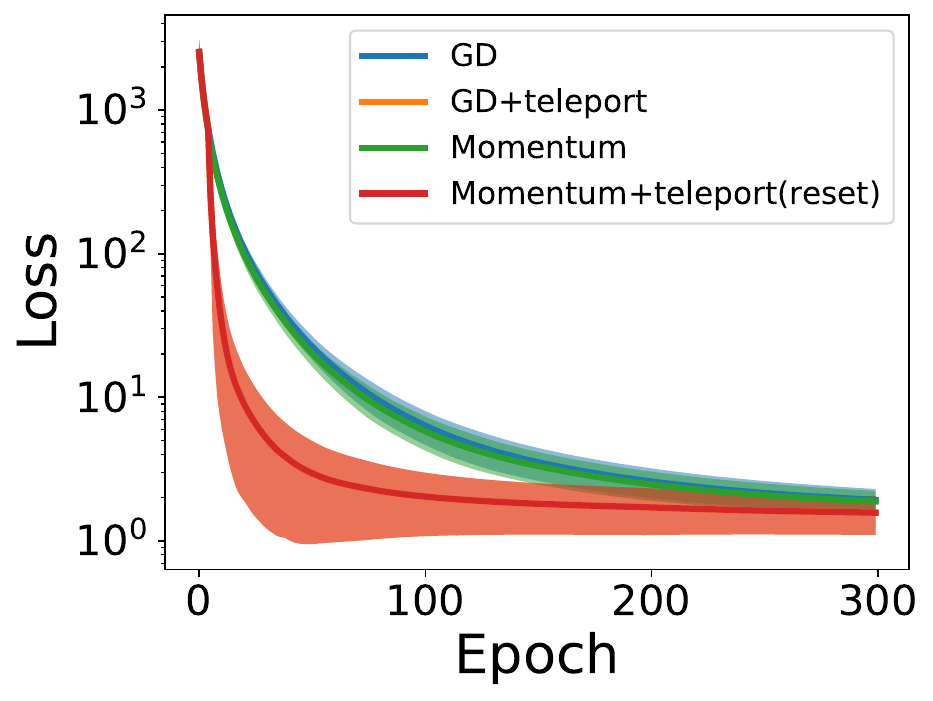}
\includegraphics[width=0.24\textwidth]{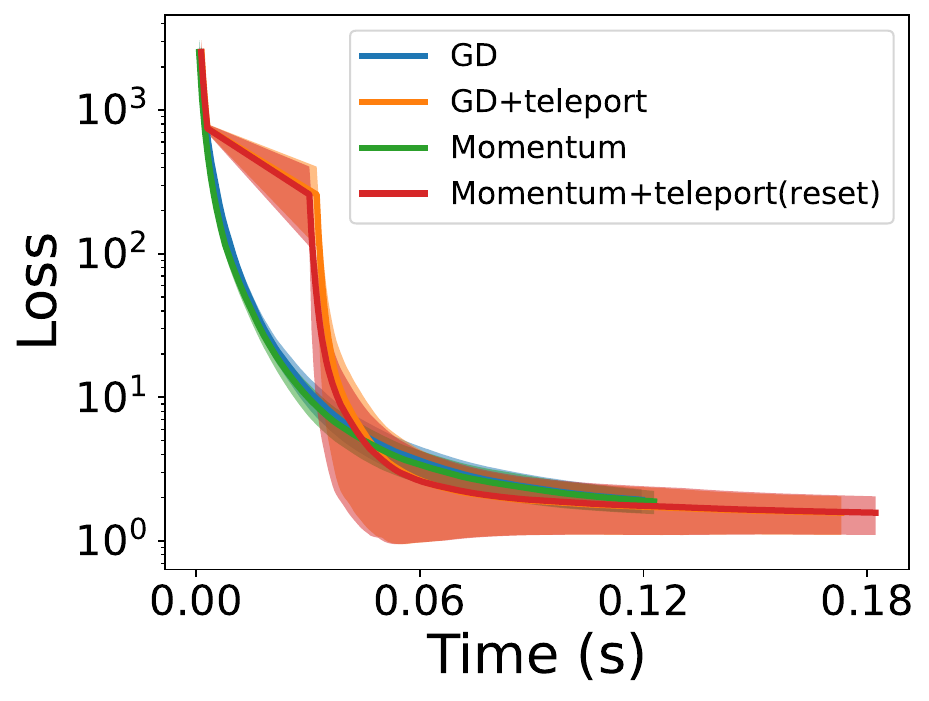}
\includegraphics[width=0.24\textwidth]{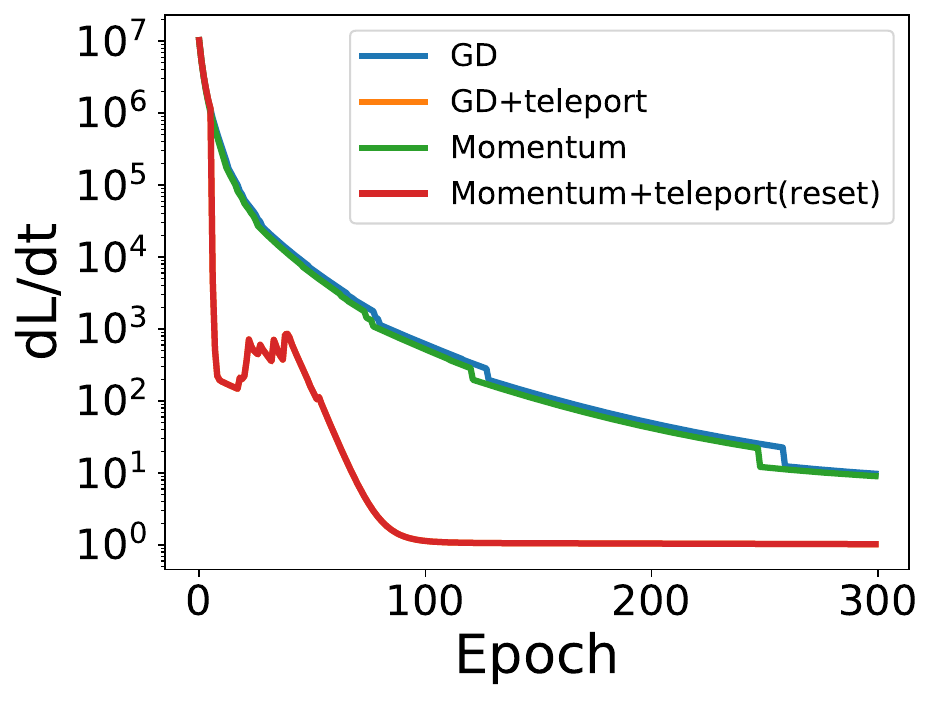}
\includegraphics[width=0.24\textwidth]{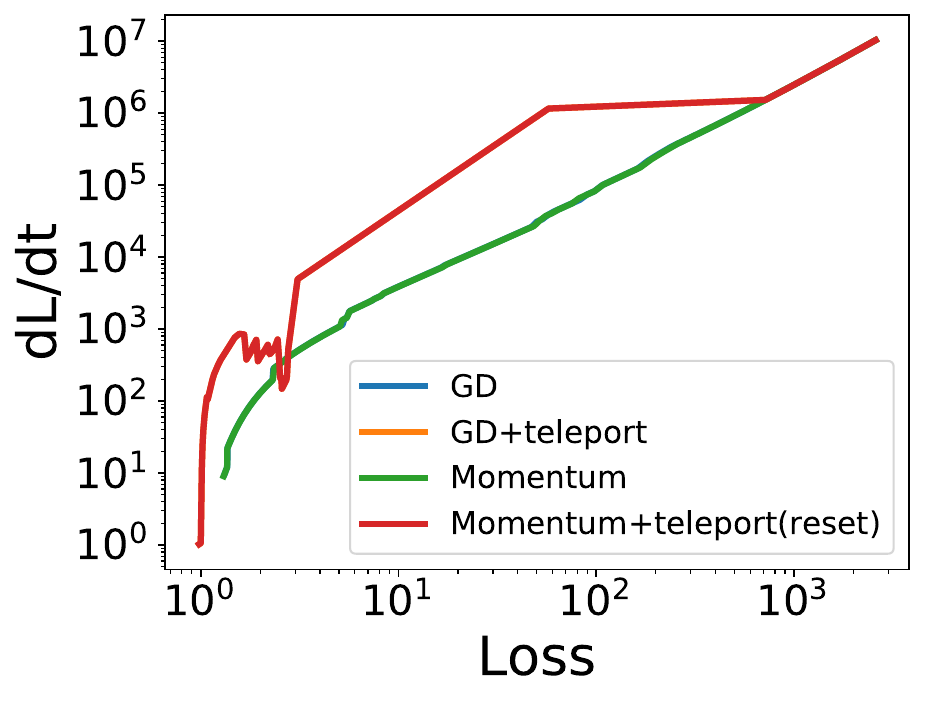}
\caption{Momentum best setting vs GD.}
\label{fig:teleport-momentum-GD}
\end{figure}

\paragraph{AdaGrad.} In AdaGrad (Algorithm \ref{alg:teleport-adagrad}), the rate of change in loss is 
\begin{align}
    {d\loss(\vw)\over dt} 
    = \bk{{\ro \L \over \ro \vw },{d\vw \over dt} } 
    = - \grad \L^T \frac{\eta}{\sqrt{\eps I + diag(G_{t+1})}} \grad\L
    = - \eta \|\grad \L\|_A
\end{align}
where $\eta \in \R$ is the learning rate, and $\|\grad \L\|_A$ is the Mahalanobis norm with $A = (\eps I + diag(G_{t+1}))^{-\frac{1}{2}}$. Previously, we optimize $\|\grad \L\|_2$ in teleportation. We compare that to optimizing $\|\grad \L\|_A$. Since the magnitude of $A$ is different than 1, a different learning rate for the gradient ascent in teleportation is required. We choose the largest learning rate (with two significant figures) that does not lead to divergence. The teleportation learning rates used are $1.2 \times 10^{-5}$ for objective $\max_g \|\grad \L\|_2$ and $7.5 \times 10^{-3}$ for objective $\max_g \|\grad \L\|_A$.

Teleporting using the group element that optimizes $\|\grad \L\|_A$ has a slight advantage (Figure \ref{fig:teleport-adagrad}).

\begin{figure}[h!]
\centering
\ \ \ a\hfill b \hfill c\hfill d \hfill ~ \\
\includegraphics[width=0.24\textwidth]{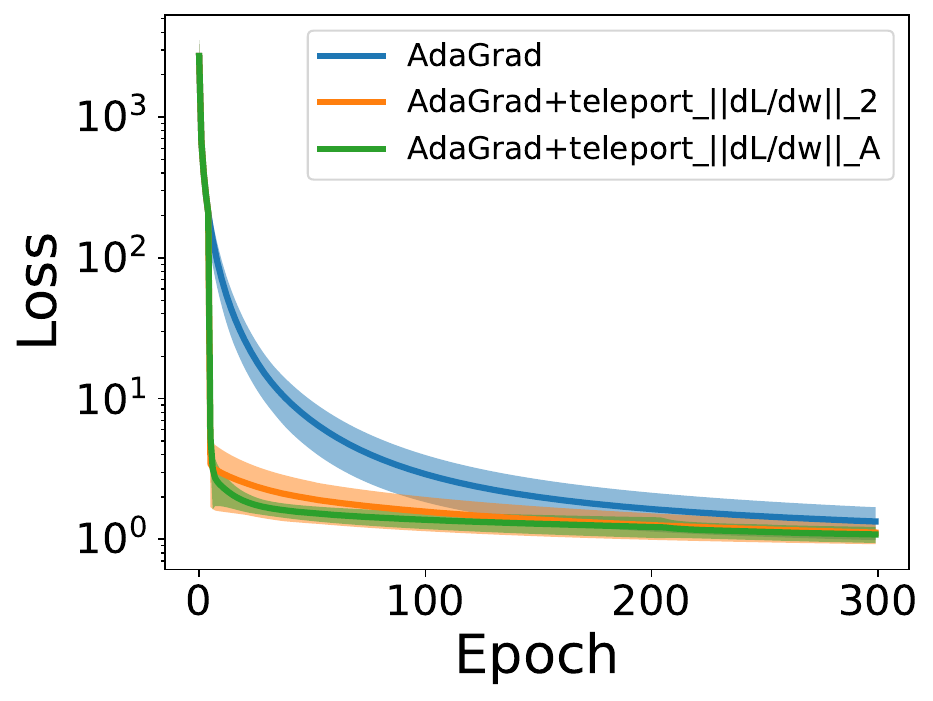}
\includegraphics[width=0.24\textwidth]{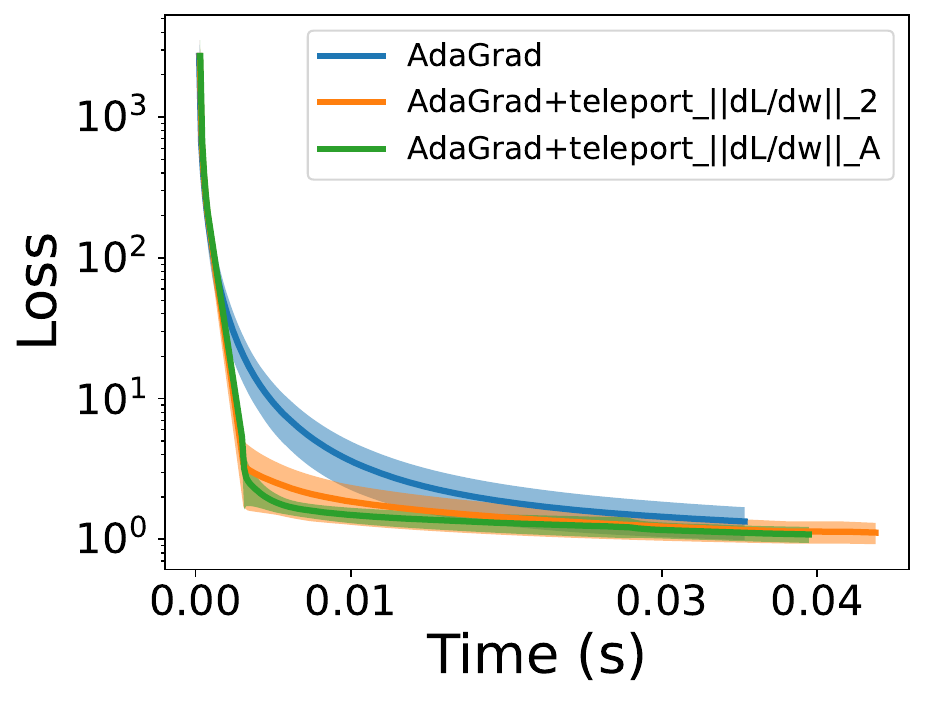}
\includegraphics[width=0.24\textwidth]{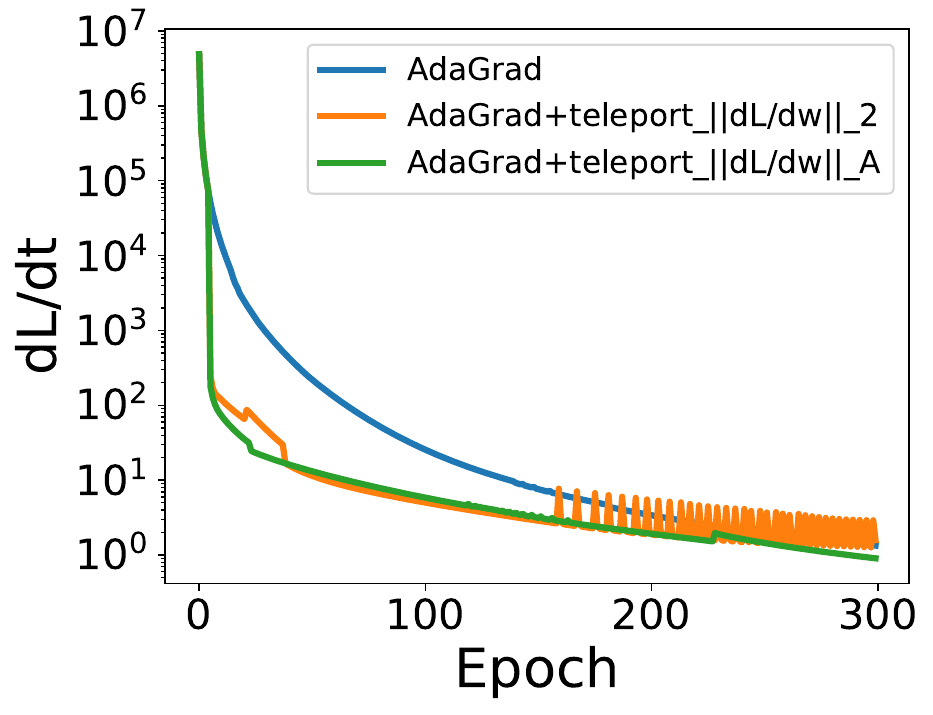}
\includegraphics[width=0.24\textwidth]{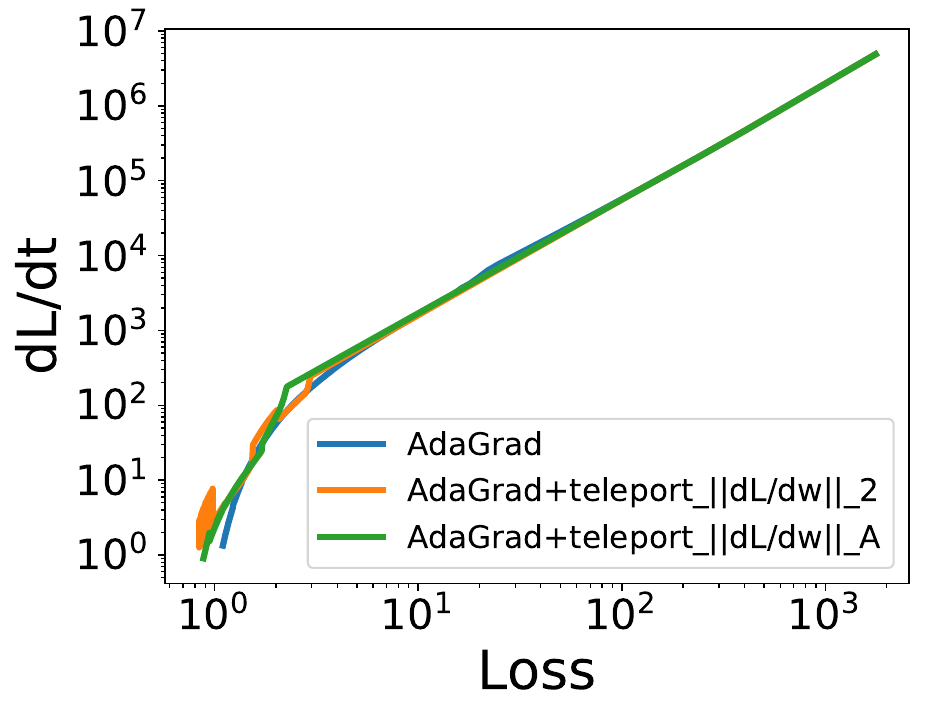}
\caption{AdaGrad.}
\label{fig:teleport-adagrad}
\end{figure}

% \paragraph{RMSprop/Adam.} AdaGrad + decayed average of $G$ + momentum
}

\subsection{Different methods of integrating teleportation with momentum and AdaGrad}

\paragraph{Setup.}
We test teleportation with various algorithms using the a 3-layer neural network and mean square error: $\min_{W_1,W_2,W_3}\| Y - W_3 \sigma(W_2 \sigma(W_1 X))\|_2$, with data $X \in \R^{5 \times 4}$, target $Y \in \R^{8 \times 4}$, and weight matrices $W_3 \in \R^{8 \times 7}$, $W_2 \in \R^{7 \times 6}$, and $W_1 \in \R^{6 \times 5}$. The activation function $\sigma$ is LeakyReLU with slope coefficient 0.1. 
Each element in the weight matrices is initialized uniformly at random over $[0, 1]$.
Data $X, Y$ are randomly generated also from $[0, 1]$. 

% GD uses learning rate $10^{-4}$ and AdaGrad uses $10^{-1}$. 
% Each algorithm is run 300 steps.
% When using teleportation, we perform symmetry transform on the parameters once at epoch 5.
% In GD, the group elements used for these transforms are found by gradient ascent on $T$ for 8 steps, with learning rate $10^{-7}$. In AdaGrad, the group elements are found by gradient ascent for 2 steps, with learning rate $10^{-5}$. 
% The choice of hyperparameters comes from a grid search described in the next section.

\paragraph{Momentum.} 
We compare three strategies of integrating teleportation with momentum: teleporting both parameters and momentum, teleporting parameters but not momentum, and reset momentum to 0 after a teleportation. In each run, we teleport once at epoch 5. Each strategy is repeated 5 times.

The training curves of teleporting momentum in different ways are similar (Figure \ref{fig:other-algorithms}a), possibly because the momentum accumulated is small compared to the gradient right after teleportations. All methods of teleporting momentum improves convergence, which means teleportation works well with momentum.

% Using the best method to teleport momentum (reset), gradient descent + teleportation with or without momentum is almost the same (Figure \ref{fig:teleport-momentum-GD}).

\paragraph{AdaGrad.} In AdaGrad, the rate of change in loss is 
\begin{align}
    {d\loss(\vw)\over dt} 
    = {\ro \loss \over \ro \vw }^T {d\vw \over dt}
    % = \bk{{\ro \loss \over \ro \vw },{d\vw \over dt} } 
    % = - \grad \loss^T \frac{\eta}{\sqrt{\eps I + diag(G_{t+1})}} \grad\loss
    = - \eta \|\grad \loss\|_A,
\end{align}
where $\eta \in \R$ is the learning rate, and $\|\grad \loss\|_A$ is the Mahalanobis norm with $A = (\eps I + diag(G_{t+1}))^{-\frac{1}{2}}$. Previously, we optimize $\|\grad \loss\|_2$ in teleportation. We compare that to optimizing $\|\grad \loss\|_A$. Since the magnitude of $A$ is different than 1, a different learning rate for the gradient ascent in teleportation is required. We choose the largest learning rate (with two significant figures) that does not lead to divergence. The teleportation learning rates used are $1.2 \times 10^{-5}$ for objective $\max_g \|\grad \loss\|_2$ and $7.5 \times 10^{-3}$ for objective $\max_g \|\grad \loss\|_A$.

Teleporting using the group element that optimizes $\|\grad \loss\|_A$ has a slight advantage (Figure \ref{fig:other-algorithms}b). Similar to the observations in \cite{zhao2022symmetry}, teleportation can be integrated into adaptive gradient descents.

% \paragraph{RMSProp/Adam}

\begin{figure}[h]
\centering
% \ \ \ a\hfill b \hfill c\hfill d \hfill ~ \\
% \ \ \ (a)\hfill (b) \hfill ~ \\
\ \ \ \hspace{30pt} (a) \hfill \hspace{-35pt} (b) \hfill ~ \\
\includegraphics[width=0.4\columnwidth]{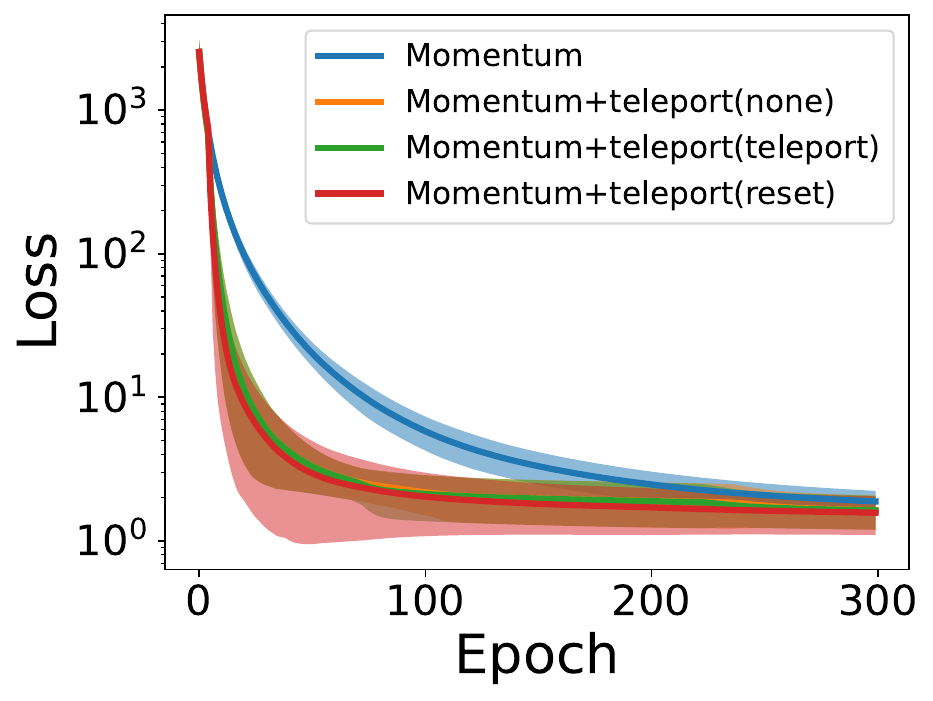}
\includegraphics[width=0.4\columnwidth]{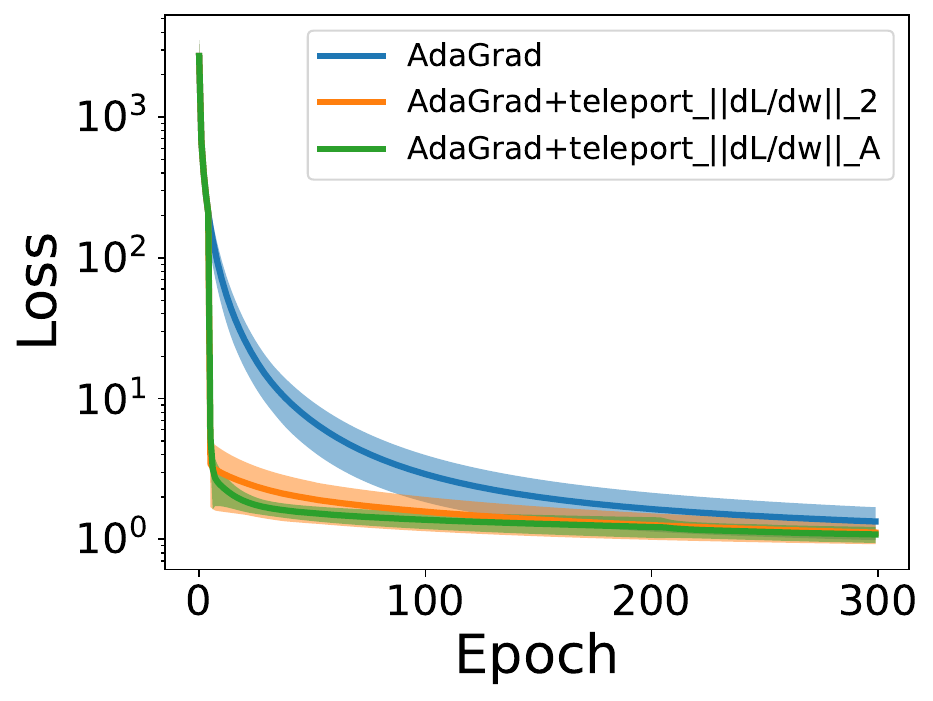}
\caption{Comparison of different methods of integrating teleportation with momentum and AdaGrad. }
\label{fig:other-algorithms}
\end{figure}

\subsection{Additional details for experiments on MNIST}
\label{sec:appendix-optimization-experiment}
% We generate the 100 different models used in Section 4.3 by randomly teleporting a model during training.

We use a three-layer model and cross-entropy loss for classification with minibatches of size 20.
For a batch of flattened input data $X \in \R^{28^2 \times 20}$ and labels $Y \in \R^{20}$, the loss function is
$\loss(W_1,W_2,W_3, X, Y) = \text{CrossEntropy} \pa{W_3 \sigma(W_2 \sigma(W_1 X)), Y}$, 
where $W_3 \in \R^{10 \times 10}$, $W_2 \in \R^{10 \times 16}$, $W_1 \in \R^{16 \times 28^2}$ are the weight matrices, and $\sigma$ is the LeakyReLU activation with slope coefficient 0.1. 
The learning rates are $10^{-4}$ for AdaGrad, and $5\times 10^{-2}$ for SGD with momentum, RMSProp, and Adam. The learning rate for optimizing the group element in teleportation is $5\times 10^{-2}$, and we perform 10 gradient ascent steps when teleporting using each mini-batch. 
We use 50,000 samples from training set for training, and 10,000 samples in the test set for testing. 

% \subsection{Wall-clock time}
% \label{sec:appendix-optimization-experiment-time}

\begin{figure}[h]
\centering
\ \ \ (a) \hfill (b) \hfill (c) \hfill (d) \hfill ~ \\
% \ \ \ (a)\hfill (b) \hfill ~ \\
\includegraphics[width=0.245\columnwidth]{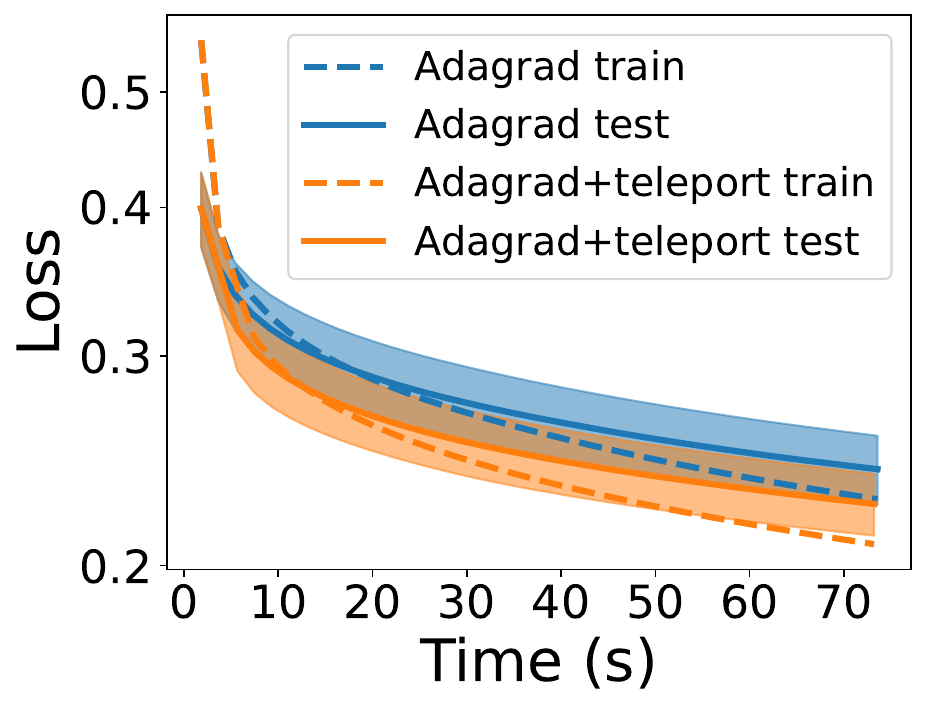}
\includegraphics[width=0.245\columnwidth]{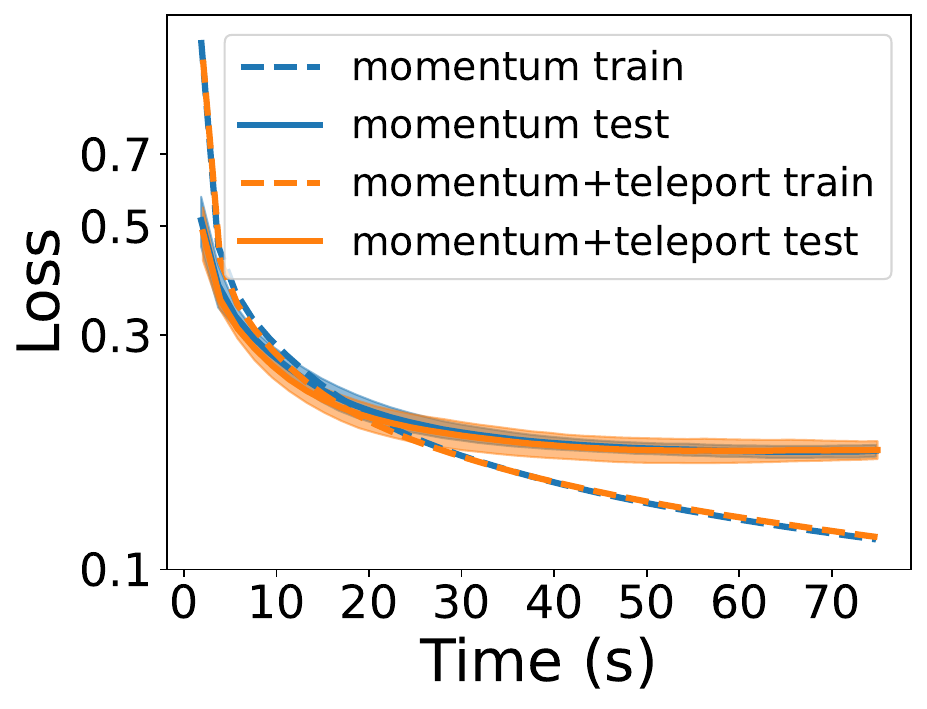}
\includegraphics[width=0.245\columnwidth]{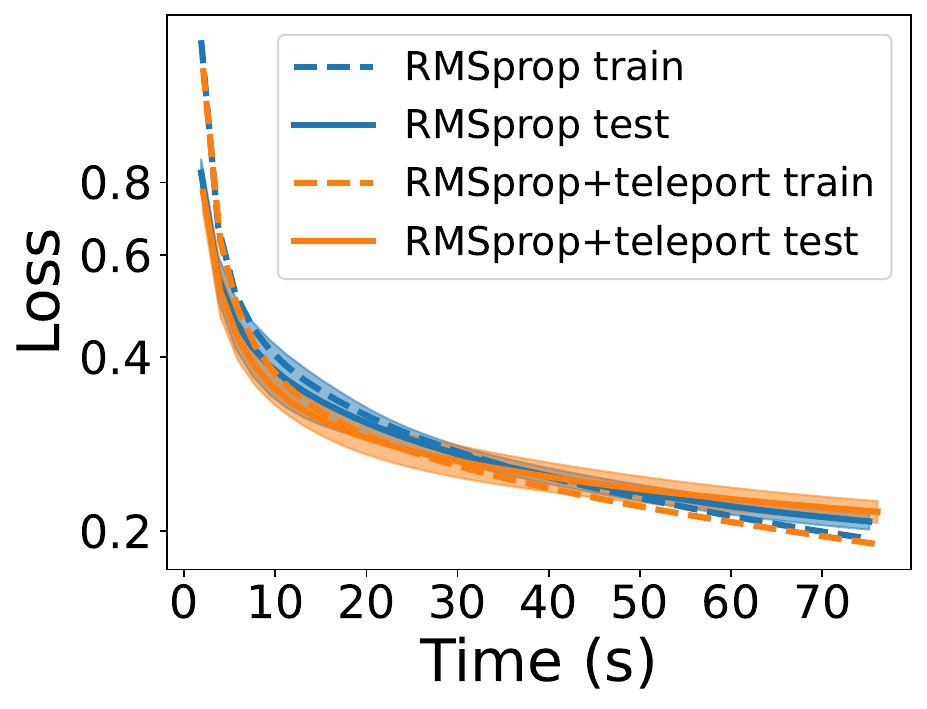}
\includegraphics[width=0.245\columnwidth]{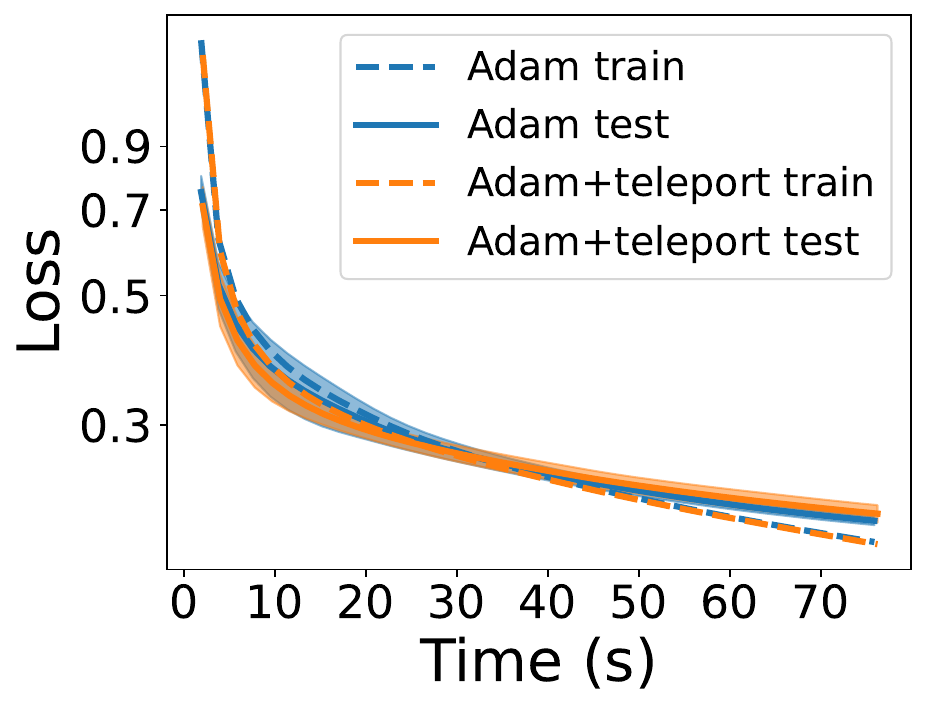}
\caption{Runtime comparison for integrating teleportation into various algorithms. Solid line represents average training loss, and dashed line represents average test loss. Shaded areas are 1 standard deviation of the test loss across 5 runs. The  plots look almost identical to Figure \ref{fig:other-algorithms-loss-vs-epoch}, indicating that the cost of teleportation is negligible compared to gradient descents. }
\label{fig:other-algorithms-loss-vs-time}
\end{figure}

% \subsection{Details for experiments on learning to teleport}

 %%
% \input{secs/appendix_app_meta.tex} %%
% \input{secs/appendix_generalization.tex}
% \input{secs/appendix_SPS.tex}
% \input{secs/icml_backup.tex}

\end{document}